%% file: arXiv.tex
\documentclass{article}

\usepackage{arxiv}

\usepackage[utf8]{inputenc} 
\usepackage[T1]{fontenc}    
\usepackage{hyperref}       
\usepackage{url}            
\usepackage{booktabs}       
\usepackage{amsfonts}       
\usepackage{nicefrac}       
\usepackage{microtype}      
\usepackage{xcolor}         
\usepackage{amsmath,amssymb}

\usepackage[ruled,vlined]{algorithm2e}

\input{def.tex}

\usepackage{amsmath,amssymb,amsthm}

\theoremstyle{plain}

\newtheorem{theorem}{Theorem}[section]
\newtheorem{lemma}[theorem]{Lemma}

\newtheorem{corollary}[theorem]{Corollary}
\newtheorem{proposition}[theorem]{Proposition}

\newtheorem{definition}[theorem]{Definition}

\begin{document}

\title{The Role of Contextual Information in Best Arm Identification}

\newcommand*{\affaddr}[1]{#1} 
\newcommand*{\affmark}[1][*]{\textsuperscript{#1}}
\newcommand*{\equalcontribution}[1][*]{\textsuperscript{*}}
\newcommand*{\email}[1]{\texttt{#1}}

\author{%
Masahiro Kato\affmark[1]\thanks{Equal contributions. Masahiro Kato: \email{mkato-csecon@g.ecc.u-tokyo.ac.jp}. Kaito Ariu: \email{ariu@kth.se}.}\ \ \ \ \ \ Kaito Ariu\affmark[2]\footnotemark[1]\\
\affaddr{\affmark[1] University of Tokyo}\\
\affaddr{\affmark[2] CyberAgent Inc}
}

\maketitle

\begin{abstract}
We study the best-arm identification problem with fixed confidence when contextual (covariate) information is available in stochastic bandits. In each round, we observe contextual information before selecting an arm. The distribution of the reward associated with the selected arm depends on the observed contextual information. We are interested in finding the arm with the maximum mean reward marginalized over the contextual distribution and not the mean reward conditioned on contexts. Our goal is to identify the best arm with a minimal number of samplings under a given value of the error rate. First, we derive the instance-specific sample-complexity lower bounds under the contextual information. Then, we propose a context-aware version of the ``Track-and-Stop'' strategy, wherein the proportion of the arm draws tracks the set of optimal allocations, and prove that the expected number of arm draws asymptotically matches the lower bound. We demonstrate that the contextual information can be used to improve the efficiency of the identification of the best marginalized mean reward when compared with the results of \citet{Garivier2016}. Furthermore, we experimentally confirm that context information contributes to faster best-arm identification.
\end{abstract}

\section{Introduction}
This paper studies best-arm identification (BAI) with contextual information in stochastic multi-armed bandit (MAB) problems. 
We define the best arm as the arm with the maximum marginalized mean reward, where the expectation is defined over the context distribution, not on a specific context.
We call this setting contextual BAI. The goal is to identify the best arm with a fixed confidence level and a smaller sample complexity defined by the probably approximately correct (PAC) framework. The instance-specific sample complexity of BAI without contextual information is now well understood. There exists an instance-specific lower bound \citep{Kaufman2016complexity,Garivier2016} and optimal algorithms whose performance guarantee matches the lower bound \citep{Kaufman2016complexity, Garivier2016, degenne2019non}; however, that of contextual BAI has never been elucidated. 

Formally, we consider the following setting. 
At each time $t=1,2,\dots$, an agent observes a context (covariate) $X_t\in\mathcal{X}$ and chooses an arm $A_t \in [K] = \{1,\dots, K\}$, where $\mathcal{X}$ denotes the context space. Then, the agent immediately receives a reward (or outcome) $R_t$ linked to the arm $A_t$. This setting is called the bandit feedback or Rubin causal model \citep{Neyman,Rubin1974}; that is, a reward in round $t$ is $R_t=\sum^K_{a=1}\mathbbm{1}[A_t = a]R_{t, a}$, where $R_{t, a}$ is a potential independent (random) reward. 
We assume that $X_t$ is independent and identically distributed (i.i.d.) over $[T]$ and denote the distribution of $X_t$ by $\zeta$. 
Given the context $x \in \mathcal{X}$, we denote the reward distributions of the potential outcomes as $\vec{p} = (p_{1, x}, p_{2, x}, \ldots, p_{K, x})$ and their means as $\vec{\mu} = (\mu_{1, x}, \mu_{2, x}, \ldots, \mu_{K, x})$. 
Let $\mathcal{V} = (\vec{p}, \zeta)$ (this can be written as $  \nu = ((\mu_{a,x}), (\zeta_x))$ when the rewards follow a distribution that belongs to a single parameter exponential family, and the contexts are finite) be a bandit problem. 
Let $\mathbb{P}_{\set{V}}$ (resp. $\mathbb{P}_{\nu}$ ) and $\mathbb{E}_{\set{V}}$ (resp. $(\mathbb{E}_{\nu})$) be the probability and expectations under model $\set{V}$ (resp. $\nu$), respectively. 
Then, $\mu_a = \mathbb{E}_{X \sim \zeta}[\mu_{a, X}] = \mathbb{E}_{X \sim \zeta}[\mathbb{E}_{\set{V}}[R_{t, a} | X]] = \mathbb{E}_{\set{V}}[R_{t, a}]$ is the average reward marginalized over $\set{X}$. We assume that $\mathcal{V}$ belongs to a class $\Omega = \{ (\vec{p}, \zeta): \exists a^* \in [K] \;s.t.\; \forall a \neq a^*, \mu_{a^*} > \mu_{a} \}$; that is, the best arm $a^*(\set{V}) = \argmax_a \mu_a$ is uniquely defined. 
 Let $p_{a, x}$ and $q_{a, x}$ be two absolutely continuous probability distributions (w.r.t. the Lebesgue measure) of $R_{t, a}$, given $X_{t} = x$. We define the Kullback--Leibler (KL) divergence from $p_{a, x}$ to $ q_{a, x}$ as
\begin{align*}
&\KL(p_{a, x}, q_{a, x}) := \begin{cases}\int_{\mathbb{R}} \log\left(\frac{p_{a, x}}{
q_{a, x}}(r)\right) \mathrm{d}p_{a, x}(r)\ \mathrm{if}\ q_{a,x}\ll p_{a,x},\\
+\infty \quad \text{otherwise}.
\end{cases}
\end{align*}
We assume that for all $(\vec{p}, \zeta),  (\vec{q}, \zeta) \in \Omega$, if $p_{a, x}\neq  q_{a, x}$, then $0 < \KL(p_{a, x}, q_{a, x}) < +\infty$.
For distributions that belong to the single parameter exponential family, we introduce the KL divergence from the distribution with mean $\mu$ to the distribution with mean $\nu$ as $\kl(\mu, \nu)$. Furthermore, for the Bernoulli distributions, we denote the KL divergence by $d(\mu, \nu) =\mu\log (\mu/\nu) + (1-\mu)\log((1-\mu)/(1-\nu))$ with the convention that $d(0, 0) = d(1,1) = 0$.

Let $\mathcal{F}_t = \sigma(X_1, A_1, R_1, \ldots, X_t, A_t, R_t, X_{t+1})$ and $\mathcal{G}_t = \sigma(X_1, A_1, R_1, \ldots, X_t, A_t, R_t)$ be the sigma-algebras generated by the observations up to immediately before the selection of the arm at time $t+1$ and all observations up to time $t$, respectively. The strategy or algorithm of the best arm identification consists of the following three elements: a sampling, stopping, and decision rules. A sampling rule selects from which arm we collect the sample each time based on past observation ($ A_t$ is $\mathcal{F}_{t-1}$-measurable). The stopping rule determines when to stop sampling based on the past observation. We denote $\tau$ as this time; $\tau$ is the stopping time with respect to the filtration $(\mathcal{G}_t)_{t\ge1}$. The decision rule estimates the best arm $\hat{a}_\tau$ based on observation up to time $\tau$ ($\hat{a}_\tau$ is $\mathcal{G}_\tau$-measurable). 

We focus on the fixed confidence setting; that is, with a given admissible failure probability $\delta \in (0, 1)$, the algorithm is guaranteed to have $\mathbb{P}( \hat{a}_\tau \neq \argmax_a \mu_a) \leq \delta$. We define $\delta$-PAC to formalize this property:
\begin{definition}
An algorithm is $\delta$-PAC if for all $\set{V} \in \Omega$, $\mathbb{P}_\set{V}( \hat{a}_\tau^* \neq {a}^*(\set{V}) ) \leq \delta $ and $ \Pr_{\set{V}}(\tau < \infty) = 1$.
\end{definition}
Later, we propose algorithms that are $\delta$-PAC.

We reemphasize that although we can use contextual information, our primary interest is not in the mean reward conditioned on each context. Similar problems are frequently considered in the literature on causal inference that mainly discusses the efficient estimation of causal parameters. The assigned treatment (chosen arm) and observed outcomes for each treatment (reward) and covariate (context) are given therein. Here, we are not interested in the distribution of the covariate; rather we are interested in the estimation of the expected value of the outcome of the treatment marginalized over the covariate distribution; that is, the average treatment effect (ATE) \citep{imbens_rubin_2015}. For this setting, \citet{Laan2008TheCA} and \citet{Hahn2011} proposed experimental design methods to estimate the ATE more efficiently by assigning treatments based on the covariate. According to their results, even if the covariates are marginalized, the variance of the estimator can be reduced with the help of the covariate information. \citet{Karlan2014} applied the method of \citet{Hahn2011} to test how donors respond to new information about the effectiveness of charity. These studies have been attempted to be improved by \citet{Meehan2018} and \citet{kato2020efficienttest}. 

For each $x\in\mathcal{X}$, we define allocations for each arm with the context $x$ as $\Sigma^K_x = \{w_{a, x}\in\mathbb{R}^k_+: w_{1, x} + \cdots + w_{K, x} = 1\}$. Let $\set{W}$ be all possible such allocations.
We denote by $N_x(t)$ and $N_{a, x}(t)$ the number of times we observe context $x$, and we choose arm $a$ given context $x$; that is, $N_x(t) = \sum_{s=1}^t \indicator\{ X_s = x\}$ and $N_{a, x}(t) = \sum_{s=1}^t \indicator\{X_s = x, A_s = a\}$, respectively.

\paragraph{Main results.} We briefly summarize our contributions. 

First, we establish the instance-specific lower bound on contextual BAI for both continuous and finite context cases. The derived lower bound formula has smaller sample complexity than that of lower bound formula in \citet{Garivier2016}, suggesting that a faster BAI may be possible.

Then, we propose optimal algorithms for two cases: (i) two-armed Gaussian bandits where the arms and context jointly follow the multivariate normal distribution; and  
(ii) MAB with reward distributions belonging to the single parameter exponential family and finite contexts. 
We prove that the sample complexity upper bounds of the proposed algorithms asymptotically match the lower bounds.

\paragraph{Organization.} This paper is organized as follows. In Section~\ref{sec:lower_bounds}, we derive the general instance-specific lower bounds for contextual BAI for a case with continuous contexts. Then, in Section~\ref{sec:cont_2arm}, we discuss an optimal algorithm for two-armed Gaussian bandits with continuous contexts.
Section~\ref{sec:lower_bound_finite} focuses on the lower bound when the number of contexts is finite, and the reward distributions are from the single parameter exponential family. In Section~\ref{sec:optimal}, for the finite context case, we obtain the optimal allocations for each pair of contexts and actions by simplifying the lower bound formula.
In Section~\ref{sec:trak_stop}, in the same setting of Section~\ref{sec:lower_bound_finite}, we show an optimal algorithm.
We describe details of the sampling, stopping, and decision rules that are the core of the proposed algorithm and demonstrate that the algorithm is $\delta$-PAC. 
We further confirm that the sample complexity of the proposed algorithm is asymptotically optimal. Section~\ref{sec:experiments} presents the results of our numerical experiments.

\paragraph{Related work.} The stochastic MAB problem is a classical abstraction of the sequential decision-making problem \citep{Thompson1933,Robbins1952,Lai1985}. BAI is a paradigm of the MAB problem, where we consider pure exploration to find the best arm. Several strategies and efficiency metrics have been proposed for BAI \citep{bechhofer1968sequential,Paulson1964,Mannor2004,EvanDar2006,Bubeck2011,Gabillon2012,Karnin2013,Garivier2016,Jamieson2014}. BAI with linear bandits \citep{Soare2014,Xu2018,Tao2018,Fiez2019,jedra2020optimal}, BAI with multiple queries, and the partition identification problem \citet{Juneja2019} are different directions for the generalization of BAI. 

Our setting is a generalization of BAI without contextual information. We can use the side information (explicitly or implicitly) at each round. There have been limited studies that address pure exploration in contextual bandits. \citet{Tekin2015}, \citet{GuanJiang2018}, and \citet{Deshmukh2018} also consider BAI with contextual information; however, they do not discuss the instance-specific optimality. After this study, \citet{Qin2022} also considers a related topic.

From the causal inference perspective, contextual BAI is closely related to a (semiparametric) experimental design for efficient ATE estimation \citep{Laan2008TheCA,Hahn2011,Karlan2014,Athey2016,Meehan2018}. The goal of efficient ATE estimation by adaptive experimentation is often in choosing the best treatment (arm) via hypothesis testing. Therefore, it can be considered as a case where the proposed method should be applied, especially when there are multiple treatments (arms).

\citet{Russac2021} also addresses a similar problem independently of us. Their problem setting is the same as ours in that they can observe discrete contexts. However, they are considering a slightly different problem than best arm identification, i.e., A/B/n testing, where they consider the comparison with a designated control arm. In that problem setting, optimal allocation is uniquely obtained, and they do not have to consider multiple candidates of optimal allocations as we do. Besides, we also derive the result for the case of continuous contexts, which they do not address. On the other hand, they discuss the problem more generally by considering four situations, (a) active mode, (b) proportional mode, (c) agnostic mode, and (d) oblivious mode, depending on how the decision is made. The (b) proportional mode discussed by them is closer to the setting discussed in this paper. In these senses, our results and theirs, while similar, are independent and parallel, and correspond to complementary studies.

\section{General Non-Asymptotic Lower Bounds}\label{sec:lower_bounds}
In this section, we provide the instance-specific sample complexity lower bounds for general contextual BAI.  The proof is based on standard change-of-measure arguments \citep{Kaufman2016complexity}. However, the derivations must consider the possibly continuous context distributions, which are non-trivial.

 Based on the lower bound, we find that the contextual information either helps or does not harm the BAI. Our result is the same as those of existing studies on fixed-confidence BAI without contextual information, except that we can obtain help from the existence of the contextual information. At first glance, it does not necessarily seem advantageous to use contextual information as the marginalized mean reward is not directly related to the contextual information. However, the lower bound with contextual information (see Section~\ref{sec:lower_bounds}) is strictly lower than the sample complexity derived by \citet{Kaufman2016complexity} and \citet{Garivier2016}.

Assume $\set{X} = \mathbb{R}$. Then, we present the non-asymptotic sample complexity lower bound. 
\begin{theorem}
\label{thm:lower_bound_continuous} Let $\delta \in (0, 1/2)$. Assume that for all $x \in \mathbb{R}$, distributions $p_{1,x}, \ldots, p_{K,x}$ are absolutely continuous with respect to the Lebesgue measure.
Let $\delta\in(0,1/2)$. Then, for any $\delta$-PAC strategy, for any $\mathcal{V} = (\vec{p}, \zeta) \in \Omega$,
\begin{align*}
&\Ebb_{\set{V}}[\tau_\delta]\geq T^\star(\mathcal{V}) d(\delta, 1-\delta),
\end{align*}
where 
\begin{align*}
 &T^\star(\mathcal{V}):= \left(\sup_{\bm{w} \in\set{W}} \inf_{(\vec{q}, \zeta)\in\Alt(\mathcal{V})}\sum^K_{a=1} \int_{\mathbb{R}}w_{a, x}\KL(p_{a, x}, q_{a, x}) \zeta(x) \mathrm{d}x \right)^{-1}.
 \end{align*}
\end{theorem}
We provide the proof of Theorem~\ref{thm:lower_bound_continuous} in Appendix~\ref{appdx:proof_lower_bound_continuous}.

\paragraph{Efficiency gains from the context use.} 
In Figure~\ref{fig:efficiency_gain}, we illustrate the efficiency gain by using contextual information. We consider a two-armed, one-dimensional context $X_t\in\mathbb{R}$. Suppose that $(R_{t,1} \  R_{t,2} \  X_t)^\top$ follows a multivariate normal distribution with mean vector $(1\ 0\ 0)^\top$. We assume that the variances of $R_{t,1}$, $R_{t,2}$, and $X_t$ are $1$. We investigate the variation in the theoretical sample complexity by varying the correlation coefficients between $X_t$ and $R_{t,1}$ and $X_t$ and $R_{t,2}$, which are denoted as $\rho_{1\mathcal{X}}\in[0,1]$ and $\rho_{2\mathcal{X}}\in[0,1]$, respectively. Note that we omit the other domains due to symmetry with the current domain. Note that when ignoring (marginalizing) the context, arm $1$ follows $\mathcal{N}(1, 1)$ and arm $2$ follows $\mathcal{N}(0, 1)$, where $\mathcal{N}(\mu, \sigma^2)$ denotes a normal distribution with a mean $\mu$ and variance $\sigma^2$.  Here, for $\delta=0.05$, we calculate the sample complexity lower bounds of the standard setting of BAI from the result of \citet{Garivier2016} and those of the contextual case from our results. We denote the former as $\ell$ and the latter as $\tilde{\ell}$. Then, we compute the sample complexity gain ($1 - \tilde{\ell}/ \ell$) for different pairs of $(\rho_{1\mathcal{X}}, \rho_{2\mathcal{X}})$ and illustrate it in Figure~\ref{fig:efficiency_gain}.
\begin{figure}[htb]
  \begin{center}
    \includegraphics[height=52mm]{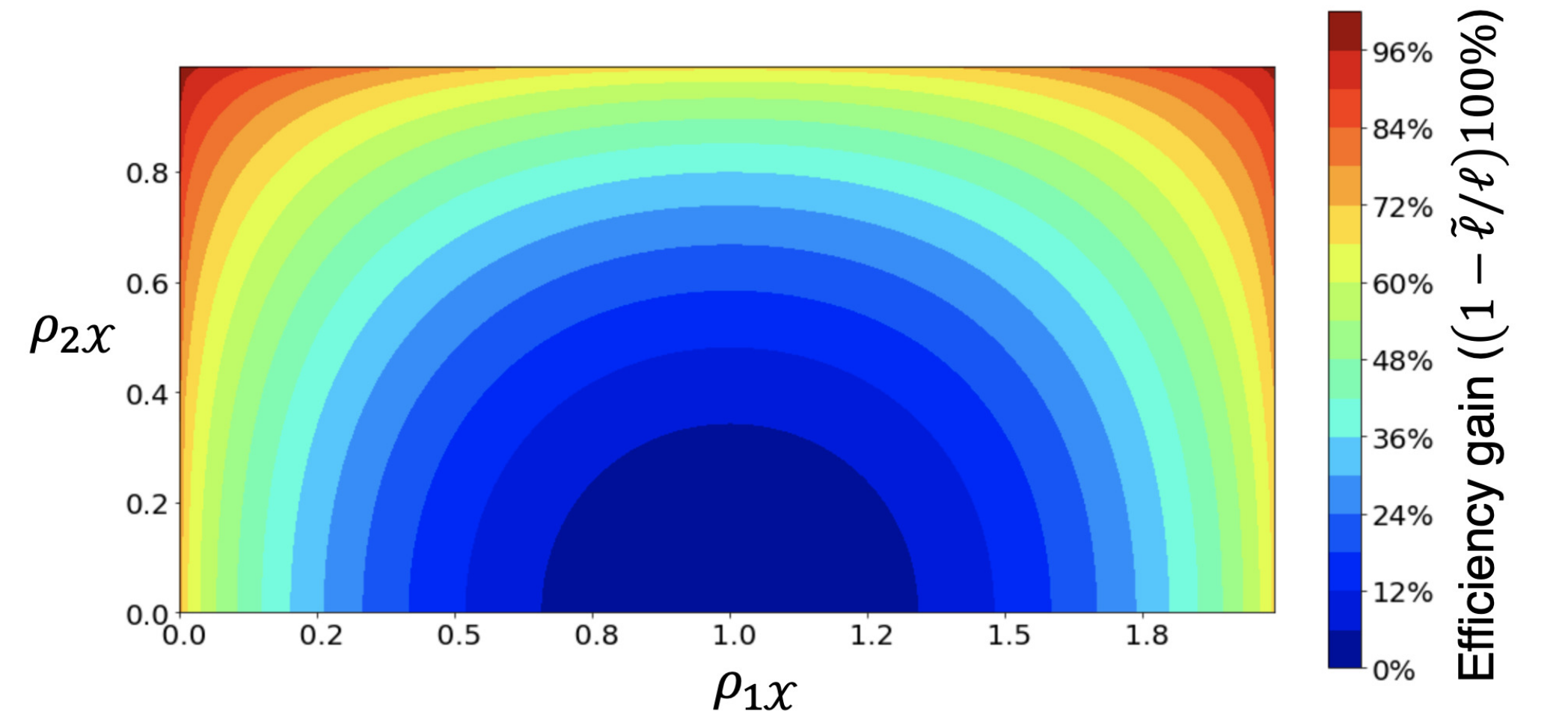}
  \end{center}
  \vspace{-0.3cm}
  \caption{Sample complexity gains through context. The $x$ axis denotes $\rho_{1\mathcal{X}}\in[0,1]$ and the $y$ axis denotes $\rho_{2\mathcal{X}}\in[0,1]$. The contour lines indicate the sample complexity gains: $(1 - \tilde{\ell}/ \ell)100\%$.
  }
  \label{fig:efficiency_gain}
\end{figure}

\section{Two-armed Gaussian Bandits with Continuous Context}\label{sec:cont_2arm} 
In this section, we provide an example for the case of continuous contexts and prove the upper bound of the sample complexity. 
We consider the following two-armed bandit problem. For each $t$, $R_{t,1}$, $R_{t,2}$, and $X_t\in\mathbb{R}$ are drawn from the following Gaussian distributions $\set{N}(\mu_1, \sigma^2_1)$, $\set{N}(\mu_2, \sigma^2_2)$, and $\set{N}(\mu_{\set{X}}, \sigma_{\set{X}}^2)$, respectively ($\mu_1 > \mu_2$). Assume that the vector $(R_{t,1}, R_{t,2}, X_t)$ forms a multivariate Gaussian distribution. We denote $ \textnormal{Cov}(R_{t,1}, X_t) = \sigma_{\set{X}1}$ and $ \textnormal{Cov}(R_{t,2}, X_t) = \sigma_{\set{X}2}$. Suppose the algorithm knows that $(R_{t,1}, R_{t,2}, X_t)$ form a multivariate Gaussian distribution, knows the values of $\sigma^2_1$, $\sigma^2_2$, $ \mu_{\set{X}}$, $ \sigma_{\set{X}}$, $\sigma_{\set{X}1}$, and $ \sigma_{\set{X}2}$, and does not know the values of $\mu_1$ and $\mu_2$. 
Let $ \tilde{\Omega}$ be a set of all such problems. Given an observation $X_t = x$, we have conditional distributions of $R_{t,1}$ and $R_{t,2}$ where for each $a \in \{1, 2\}$, $
 R_{t,a} \sim \set{N}\left(\mu_a + \frac{\sigma_{\set{X}a}}{\sigma_{\set{X}}^2}(x - \mu_{\set{X}}), \sigma^2_a - \frac{\sigma_{\set{X}1}^2}{\sigma_{\set{X}}^2}\right) =\set{N}\left(\mu_a + \frac{\rho_{\set{X}a} \sigma_a}{\sigma_{\set{X}}}(x - \mu_{\set{X}}), \sigma^2_a(1 - \rho^2_{\set{X}a})\right)$. Here, $\rho_{\set{X}a}$ is the correlation coefficient between the context and arm $a\in\{1,2\}$. 
We denote $\sigma_1'^2 = \sigma^2_1 - \frac{\sigma_{\set{X}1}^2}{\sigma_{\set{X}}^2}$ and $ \sigma_2'^2 = \sigma^2_2 - \frac{\sigma_{\set{X}2}^2}{\sigma_{\set{X}}^2}$. From our lower bound in Theorem~\ref{thm:lower_bound_continuous}, we can derive the following lower bound for this specific problem. We give the proof in Appendix~\ref{appdx:thm:cont_2arm_lowerbound}.
\begin{theorem}\label{thm:cont_2arm_lowerbound}
Let $\delta\in(0,1/2)$. For any $\delta$-PAC strategy and $\set{V} \in \tilde{\Omega}$, we have 
$$
\Ebb_{\set{V}}[\tau_\delta]\geq \frac{2(\sigma'_1+ \sigma'_2)^2}{(\mu_1 - \mu_2)^2}d(\delta, 1-\delta).
$$
\end{theorem}

Note that when $\sigma_{\set{X}1}^2>0$ or $\sigma_{\set{X}2}^2>0$, $ \sigma_1 + \sigma_2> \sigma_1' + \sigma_2'$; that is, the value of the lower bound derived in Theorem~\ref{thm:cont_2arm_lowerbound} is strictly smaller than that of the lower bound derived by \citet{Kaufman2016complexity}, $\frac{2(\sigma_1+ \sigma_2)^2}{(\mu_1 - \mu_2)^2}\kl(\delta, 1-\delta)$. Let $\alpha = {\sigma_1'}/{(\sigma_1' + \sigma_2')}$. We also note that the simple $\alpha$-elimination algorithm by \citet{Kaufman2016complexity} with $\alpha$ achieves the lower bound as well as a strictly better sample complexity than that given in \citet{Kaufman2016complexity}. We give the proof in Appendix~\ref{appdx:thm:cont_2arm_upperbound}.
\begin{theorem}\label{thm:cont_2arm_upperbound}
 If $\alpha = {\sigma_1'}/{(\sigma_1' + \sigma_2')}$, then the $\alpha$-elimination strategy using the exploration rate $\beta(t,\delta)=\log\frac{t}{\delta} + 2 \log\log(6t)$ is $\delta$-PAC on $\tilde{\Omega}$ and for every $\set{V} \in \tilde{\Omega}$ and $\epsilon>0$, satisfies
$$\EXP_\set{V}[\tau_\delta] \leq (1+\epsilon)\frac{2(\sigma_1' + \sigma_2')^2}{(\mu_1-\mu_2)^2}\log\left(\frac{1}{\delta}\right) + {o}\left(\log\left(\frac{1}{\delta}\right)\right).$$
\end{theorem}
Hence, $\alpha$-elimination is optimal for this problem. The details of $\alpha$-elimination with contextual information is shown in Appendix~\ref{appdx:alpha_elim}. The pseudo-code is shown in Algorithm~\ref{AlgoBox:Elimination}. Thus, apparently irrelevant contextual information improves optimal sample complexity.

\begin{algorithm}[t]
\SetAlgoLined
\KwIn{Confidence level $\delta$, threshold $\beta(t,\delta)$, $\sigma_{a}$, $\sigma_{\set{X}}$, $\rho_{\set{X}a}$,  $\mu_{\set{X}}$.}
 {\bf Initialization:}  $t=0$, $N_x(0) = 0$, $N_{a,x}(0) = 0$. $\hat{\mu}_1(0)=\hat{\mu}_2(0)=0$,
    $\sigma^2_{0}(\alpha)=1$.\\
$\sigma_1'^2 \gets \sigma^2_1 - \frac{\sigma_{\set{X}1}^2}{\sigma_{\set{X}}^2}$, $ \sigma_2'^2 \gets \sigma^2_2 - \frac{\sigma_{\set{X}2}^2}{\sigma_{\set{X}}^2}$.
\\
$\alpha \gets  {\sigma_1'}/{(\sigma_1' + \sigma_2')}$
\\
 \DontPrintSemicolon
 \While{$|\hat{\mu}_1(t) - \hat{\mu}_2(t)|
      \leq \sqrt{2\sigma^2_{t}(\alpha)\beta(t,\delta)}$}{
 $t\leftarrow t+1$.\\
 Observe $X_t$.\\
 \eIf{$\lceil\alpha t \rceil = \lceil \alpha(t-1) \rceil$}{$A_t\leftarrow 2$}{$A_t\leftarrow 1$}
 Observe $R_t$.\\
 $\hat{\mu}_1(t) = \frac{1}{\sum^t_{s=1} \mathbbm{1}[A_s = 1]}\sum^t_{s=1}\left( R_{s, 1} - \frac{\rho_{\set{X}1} \sigma_1}{\sigma_{\set{X}}}(X_s - \mu_{\set{X}})\right)\mathbbm{1}[A_s = 1]$. \\
    $\hat{\mu}_2(t) = \frac{1}{\sum^t_{s=1} \mathbbm{1}[A_s = 2]}\sum^t_{s=1}\left( R_{s, 2} - \frac{\rho_{\set{X}2} \sigma_2}{\sigma_{\set{X}}}(X_s - \mu_{\set{X}})\right)\mathbbm{1}[A_s = 2]$.\\ 
 Compute $\sigma_t^2(\alpha)=\sigma_1^2/\lceil\alpha
    t\rceil + \sigma_2^2/(t-\lceil \alpha t \rceil)$.\\
  }
 \Return { $\argmax_{a=1,2} \ \hat{\mu}_a(t)$}
 \caption{$\alpha$-elimination with contextual information}
 \label{AlgoBox:Elimination}
\end{algorithm}

\section{Lower Bound with Finite Contexts}
\label{sec:lower_bound_finite}
Although we derived the optimal algorithm for BAI with continuous contexts in the previous sections, it requires some assumptions that may not be practical, e.g., multivariate normal distribution and known variance. We also consider a more practical algorithm by considering BAI with finite contexts. In this section, we consider a lower bound when the number of contexts is finite. For $\nu = ((\mu_{a,x})_{a,x}, (\zeta_x))$, we suppose that $\set{X}$ is finite ($\zeta$ follows the multinomial distribution), and for each arm $a$ and context $x$, arm distribution belongs to the canonical one-parameter exponential family  \citep{Cappe2013,Kaufman2016complexity,Garivier2016,Juneja2019}:
\begin{align}
\label{eq:spef}
    \mathcal{P} = \left\{(p_\pi)_{\pi\in\Pi}: \frac{\mathrm{d}p_\pi}{\mathrm{d}\lambda}=\exp(\pi u - b(\pi))\right\},
\end{align}
where $\lambda$ is some reference measure on $\mathbb{R}$, $b:\Pi\mapsto \mathbb{R}$ is a convex, twice differential function, and $\Pi \subset \mathbb{R} $ is a parameter space. Note that a distribution $p_\pi \in \mathcal{P}$ can be parameterized by its mean $\dot{b}(\pi)$. As discussed in \citet{Cappe2013,Garivier2016}, the KL divergence from $p_\pi$ to $p_{\pi'}$ is given by
\begin{align*}
    \KL(p_\pi, p_{\pi'}) = \kl (\dot{b}(\pi), \dot{b}(\pi')) = b(\pi') - b(\pi) - \dot{b}(\pi)(\pi' - \pi).
\end{align*}
For each arm $a$ and context $x$ pair, we represent the unique distribution in $\mathcal{P}$ by $(\mu_{a,x})$. We further  write the multinomial contextual distribution by $(\zeta_x)$.

We denote by $\Theta$ a set of BAI problems with finite contexts and the single parameter (canonical) exponential family. The lower bound is given in the following theorem.
\begin{theorem}
\label{thm:lower_bound}  Let $\delta\in(0,1/2)$. For any $\delta$-PAC strategy and any $\mathcal{V} = ((\mu_{a,x}), (\zeta_x))\in \Theta$,
\begin{align*}
\Ebb_{\nu}[\tau_\delta] &\geq T^\star(\nu ) d(\delta, 1-\delta),
\end{align*}
where 
\begin{align*}
T^\star(\nu)^{-1}  \coloneqq \sup_{\bm{w}\in\set{W}}\inf_{((\lambda_{a,x}), (\zeta_x))\in\Alt(\nu)}\sum_{x \in \set{X}}\zeta_x\sum^K_{a=1}w_{a, x}\kl(\mu_{a,x}, \lambda_{a,x}).
\end{align*}
\end{theorem}
As an intuition behind $T^\star(\nu)$, the probability of misidentification is roughly $\exp (-\tau \left(T^\star(\nu)\right)^{-1})$; that is, larger $\left(T^\star(\nu)\right)^{-1}$ means a strategy with smaller sample complexity.

We note specific properties of this lower bound. From the results in \citet{Garivier2016}, we know that when the optimal arm is unique, the expected value of the sampling budget of the optimal BAI algorithm does not diverge; rather it is less than or equal to the order of $\log(1/\delta)$. Therefore, from the assumption of the proposed model, $T^\star(\nu)$ is finite under certain regularity conditions; for example, the context marginalized distribution of the reward $ R_t$ is sub-Gaussian.
 
To derive the lower bound, we show the following lemma, which is an extension of Lemma~1 of \citet{Kaufman2016complexity}. 

\begin{lemma}\label{lem:kauf_lemma_extnd_finite}
Let $N_{a, x}(\tau) = \sum_{t=1}^\tau \indicator\{X_t = x, A_t = a\}$. Let $\nu = ((\mu_{a, x}), \zeta),  \nu' = ((\lambda_{a,x}), \zeta) \in \Theta$. 
For any almost surely finite stopping time $\tau$ with respect to $(\mathcal{G}_{t})_{t \ge 1}$,
\begin{align*}
& \sum_{x \in \set{X}}\sum_{a \in [K]} \EXP_{\nu}[N_{a, x}(\tau)] \kl(\mu_{a, x}, \lambda_{a, x}) \geq \sup_{\Ecal\in\mathcal{G}_\tau} d(\Pbb_{\nu}(\Ecal), \Pbb_{\nu'}(\Ecal)).
\end{align*}
\end{lemma}
The proof is provided in Appendix~\ref{appdx:kauf_lemma_extnd_finite}. Here, we offer the proof sketch of Theorem~\ref{thm:lower_bound} as follows.
\paragraph{Proof sketch.} 
From Lemma~\ref{lem:kauf_lemma_extnd_finite} with $\set{E} =\{\hat{a}_\tau = {a}^*(\nu)\}$, for each $\nu \in \Theta$ and $\nu' \in \Alt(\nu)$, we have
\begin{align*}
\sum_{x \in \set{X}} \sum_{a \in [K]} \EXP_{\nu} [N_{a, x}(\tau)] \kl(\mu_{a, x}, \lambda_{a, x}) & \ge d(\Pr_{\nu}(\set{E}), \Pr_{\nu'}(\set{E})) \ge d(\delta, 1- \delta),
\end{align*}
where, for the last inequality, we use the definition of the $\delta$-PAC algorithm and monotonicity of the KL divergence. Then, for each $\nu \in \Theta$, for some $(w_{a, x})_{a \in [K], x \in \set{X}} \in \set{W}$, we can obtain $\kl(\delta, 1- \delta) \le \EXP_{\nu}[\tau_\delta] \sup_{w \in \set{W}} \inf_{((\lambda_{a,x}), \zeta) \in \Alt(\nu)} \sum_{x \in \set{X}} \zeta_x \sum_{a \in [K]} w_{a, x} \kl(\mu_{a, x}, \lambda_{a, x})$.

In Section~\ref{sec:efficiency_gain}, we explain that the lower bound with contextual information is smaller than or equal to the lower bound without contextual information shown by \citet{Garivier2016}.

\section{Optimal Allocation in Contextual BAI with Finite Contexts}
\label{sec:optimal}
In this section, we first provide a simplification of the lower bound derived in Section~\ref{sec:lower_bounds}. Then, we examine the characteristics of the optimal allocations used in the proof. It becomes apparent that the set of optimal allocations is, in general, not unique. Therefore, we define the notion of convergence to the set and prove that the estimated optimal allocations converge to the set of optimal allocations (even though they might not converge to a point). 

\subsection{Simplification of the Lower Bound}
Without loss of generality, let $a^*(\nu) = 1$. First, we show a simpler equivalence form for the optimization problem $T^\star(\nu)^{-1}$ in the following theorem. 
\begin{lemma}\label{lem:CDinterpretation} 
For each $ w \in \set{W}$, we have
\begin{align}
& \inf_{((\lambda_{a,x}), \zeta)\in\Alt(\nu)}\sum_{x\in \set{X}}\zeta_x\sum^K_{a=1}w_{a, x}\kl(\mu_{a,x}, \lambda_{a,x}) \nonumber\\
\label{eq:mininf_problem}
& = \min_{a\neq 1} \inf_{ \sum_{x \in \set{X}} \zeta_x \lambda_{a, x} > \sum_{x \in \set{X}} \zeta_x \lambda_{1, x}} \sum_{x \in \set{X}}\zeta_x\Big(w_{1, x}\kl(\mu_{1,x}, \lambda_{1,x}) + w_{a, x}\kl(\mu_{a,x}, \lambda_{a,x})\Big)
\end{align}
\end{lemma}
We provide the proof in Appendix~\ref{appdx:CDinterpretation}.

Moreover, we can further simplify the constraint in the minimization problem. We define 
\[f_a((\lambda_{x,a})) = \sum_{x \in \set{X}}\zeta_x\Bigg\{ w_{1, x}\kl(\mu_{1, x}, \lambda_{1, x}) + w_{a, x}\kl(\mu_{a, x}, \lambda_{a, x})\Bigg\}.\]
Then, we show the following lemma.
\begin{lemma}\label{lem:allEqual}
For each $w \in \set{W}$, suppose that $(\lambda_{a,x}^*)$ satisfies:
\[\min_{a\neq 1} \inf_{ \sum_{x \in \set{X}} \zeta_x \lambda_{a, x} > \sum_{x \in \set{X}} \zeta_x \lambda_{1, x}} f_a((\lambda_{x,a})) = \min_{a\neq 1} f_a((\lambda_{x,a}^*)).\]
For all $a^* \in \argmin_{a \in [K]}\inf_{ \sum_{x \in \set{X}} \zeta_x \lambda_{a, x} > \sum_{x \in \set{X}} \zeta_x \lambda_{1, x}} f_{a}((\lambda_{x,a}))$, we have 
\[\sum_{x \in \set{X}}\zeta_x \lambda^*_{a,x} =\sum_{x \in \set{X}}\zeta_x \lambda^*_{1,x}.\]
Consequently, we can equivalently write the optimization problem as
\begin{align}
\label{eq:inner_opt}
T^\star(\nu)^{-1} &=\max_{\vec{w} \in \set{W}}\min_{a\neq 1}
L_{1,a}((\mu_{1,x}, \mu_{a,x}, \zeta_x, w_{1,x}, w_{a,x})_{x\in\set{X}}),
\end{align}
where for $a, b\in[K]$,
\begin{align*}
    &L_{a,b}((\mu_{a,x}, \mu_{b,x}, \zeta_x, w_{a,x}, w_{b,x})_{x\in\set{X}})\\
    &= \min_{ \sum_{x \in \set{X}} \zeta_x \lambda_{b, x} = \sum_{x \in \set{X}} \zeta_x \lambda_{a, x}} \sum_{x \in \set{X}}\zeta_x\Bigg\{ w_{a, x}\kl(\mu_{a, x}, \lambda_{a, x}) + w_{b, x}\kl(\mu_{b, x}, \lambda_{b, x})\Bigg\}
\end{align*}
\end{lemma}
We provide the proof in Appendix~\ref{appdx:allEqual}.

\subsection{Characteristics of the Lower Bound}
\label{sec:characteristic}

Let $2^{\set{W}}$ be a power set of ${\set{W}}$. We define a point-to-set map $\Phi: \Theta \to 2^{\set{W}}$; that is, the set of all optimal allocations for the bandit problem $\nu$ as 
$$
\Phi(\nu) = \Big\{\bm w \in \mathcal{W} \mid m(\bm w, \nu) = \max_{\vec{w}' \in \set{W}} m(\vec{w}', \nu)\Big\},
$$
where 
\begin{align*}
&m(\bm w, \nu) = \min_{a \neq 1}\min_{ \sum_{x\in\mathcal{X}} \zeta_x \lambda_{a,x} = \sum_{x\in\mathcal{X}} \zeta_x \lambda_{1,x}}\sum_{x\in\mathcal{X}}\zeta_x\Bigg\{ w_{1, x}\kl(\mu_{1, x}, \lambda_{1, x}) + w_{a, x}\kl(\mu_{a, x}, \lambda_{a, x})\Bigg\}.
\end{align*}
The interpretation of $m(\bm w, \nu)$ is that, unlike the corresponding part in \citet{Garivier2016}, we can further minimize the lower bound by choosing an optimal allocation from a wider domain than the case without contextual information as long as the constraints are satisfied. For example, let us consider a case where two arms $a$ and $b$, and two contexts $1$ and $2$ are given. Here, under certain circumstances, one needs to think about saving the allocations to arm $a$ in context $1$, allocating more to arm $a$ in context $2$, and get more budget to arm $b$ in context $1$. Thus, solving $m(\bm w, \nu)$ is inherently different from optimizing the allocations separately for each context; that is, a case where we apply a BAI algorithm without contextual information for each discrete context such as \citet{Garivier2016}.

From this simplified formula of the lower bound, we obtain the following lemmas. We provide the proofs in Appendix~\ref{appdx:conti_V}-\ref{appdx:continui_w}.
\begin{lemma}\label{lem:conti_V}
Fix $\vec{w} \in \set{W}$. We regard $\nu$ as a point in $\mathbb{R}^{|\set{X}|(K + 1)}$: $\nu= ((\mu_{a,x}), \zeta) \in \mathbb{R}^{|\set{X}|(K + 1)}$. Then, $m(\bm w, \nu) $ is continuous at every $\nu \in \Theta$. 
\end{lemma}
Note that the reason why $\nu$ is in $\mathbb{R}^{|\set{X}|(K + 1)}$ is that we include $\zeta\in\mathbb{R}^{|\mathcal{X}|}$ in $\nu$ with $(\mu_{a,x})\in\mathbb{R}^{|\mathcal{X}|K}$.

\begin{lemma}\label{lem:continui_w}
We fix $\nu \in \Theta$. Then, $m(\bm w, \nu) $ is continuous at every $\vec{w} \in \set{W}$. 
\end{lemma}

The set of the optimal allocations is not, in general, unique. Therefore, we introduce the notion of convergence, where the metric is defined as the minimum distance from the point to the set. 
\begin{definition}
Let $(\vec{w}_k)_{k\ge1} = ((w_{a, x}^{(k)}))_{k\ge1}$ be a sequence of points in $\set{W}$. Let $ \bar{\set{W}} \subset \set{W}$. We say $(\vec{w}_k)_{k\ge1}$ {\it converges to } $\bar{\set{W}}$ if for any $\varepsilon > 0 $, there exists $n(\varepsilon) \in \mathbb{N}$ subject to for all $k \ge n(\varepsilon)$, 
\begin{align*}
\inf_{(w_{a,x}) \in \bar{\set{W}}} \max_{a, x}| w_{a,x}^{(k)} - w_{a,x}| < \varepsilon.
\end{align*}
\end{definition}
Using this definition of convergence, we obtain the following lemmas. We provide the proofs in Appendix~\ref{appdx:converge_w}--\ref{appdx:aloca_convex}.
\begin{lemma}\label{lem:converge_w}
Let $(\nu^k = ((\mu_{a, x}^{(k)}), \zeta_x^{(k)}))_{k \ge 1}$ be a sequence converging to $\nu$. Construct a sequence $ (\vec{w}_k)_{k \ge 1}$ such that $\vec{w}_k \in \Phi(\nu^k)$. 
Then $\vec{w}_k$ converges to $\Phi(\nu)$. 
\end{lemma}

\begin{lemma}\label{lem:aloca_convex}
 The set of all optimal allocations for the bandit problem $\nu$, $\Phi(\nu)$, is convex. 
\end{lemma}

\subsection{Efficiency Gain}
\label{sec:efficiency_gain}
Here, we show that the lower bound with contextual information is smaller than or equal to the lower bound without contextual information shown by \citet{Garivier2016}. For simplicity of discussion, we consider a two-armed bandit case. Let us denote the lower bound without contextual information by $\Gamma^\star(\nu)\kl(\Pbb_{\nu}(\Ecal), \Pbb_{\nu'}(\Ecal))$, where $\Gamma^\star(\nu)$ is defined as the same quantity as $T^\star(\bm{\mu})$ in \citet{Garivier2016}. Let us also denote the optimal allocation in \citet{Garivier2016} by $\gamma^*_1$ and $\gamma^*_2$ and one of the optimal allocations in ours by $(w^*_{1, x})$ and $(w^*_{2, x})$. Then, $\Gamma^\star(\nu)^{-1} \leq T^\star(\nu)^{-1}$ holds as follows:
\begin{align*}
    &\Gamma^\star(\nu)^{-1}\\
    &= \min_{\lambda_1 = \lambda_2}\Bigg\{ \gamma^*_{1}\kl(\mu_{1}, \lambda_{1}) + \gamma^*_{2}\kl(\mu_{2}, \lambda_{2})\Bigg\}\\
    &= \min_{\lambda_1 = \lambda_2}\Bigg\{ \gamma^*_{1}\kl\left(\sum_{x\in\mathcal{X}}\zeta_x\mu_{1, x}, \lambda_{1}\right) + \gamma^*_{2}\kl\left(\sum_{x\in\mathcal{X}}\zeta_x\mu_{2, x}, \lambda_{2}\right)\Bigg\}\\
    &= \min_{\sum_{x\in\mathcal{X}}\zeta_x\lambda_{1,x} =\sum_{x\in\mathcal{X}}\zeta_x\lambda_{2,x}}\Bigg\{ \gamma^*_{1}\kl\left(\sum_{x\in\mathcal{X}}\zeta_x\mu_{1, x}, \sum_{x\in\mathcal{X}}\zeta_x\lambda_{1,x}\right) + \gamma^*_{2}\kl\left(\sum_{x\in\mathcal{X}}\zeta_x\mu_{2, x}, \sum_{x\in\mathcal{X}}\zeta_x\lambda_{2,x}\right)\Bigg\}\\
    &\stackrel{(a)}{\leq} \min_{\sum_{x\in\mathcal{X}}\zeta_x\lambda_{1,x} =\sum_{x\in\mathcal{X}}\zeta_x\lambda_{2,x}}\sum_{x\in\mathcal{X}}\zeta_x\Bigg\{ \gamma^*_{1}\kl\left(\mu_{1, x}, \lambda_{1,x}\right) + \gamma^*_{2}\kl\left(\mu_{2, x}, \lambda_{2,x}\right)\Bigg\}\\
    &\leq \min_{\sum_{x\in\mathcal{X}}\zeta_x\lambda_{1,x} =\sum_{x\in\mathcal{X}}\zeta_x\lambda_{2,x}}\sum_{x\in\mathcal{X}}\zeta_x\Bigg\{ w^*_{1,x}\kl\left(\mu_{1, x}, \lambda_{1,x}\right) + w^*_{2,x}\kl\left(\mu_{2, x}, \lambda_{2,x}\right)\Bigg\} 
    \\
    &= T^\star(\nu)^{-1},
\end{align*}
where for $(a)$, we use the convexity of the KL divergence. 
Next, we discuss when the equality holds. For brevity, we consider a case with only two contexts. Let us denote the optimal $\lambda_1$ in the case without contextual information by $\lambda^*_1$ (note that $\lambda_1 = \lambda_2$) and the optimal $(\lambda_{1,1}, \lambda_{1,2})$ and $(\lambda_{2,1}, \lambda_{2,2})$ in the case with contextual information by $(\lambda^*_{1,1}, \lambda^*_{1,2})$ and $(\lambda^*_{2,1}, \lambda^*_{2,2})$. Then, the equality holds only if the following three conditions simultaneously hold:
\begin{itemize}
    \item $\lambda^*_1 = \zeta_1\lambda^*_{1,1} = \zeta_2\lambda^*_{1,2} = \zeta_1\lambda^*_{2,1} = \zeta_2\lambda^*_{2,2}$;
    \item $\frac{\mu_{1,1}}{\lambda^*_{1,1}} = \frac{\mu_{1,1}}{\lambda^*_{1,2}}$ and $\frac{\mu_{2,1}}{\lambda^*_{2,1}} = \frac{\mu_{2,2}}{\lambda^*_{2,2}}$;
    \item $\gamma^*_1 = \gamma^*_{1,1} = \gamma^*_{1,2}$.
\end{itemize}
We believe that it is difficult to summarize these conditions in a simpler form, but except for cases where the expected reward does not change among contexts, situations satisfying these conditions are extremely limited.

\section{Contextual Track-and-Stop Algorithm}
\label{sec:trak_stop}
In this section, we propose an optimal algorithm for contextual BAI, called the Contextual Track-and-Stop (CTS) algorithm for the case of finite context. The strategy is an extension of the Track-and-Stop (TS) algorithm by \citet{Garivier2016} for contextual BAI. We further prove that the proposed algorithm is $\delta$-PAC.

Recall that the optimal algorithm of BAI with fixed confidence \citet{Garivier2016} consists of sampling, stopping, and decision rules. We follow the same path for the contextual BAI. We show the pseudo-code of the proposed CTS algorithm in Algorithm~\ref{algo:CTS}. 
There, the empirical averages $\hat{\mu}_{a,x}(t)$ and $\hat{\zeta}_x(t)$ are defined as for each $a \in [K]$ and $x \in \mathcal{X}$, $\hat{\mu}_{a,x}(t) = (\sum_{s=1}^tR_{t}\indicator\{A_s = a, X_s = s\})/N_{a,x}(t)$ and $\hat{\zeta}_x(t) = (\sum_{s=1}^t\indicator\{X_s = s\})/N_x(t)$.
Our procedure is similar to TS with D-tracking, proposed by \citet{Garivier2016}. However, incorporating contextual information is a non-trivial extension of their method. The algorithm consists of sampling, stopping, and decision rules. The detail of the sampling rule is described in the following Section~\ref{sec:sampling}. The stopping rule, in particular, for determining the threshold $\beta(t, \delta)$, is described in Section~\ref{sec:stop} when the reward distributions are Bernoulli and in Section~\ref{sec:stopping_oneexp} when the reward distributions belong to the canonical one-parameter exponential family.

Our proposed algorithm consists of sampling, stopping, and recommendation rules. In the sampling rule, we use the forced exploration, which is an extension of D-tracking of \citet{Garivier2016} and is known to be empirically superior to their C-tracking. 
To estimate the optimal weights, we solve an empirically approximated optimization problem \eqref{eq:inner_opt} by applying optimization solvers directly. Several methods are proposed to solve the maximin problem more efficiently, such as the application of no-regret learning algorithms in \citet{degenne2019non}. However, we cannot use them directly for solving contextual BAI, in which we have a different form of the maximin problem than that of BAI without context. \citet{jedra2020optimal} (BAI with linear models) and \citet{Russac2021} (A/B/n texting with contextual information) also directly solve the maximin problem.  
In the stopping rule, we use the criterion proposed by \citet{Kaufmann2021}, which refines the stopping rule of \citet{Garivier2016}. Then, we recommend an arm with the maximum sample average of the reward.

\begin{algorithm}[t]
\SetAlgoLined
\KwIn{Confidence level $\delta$ and threshold $\beta(t,\delta)$.}
 {\bf Initialization:}  $t=0$, $N_x(0) = 0$, $N_{a,x}(0) = 0$.
 \DontPrintSemicolon
 \\
 \While{$(Z(t) := \max_{a\in[K]}\min_{b\in [K]\backslash\{a\}} Z_{a,b}(t) < \beta(t,\delta)$)}{
 $t\leftarrow t+1$.\\
 Observe $X_t$.\\
 \eIf{$\varphi^g_{X_t,t} = \{ a : N_{a, X_t}(t) < \sqrt{N_{x}(t)} - K/2\} \neq \emptyset$}{
   $a\leftarrow \argmin_{a\in \varphi^g_{X_t, t}} N_{a,X_t}$\;
   }{
   $a\leftarrow \argmax_{a \in [K]}\sum^t_{s=0}\mathbbm{1}[X_s = X_t]w_{a, X_t}(t) - N_{a,X_t}(t).$ }
 Sample arm $a$ and update $N_x(t)$, $N_{a,x}(t)$, $\hat{\zeta}_x(t)$, $\hat{\mu}_{a,x}(t)$, $Z(t)$.\\ 
 $\hat{a}_t = \argmax_{a\in [K]}\sum_{x\in\set{X}} \hat{\zeta}_x(t)\hat{\mu}_{x, a}(t)$.\\
 $\vec{w}(t) \leftarrow \argmax_{\vec{w} \in \set{W}}\min_{a\neq \hat{a}_t} L_{\hat{a}_t,a}((\hat{\mu}_{\hat{a}_t,x}(t), \hat{\mu}_{a,x}(t), \hat{\zeta}_x(t), w_{a,x}, w_{a,x})_{x\in\set{X}})$.\\
  }
 \Return {$\hat{a}_\tau = \hat{a}_t$}
 \caption{CTS algorithm}
 \label{algo:CTS}
\end{algorithm}

\subsection{Sampling Rule}
\label{sec:sampling}
To design an algorithm with minimal sample complexity, the sampling rule should match the optimal proportions of the arm draws; that is, an allocation in the set $\Phi(\nu)$. Because $\mu_{a,x}$ and $\zeta_x$ are unknown, our sampling rule tracks, in round $t$, the optimal allocations in the plug-in estimate $\Phi(\hat{\nu}(t))$, where $\hat{\nu}(t) = ((\hat{\mu}_{a,x}(t))_{a \in [K], x \in \mathcal{X}}, (\hat{\zeta}_x(t))_{x \in \mathcal{X}})$. 

The design of our tracking rule is equivalent to computing a sequence of allocations $(w_{a,x}(t))_{t\ge 1}$. The only requirement we actually impose on this sequence is the following condition: 
\begin{equation}
\label{eq:lazycond}
 \lim_{t\to\infty}\min_{\bm{w}'\in \Phi(\hat{\nu}(t))}\max_{a\in[K], x\in\mathcal{X}}\left|w_{a, x}(t) - w'_{a,x}\right|=0. \qquad \mathrm{a.s.}
\end{equation}
This condition is sufficient to guarantee the asymptotic optimality of the algorithm. We introduce a set $\varphi^g_{x,t} = \{ a : N_{a, x}(t) < \sqrt{N_{x}(t)} - K/2\}$ consisting of the context-action pairs that are poorly explored. Then, in round $t$, after observing a context $X_t\in\mathcal{X}$, our sampling rule $(A_t)$ is sequentially defined as
\begin{align}
\label{eq:forced}
A_{t} \in \begin{cases}
\argmin_{a\in \varphi^g_{X_t, t}} N_{a,X_t} & \mathrm{if}\ \varphi^g_{X_t,t} \neq \emptyset\\
\argmax_{1\leq a \leq K}\sum^t_{s=0}\mathbbm{1}[X_s = X_t]w_{a, X_t}(t) - N_{a,X_t}(t).
\end{cases}
\end{align}

We offer the following lemma under this sampling rule. The proof is provided in Appendix~\ref{appdx:tracking}.
\begin{lemma}\label{lmm:tracking}
Under any sampling rule (\ref{eq:forced}) that satisfies the condition (\ref{eq:lazycond}), 
\[\mathbb{P}_\nu \left(\inf_{\vec{w}^* \in \Phi(\nu)}\lim_{t\to\infty}\max_{a\in[K], x\in\mathcal{X}}\left|\frac{N_{a,x}(t)}{t} - \zeta_xw^*_{a, x} \right| = 0\right) = 1.\]
\end{lemma}
This lemma shows that the sampling rule can keep the allocation close to the optimal allocations. Thus, we can ensure that the sampling rule defined by \eqref{eq:forced} (sampling rule) satisfies \eqref{eq:lazycond} (allocation convergence).

To compute $w_{a,x}(t)$ in \eqref{eq:forced}, we need to solve the minimax problem defined in \eqref{eq:inner_opt} with the estimated parameters. If the number of contexts and arms is very large, it may be difficult to solve. However, except for such an extreme case, we can solve the problem by using minimax optimization based on the convex optimization algorithm in a short time. The computation is similar to that in \citet{jedra2020optimal}.

We remark that the application of the original TS algorithm \citep{Garivier2016} for each context separately is not optimal for contextual BAI. Our problem setting makes finding the best allocations difficult, which is quite different from running BAI in parallel for each context. It is necessary to find good allocations of each arm to the right context, and the allocations among contexts are entangled.
For example, to achieve our derived lower bound, one needs to think about saving the allocations to an arm $a$ in context $1$, then allocating more to arm $a$ in context $2$, and getting more budget to another arm $b$ in context $1$. In contrast, when separately applying the original TS algorithm, we cannot attain such an optimal allocation.

\subsection{Threshold in the Stopping Rule}
\label{sec:stop}
In this subsection, we present the stopping rule, in particular the threshold for the Bernoulli bandit model. We aim to design an algorithm that stops as early as possible while maintaining the failure probability less than or equal to $\delta$. We demonstrate that the stopping rule using the generalized likelihood ratio test (GLRT) for contextual BAI is $\delta$-PAC when the exploration ratio is properly tuned. Such a stopping rule is also known as Chernoff's stopping rule \citep{Chernoff1959}. Although the approach for deriving the threshold is inspired by and similar to that of \citet{Garivier2016}, our computation with the contextual information is more involved.

We consider a case where the reward $R_{t,a}$ follows a Bernoulli distribution conditioned on $X_t = x$. Here, the likelihood is given as
\begin{align*}
&p_{\mu_a}\big((\underline{R}_{a,x}(t)), \underline{X}(t)\big)= \prod_{x\in\mathcal{X}} \big(\zeta_x\mu_{a,x}\big)^{\sum^t_{s=1}\mathbbm{1}[A_s = a, X_s = x, R_s = 1]}\big(\zeta_x(1-\mu_{a,x})\big)^{\sum^t_{s=1}\mathbbm{1}[A_s = a, X_s = x, R_s = 0]}.
\end{align*}
Then, for all pairs of the arms, $a,b\in [K]$, the GLRT statistic is given as
\begin{align*}
Z_{a,b}(t) = \log \frac{\max_{\overline{\xi}_a(t)\geq \overline{\xi}_b(t)} p_{\xi_a}\big((\underline{R}_{a,x}(t)), \underline{X}(t)\big) p_{\xi_b}\big((\underline{R}_{b,x}(t)), \underline{X}(t)\big)}{\max_{\overline{\xi}_a(t)\leq \overline{\xi}_b(t)} p_{\xi_a}\big((\underline{R}_{a,x}(t)), \underline{X}(t)\big) p_{\xi_b}\big((\underline{R}_{b,x}(t)), \underline{X}(t)\big)},
\end{align*}
where $\overline{\xi}_a(t) = \sum_{x \in \set{X}} \hat{\zeta}_x(t)\xi_{a,x}$.
Note that the maximizer of 
\begin{align*}
\max_{\overline{\xi}_a(t)\geq \overline{\xi}_b(t)} p_{\xi_a}\big((\underline{R}_{a,x}(t)), \underline{X}(t)\big) p_{\xi_b}\big((\underline{R}_{b,x}(t)), \underline{X}(t)\big)
\end{align*}
is equivalent to that of 
\begin{align*}
\max_{\substack{(\xi_{a,x}, \xi_{b,x})_{x \in \set{X}}\in [0,1]^{|\set{X}| \times 2}\\
\sum_{x\in\mathcal{X}}\hat{\zeta}_x(t)\xi_{a, x} \geq \sum_{x\in\mathcal{X}}\hat{\zeta}_x(t)\xi_{b, x}}} &\ t\sum_{x\in\mathcal{X}}\sum_{c\in\{a,b\}}\Bigg\{\frac{N_{c,x}}{t}\left\{\hat{\mu}_{c,x}(t)\log \frac{\xi_{c, x}}{1 - \xi_{c, x}} + \log (1 - \xi_{c, x})\right\}\Bigg\}.
\end{align*}
We denote the maximizers by $(\tilde{\xi}_{a,x}(t))$ 
and $(\tilde{\xi}_{b,x}(t))$. Similarly, we denote the solution of the maximization problem in the denominator by  $(\tilde{\xi}^\dagger_{a,x}(t))$ and $(\tilde{\xi}^\dagger_{b,x}(t))$. 

In the numerator, if $\sum_{x\in\mathcal{X}}\hat{\zeta}_x(t)\hat{\mu}_{a,x}(t) \geq \sum_{x\in\mathcal{X}}\hat{\zeta}_x(t)\hat{\mu}_{b,x}(t)$, then the maximum likelihood estimator falls within the optimization constraint; that is, $\tilde{\xi}_{a,x}(t) = \hat{\mu}_{a,x}(t)$ and $\tilde{\xi}_{b,x}(t) = \hat{\mu}_{b,x}(t)$. Therefore, our remaining problem is to compute the denominator. Because $\sum_{x\in\mathcal{X}}\hat{\zeta}_x(t)\hat{\mu}_{a,x}(t) \geq \sum_{x\in\mathcal{X}}\hat{\zeta}_x(t)\hat{\mu}_{b,x}(t)$ does not satisfy the constraint condition in the denominator, it is hard to obtain the closed-form expression of the denominator and we need to solve the optimization problem numerically. Given the solutions, $(\tilde{\xi}^\dagger_{a,x}(t))$ and $(\tilde{\xi}^\dagger_{b,x}(t))$, the GLRT statistic $Z_{a,b}(t)$ is equal to
\begin{align*}
&t\sum_{x\in\mathcal{X}}\sum_{c\in\{a,b\}}\Bigg\{\frac{N_{c,x}(t)}{t}\Big\{\hat{\mu}_{c,x}(t)\log \frac{\hat{\mu}_{c,x}(t)}{1 - \hat{\mu}_{c,x}(t)} + \log (1 - \hat{\mu}_{c,x}(t))\\
&\ \ \ \ \ \ \ \ \ \ \ \ \ \ \ \ \ \ \ \ \ \ \ \ \ \ \ \ \ \ \ \ \ \ \ \ \ \ \ \ \ \ \ \ \ \ \ \ \ \ \ \ \ \ \ \ \ - \hat{\mu}_{c,x}(t)\log \frac{\tilde{\xi}^\dagger_{c, x}(t)}{1 - \tilde{\xi}^\dagger_{c, x}(t)} - \log (1 - \tilde{\xi}^\dagger_{c, x}(t))\Big\}\Bigg\}\\
&=\max_{\substack{(\xi_{a,x}, \xi_{b,x})_{x \in \set{X}}
\\
\sum_{x\in\mathcal{X}}\hat{\zeta}_x(t)\xi_{a, x} \leq \sum_{x\in\mathcal{X}}\hat{\zeta}_x(t) \xi_{b, x}}}t\sum_{x\in\mathcal{X}}\left(\frac{N_{a,x}(t)}{t}d(\hat{\mu}_{a,x}(t), \xi_{a,x}(t)) + \frac{N_{b,x}(t)}{t}d(\hat{\mu}_{b,x}(t), \xi_{b,x}(t))\right).
\end{align*}
By multiplying $Z_{a,b}(t)$ by $-1/t$, we can find that solving the maximization problem is equal to solving the inner minimization problem of \eqref{eq:mininf_problem}, or equivalently the problem defined in \eqref{eq:inner_opt}, by letting $\zeta_xw_{1,x}=\frac{N_{a,x}}{t}$, $\zeta_xw_{b,x}=\frac{N_{a,x}}{t}$, $\mu_{1,x} = \hat{\mu}_{a,x}(t)$, and $\mu_{a,x} = \hat{\mu}_{b,x}(t)$. 
From Lemma~\ref{lem:allEqual}, the constraint $\sum_{x\in\mathcal{X}}{\zeta}_x \xi_{a, x} \leq \sum_{x\in\mathcal{X}}{\zeta}_x \xi_{b, x}$ holds with equality; that is, 
\[Z_{a,b}(t) = -tL_{a,b}\left( \left(\hat{\mu}_{a,x}(t), \hat{\mu}_{b,x}(t), {\zeta}_x t, {N_{a,x}(t)}/N_{x}(t), {N_{b,x}(t)}/N_{x}(t) \right)_{x \in \set{X}}\right).\] 
It is also easy to observe that when $\sum_{x\in\mathcal{X}}{\zeta}_x  \hat{\mu}_{a,x}(t) \leq \sum_{x\in\mathcal{X}} {\zeta}_x \hat{\mu}_{b,x}(t)$, then $Z_{a,b}(t) = -Z_{b,a}(t)$.

Using the GLRT statistic, we use the following stopping rule:
\begin{align}\label{eq:stopping_rule}
\tau_\delta & = \inf\Big\{t\in\mathbb{N}: Z(t) := \max_{a\in[K]}\min_{b\in[K]\backslash\{a\}} Z_{a,b}(t) > \beta(t,\delta) \Big\},
\end{align}
where $\beta(t, \delta)$ is the threshold of the GLRT statistic $Z_{a,b}(t)$ (exploration rate), which controls the failure probability under the stopping rule. 

Next, we determine $\beta(t, \delta)$ such that the proposed algorithm is $\delta$-PAC. We present the following theorem to decide the threshold $\beta(t, \delta)$ in the stopping rule.
\begin{theorem}\label{thm:deltaPAC} 
Let $\delta \in (0,1)$. For a Bernoulli bandit model, if $\beta(t,\delta) = \log \left(\frac{2t (K-1)}{\delta}\right)$, then for all $\nu \in \Theta$
$$
\mathbb{P}_{\nu}\left(\tau_\delta<\infty, \,\hat{a}_{\tau_\delta} \neq a^*\right) \leq \delta.
$$
\end{theorem}
The proof is provided in Appendix~\ref{appdx:deltaPAC}. The proof with contextual information is accomplished by using the fact that joint distribution of the contexts and the rewards is the Multinomial distribution. This theorem confirms that the proposed algorithm is $\delta$-PAC when $\beta(t, \delta)=\log \left((2t (K-1))/\delta\right)v$. We note that this threshold does not depend on the cardinality of $\mathcal{X}$.

\subsection{Stopping Rule for a Canonical One-parameter Exponential Family and Known Contextual Distribution}
\label{sec:stopping_oneexp}
For the Bernoulli bandit, we derive the stopping and recommendation rule by using the fact that the rewards and finite contexts jointly follow a Multinominal distribution. We cannot use this property when the conditional rewards follow different distributions such as a Gaussian distribution. For example, when the rewards follow a Gaussian distribution, the rewards and contexts jointly follow a Gaussian mixture model, not a Gaussian distribution. This fact makes derivation of the $\delta$-PAC threshold difficult. However, if the contextual distribution is known, we can extend the existing results, such as \citet{Garivier2016} and \citet{Kaufmann2021}, to derive the threshold. 

We consider a case where for each $a \in [K]$ and $x \in \mathcal{X}$, the reward $R_{t,a}$ follows a distribution that belongs to the canonical one-parameter exponential family \eqref{eq:spef} conditioned on $X_t = x$ and the context $X_{t}$ follows a multinomial distribution with known parameters; that is, we treat the estimator $\hat{\zeta}_x$ as the true value $\zeta_x$ in our proposed CTS algorithm. Similarly to the Bernoulli case, the likelihood of the observations $(\underline{R}_{a,x}(t))_{x \in \mathcal{X}}, \forall a \in [K]$ and $\underline{X}(t)$ regarding arm $a$ is given as follows. 
\begin{align*}
&p_{\mu_{a}}\big((\underline{R}_{a,x}(t))_{x \in \mathcal{X}}, \underline{X}(t)\big)= \prod^t_{s=1}\prod_{x\in\mathcal{X}}\left(\zeta_x\exp\left(\dot{b}^{-1}(\mu_{a,x}) R_s - b(\dot{b}^{-1}(\mu_{a,x}))\right)\right)^{\mathbbm{1}[A_s = a, X_s = x]}.
\end{align*}
Then, for all pairs of the arms, $a,b\in [K]$, the GLRT statistic is given as
\begin{align*}
Z_{a,b}(t) = \log \frac{\max_{\overline{\xi}_a(t)\geq \overline{\xi}_b(t)} p_{\xi_a}\big((\underline{R}_{a,x}(t)), \underline{X}(t)\big) p_{\xi_b}\big((\underline{R}_{b,x}(t)), \underline{X}(t)\big)}{\max_{\overline{\xi}_a(t)\leq \overline{\xi}_b(t)} p_{\xi_a}\big((\underline{R}_{a,x}(t)), \underline{X}(t)\big) p_{\xi_b}\big((\underline{R}_{b,x}(t)), \underline{X}(t)\big)},
\end{align*}
where $\overline{\xi}_a(t) = \sum_{x \in \set{X}} \zeta_x\xi_{a,x}$.
For the numerator optimization problem, from the definition of the single parameter exponential family, the maximizer of
\begin{align*}
&\max_{\overline{\xi}_a (t)\geq \overline{\xi}_b(t)} p_{\xi_a}\big((\underline{R}_{a,x}(t)), \underline{X}(t)\big) p_{\xi_b}\big((\underline{R}_{b,x}(t)), \underline{X}(t)\big)
\end{align*}
is equivalent to the maximizer of the optimization problem
\begin{align*}
\max_{\substack{(\xi_{a,x}, \xi_{b,x})_{x \in \set{X}}\\
\sum_{x\in\mathcal{X}}\zeta_x\xi_{a, x} \geq \sum_{x\in\mathcal{X}}\zeta_x\xi_{b, x}}} \ t\sum_{x\in\mathcal{X}}\sum_{c\in\{a,b\}}\frac{N_{c,x}(t)}{t}\left\{ \dot{b}^{-1}(\xi_{c,x}) \hat{\mu}_{c,x}(t) -  b(\dot{b}^{-1}(\xi_{c,x}))\right\}.
\end{align*}

As for the case of a Bernoulli bandit model, using the notation, we compute the GLRT statistic $Z_{a,b}(t)$ as follows. Now, suppose that  $\sum_{x\in\mathcal{X}}\zeta_x\hat{\mu}_{a,x}(t) \geq \sum_{x\in\mathcal{X}} \zeta_x\hat{\mu}_{b,x}(t)$. Then, $\tilde{\xi}_{a,x}(t) = \hat{\mu}_{a,x}(t)$ in the denominator and $\tilde{\xi}_{b,x}(t) = \hat{\mu}_{b,x}(t)$. We numerically solve the optimization problem in the denominator and obtain the solutions, $(\tilde{\xi}^\dagger_{a,x}(t))$ and $(\tilde{\xi}^\dagger_{b,x}(t))$.
 Then, $Z_{a,b}(t)$ is equal to
\begin{align*}
&t\sum_{x\in\mathcal{X}}\sum_{c\in\{a,b\}}\frac{N_{c,x}(t)}{t}\Big\{\dot{b}^{-1}(\hat{\mu}_{c,x}(t)) \hat{\mu}_{c,x}(t) - b(\dot{b}^{-1}\left(\hat{\mu}_{c,x}(t))\right) - \dot{b}^{-1}(\tilde{\xi}^\dagger_{c,x}(t))\hat{\mu}_{c,x}(t) + b\left(\dot{b}^{-1}(\tilde{\xi}^\dagger_{c,x}(t))\right)\Big\},
\\
&=\max_{\substack{(\xi_{a,x}, \xi_{b,x})_{x \in \set{X}}
\\
\sum_{x\in\mathcal{X}}{\zeta}_x\xi_{a, x} \leq \sum_{x\in\mathcal{X}}{\zeta}_x \xi_{b, x}}}t\sum_{x\in\mathcal{X}}\left(\frac{N_{a,x}(t)}{t}\kl(\hat{\mu}_{a,x}(t), \xi_{a,x}(t)) + \frac{N_{b,x}(t)}{t}\kl(\hat{\mu}_{b,x}(t), \xi_{b,x}(t))\right).
\end{align*}
A similar argument can be made when $\sum_{x\in\mathcal{X}}\zeta_x\hat{\mu}_{a,x}(t) < \sum_{x\in\mathcal{X}} \zeta_x\hat{\mu}_{b,x}(t)$ by reversing the sign of the constraint. 

Next, we define the stopping rule using the GLRT statistic $Z_{a,b}(t)$ as follows. 
\[\tau_\delta  = \inf\Big\{t\in\mathbb{N}: Z(t) := \max_{a\in[K]}\min_{b\in[K]\backslash\{a\}} Z_{a,b}(t) > \beta(t,\delta) \Big\},\]
where we decide the threshold $\beta(t,\delta)$ later. Let $\hat{\mu}_c(t) = \sum_{x \in \set{X}} \zeta_x \hat{\mu}_{c,x}(t)$ for $c\in\mathcal{A}$. If $\sum_{x \in \set{X}} \zeta_x \mu_{a,x} = \mu_a \leq \mu_{b} = \sum_{x \in \set{X}} \zeta_x \mu_{b,x}$ and $\hat{\mu}_a > \hat{\mu}_b$, then 
\begin{align*}
&Z_{a,b}(t)
\nonumber\\
&=\min_{\substack{(\xi_{a,x}, \xi_{b,x})_{x \in \set{X}}
\\
\sum_{x\in\mathcal{X}}\zeta_x\xi_{a, x} \leq \sum_{x\in\mathcal{X}}\zeta_x\xi_{b, x}}}t\sum_{x\in\mathcal{X}}\left(N_{a,x}(t)\kl(\hat{\mu}_{a,x}(t), \xi_{a,x}(t)) + ZN_{b,x}(t)\kl(\hat{\mu}_{b,x}(t), \xi_{b,x}(t))\right)
\nonumber\\
&\leq \sum_{x\in\mathcal{X}}\Bigg(N_{a,x}(t)\kl\left(\hat{\mu}_{a,x}(t), \mu_{a,x}\right) + N_{b,x}(t)\kl\left(\hat{\mu}_{b,x}(t), \mu_{b,x}\right)\Bigg).
\end{align*}
Then, we decompose the probability $\mathbb{P}_{\nu}\left(\tau_\delta<\infty, \,\hat{a}_{\tau_\delta} \neq a^*\right) $ as
\begin{align}
     & \mathbb{P}_{\nu}\left(\tau_\delta<\infty, \,\hat{a}_{\tau_\delta} \neq a^*\right) 
     \nonumber\\
     & \leq    \mathbb{P}_{\nu}\left(\exists a  \neq a^*, \exists t \in \mathbb{N} : \hat{\mu}_a(t)> \hat{\mu}_{a^*}(t), Z_{a,a^*}(t) > \beta(t,\delta)\right) 
     \nonumber\\
     \label{eq:delta_pac_target}
     & \le \mathbb{P}_{\nu}\Bigg(\exists a\neq a^*, \exists t \in \mathbb{N} : 
\sum_{c \in \{a, a^*\}} \sum_{x\in\mathcal{X}}N_{c,x}(t)\kl\left(\hat{\mu}_{c,x}(t), \mu_{c,x}\right) > \beta(t,\delta)  \Bigg).
\end{align}
Thus, if we choose a threshold $\beta(t,\delta)$ such that the upper bound of the last equation \eqref{eq:delta_pac_target} is $\delta$, we can guarantee that the algorithm is $\delta$-PAC.

Using the results of \citet{Kaufmann2021}, which refines existing deviation bounds and the threshold in \citet{Garivier2016}, we can guarantee that the algorithm is $\delta$-PAC with a tight threshold.
We use the following theorem from \citet{Kaufmann2021}.
\begin{theorem}[From Theorem~7 of \citet{Kaufmann2021}]\label{thm:dev_inequality_KK}
Let us define $ h(u) = u - \ln u,\; \forall u \ge 1$ and $ h^{-1}(u)$ (the inverse of $ h(u)$). For each $z \in [1, e]$ and for all $x \ge 0$,
\begin{align*}
    \tilde{h}_z (x) = \begin{cases}
     e^{1/h^{-1}(x)} h^{-1}(x) & \textnormal{if} \;x \ge h(1/\ln{z}
     \\
     z (x - \ln{\ln{z}}) & \textnormal{otherwise}.
    \end{cases}
\end{align*}
We further define the function $\mathcal{C}_{\textnormal{exp}} : \mathbb{R}^+ \mapsto \mathbb{R}^+$ as
\begin{align*}
    \mathcal{C}_{\textnormal{exp}}(x) = 2 \tilde{h}_{3/2}\left(\frac{h^{-1}(1 + x) + \ln(2 \zeta(2))}{2}\right),
\end{align*}
where $\zeta(s) = \sum_{n =1}^\infty n^{-s}$. For each subset $\mathcal{S}$ of the context arm pairs $(x, a) \in \mathcal{X} \times [K]$, for all $x>0$, the following holds. 
\begin{align*}
    &\mathbb{P}_\nu \left(\exists t \in \mathbb{N}: \sum_{(x, a) \in \mathcal{S}} N_{a,x}(t) \kl(\hat{\mu}_{a,x}, \mu_{a,x})  \ge \sum_{(x, a) \in \mathcal{S}} 3 \ln (1 + \ln (N_{a,x}(t) ))  + |\mathcal{S}| \mathcal{C}_{\textnormal{exp}}\left(\frac{x}{|\mathcal{S}|}\right)\right)
    \\
    & \le \exp(-x).
\end{align*}
\end{theorem}
Let us define the threshold $\beta(t,\delta)$ as
\begin{align*}
    \beta(t, \delta) = 6 |\mathcal{X}|\ln  \left(\ln \left( \frac{t}{2}\right) + 1\right) + 2 |\mathcal{X}| \mathcal{C}_{\textnormal{exp}}\left(\frac{\ln \frac{K-1}{\delta}}{2 |\mathcal{X}|}\right).
\end{align*}
Using this threshold, the following guarantee can be obtained. 
\begin{corollary}
Assume the context distribution is known; that is, we set $\hat{\zeta}_x(t) = \zeta_x, \forall x \in \set{X}, \forall t \in \mathbb{N}$ in the GLRT statistics. Let $\delta \in (0, 1)$. For any sampling rule, using the stopping rule \eqref{eq:stopping_rule} with the threshold
\begin{align*}
    \beta(t, \delta) = 6 |\mathcal{X}|\ln  \left(\ln \left( \frac{t}{2}\right) + 1\right) + 2 |\mathcal{X}| \mathcal{C}_{\textnormal{exp}}\left(\frac{\ln \frac{K-1}{\delta}}{2 |\mathcal{X}|}\right),
\end{align*}
for all $\nu \in \Theta$, $\mathbb{P}_{\nu}\left(\tau_\delta<\infty, \,\hat{a}_{\tau_\delta} \neq a^*\right) \leq \delta$.
\end{corollary}
\begin{proof}

With Theorem~\ref{thm:dev_inequality_KK} and the union bound over the set of $(K-1)$ pairs: $(a, a^*), a \neq a^*$, we bound $\mathbb{P}_{\nu}\left(\tau_\delta<\infty, \,\hat{a}_{\tau_\delta} \neq 1\right)$ as
\begin{align*}
    \mbox{\eqref{eq:delta_pac_target}}& \le \mathbb{P}_\nu\left(\exists a \neq a^*, \exists t \in \mathbb{N}: \sum_{c \in \{a, a^*\}}\sum_{x\in\mathcal{X}}N_{c,x}(t)\kl\left(\hat{\mu}_{c,x}(t), \mu_{c,x}\right) >  \right.
    \\
   &  \quad \quad \quad \left.  \sum_{c \in \{a, a^*\}}\sum_{x\in\mathcal{X}} 3 \ln (1 + \ln (N_{c,x}(t) )) +   2 |\mathcal{X}| \mathcal{C}_{\textnormal{exp}}\left(\frac{\ln \frac{K-1}{\delta}}{2 |\mathcal{X}|}\right) \right)
    \\
    & \le \delta.
\end{align*}
Furthermore, it is easy to check that $ \mathcal{C}_{\textnormal{exp}}(x) = x +o(x)$ as $x \to \infty$ \citep{Kaufmann2021}.
\end{proof}

\subsection{Sample Complexity Analysis}
\label{sec:sample_comp}
In this section, we address the upper bound of the sample complexity of the proposed CTS algorithm.

First, we demonstrate that the sample complexity asymptotically matches the lower bound almost surely for a case where the reward follows a Bernoulli bandit model.
\begin{proposition}\label{lem:AnalysisAS} Suppose that the reward follows a Bernoulli bandit model. If the sampling rule ensures that for all $a\in[K]$, for all $ x \in \set{X}$, $\min_{\vec{w}^* \in \Phi(\nu)}\left|\lim_{t\to\infty}\frac{N_{a,x}(t)}{t} - \zeta_xw^*_{a, x} \right| = 0$, and we follow the stopping rule defined in Section~\ref{sec:stop} with $\beta(t,\delta)=\log \left(\frac{2t (K-1)}{\delta}\right)$, then for all $\delta\in (0,1)$, $\mathbb{P}_{\nu}(\tau_\delta < \infty)=1$ and \[\mathbb{P}_{\nu}\left(\limsup_{\delta \rightarrow 0}\frac{\tau_\delta}{\log(1/\delta)} \leq T^\star(\nu) \right) = 1.\]
\end{proposition}
We provide the proof in Appendix~\ref{appdx:AnalysisAS}.

We now provide an upper bound on the expected number of the stopping times $\EXP[\tau_\delta]$. The following theorem states that the proposed CTS algorithm asymptotically matches the sample complexity lower bound derived from Theorem~\ref{thm:lower_bound}. The proof of this result is provided in Appendix~\ref{appdx:ProofSC}. 
\begin{theorem}\label{thm:AsymptoticSC} Suppose that the reward follows a Bernoulli bandit model. For each $\nu \in \Theta$, if sampling rule ensures that for all $a\in[K]$, for all $ x \in \set{X}$, $\min_{\vec{w}^* \in \Phi(\nu)}\left|\lim_{t\to\infty}\frac{N_{a,x}(t)}{t} - \zeta_xw^*_{a, x} \right| = 0$, and we follow the stopping rule defined in Section~\ref{sec:stop} with $\beta(t,\delta)=\log \left(\frac{2t (K-1)}{\delta}\right)$, then
\[\limsup_{\delta \rightarrow 0} \frac{\mathbb{E}_{\nu}[\tau_\delta]}{\log(1/\delta)} \leq {T^\star(\nu)}\;.\] 
\end{theorem}

As well as the case with a Bernoulli bandit model, we can also show that an upper bound on the expected number of the stopping times $\EXP[\tau_\delta]$ matches the lower bound almost surely for a case where the reward follows a distribution that belongs to a canonical one-parameter exponential family, and the parameters of the context distribution are known.

\begin{corollary}Suppose that the reward follows a distribution that belongs to a canonical one-parameter exponential family, and $(\zeta_x)_{x\in\mathcal{X}}$ is known. For each $\nu \in \Theta$, if sampling rule ensures that for all $a\in[K]$, for all $ x \in \set{X}$, $\min_{\vec{w}^* \in \Phi(\nu)}\Big|\lim_{t\to\infty}\frac{N_{a,x}(t)}{t} - \zeta_xw^*_{a, x} \Big| = 0$, and we follow the stopping rule defined in Section~\ref{sec:stopping_oneexp} with $\beta(t,\delta)=  6\ln\Big(\ln\Big(\frac{t}{2}\Big) + 1\Big) + 2|\mathcal{X}|\mathcal{C}_{\textnormal{exp}}\Big(\frac{\ln \frac{K-1}{\delta}}{2|\mathcal{X}|}\Big)$, then for all $\delta\in (0,1)$, $\mathbb{P}_{\nu}(\tau_\delta < +\infty)=1$ and $\mathbb{P}_{\nu}\Big(\limsup_{\delta \rightarrow 0}\frac{\tau_\delta}{\log(1/\delta)} \leq \alpha T^\star(\nu) \Big) = 1$. Besides, $\limsup_{\delta \rightarrow 0} \frac{\mathbb{E}_{\nu}[\tau_\delta]}{\log(1/\delta)} \leq {\alpha}{T^\star(\nu)}$. 
\end{corollary}

\section{Simulation Studies}\label{sec:experiments}
In this section, we investigate the behavior of the proposed algorithms. First, we examine the performance of $\alpha$-elimination using contextual information. As in Section~\ref{sec:cont_2arm}, we generate samples $\{(R_{t, 1}\ R_{t, 2}\ X_t)^\top\}^T_{t=1}$ from the multivariate distribution with the mean vector $(1 \ 0 \ 0)^\top$. We denote the variances of $R_{t, 1}$, $R_{t, 2}$, and $X_t$ as $\sigma^2_{1}$, $\sigma^2_2$, and $\sigma^2_{\set{X}}$. Let the correlation coefficient between $R_{t,1}$ and $X_t$ be $\rho_{1 \set{X}}$, and the correlation coefficient between $R_{t,2}$ and $X_t$ be $\rho_{2\set{X}}$. We fix $\sigma^2_2 = 1$, $\sigma^2_{\set{X}}=1$, and $\rho_{2\set{X}}=0.5$. We investigate the performance of the proposed method by varying the combination of the variance $\sigma^2_1$ and correlation coefficient $\rho_{1\set{X}}$. We choose $\sigma^2_1$ from $\{1,2\}$ and $\rho_{1\set{X}}$ from $\{-0.9, -0.5, 0, 0.5, 0.9\}$. For the case with $\sigma^2_1=1$, the $\alpha$-elimination without contextual information of \citet{Kaufman2016complexity} results in an allocation of $\alpha=0.5$ (uniform sampling). For the case with $\sigma^2_1=2$, it results in an allocation of $\alpha=\sqrt{2}/(\sqrt{2} + \sqrt{1})$. Conversely, the proposed $\alpha$-elimination with contextual information uses different allocations for each correlation coefficient. We conducted $1000$ trials with $\delta=0.05$ and display the realized stopping time (sample complexity) in Figure~\ref{fig:exp_results1} using box plots, where the right figure shows the results with $\sigma^2_1=1$ and the left shows the results with $\sigma^2_1=2$. In Figure~\ref{fig:exp_results1}, we compare the proposed algorithm with different $\rho_{1\set{X}}$ with the $\alpha$-elimination (without context). The results demonstrate that when using contextual information, the proposed $\alpha$-elimination can stop earlier than the original $\alpha$-elimination. We note the fact that the proposed algorithm can stop earlier, even though the allocation is also $0.5$ when $\rho_{1\set{X}}$ is $0$. Here, the stopping threshold $\beta$ used in the proposed algorithm is less than that used in the original algorithm, while maintaining the $\delta$-PAC property. Note that for all cases, the realized $\delta$ does not exceed $0.05$.

\begin{figure}[t]
  \begin{center}
    \includegraphics[width=150mm]{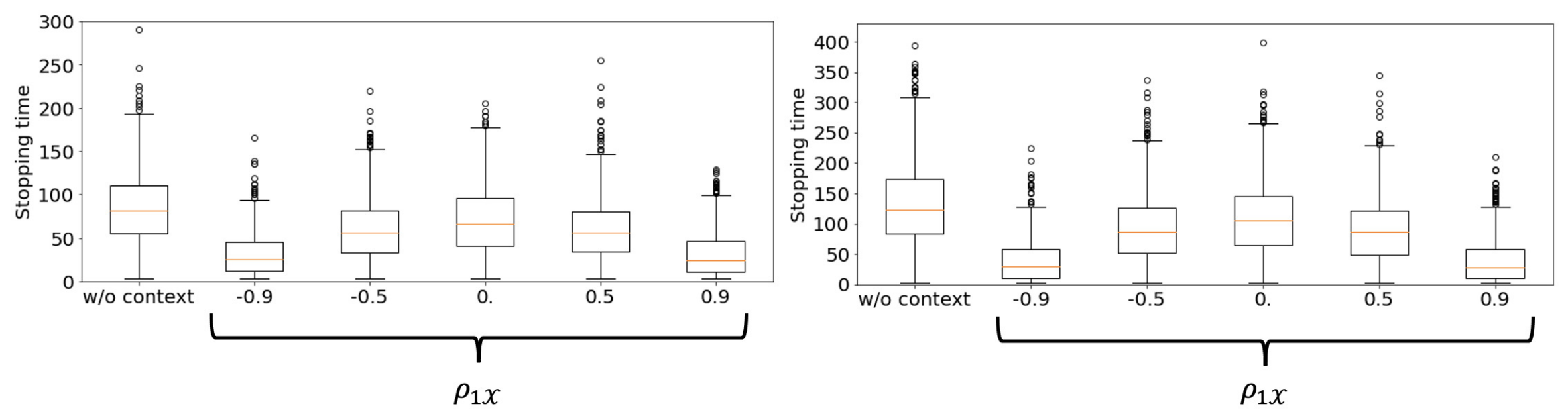}
  \end{center}
  \vspace{-0.3cm}
  \caption{Results of $\alpha$-elimination. The left figure displays the results with $\sigma^2_1=1$; the right figure displays the results with $\sigma^2_1=2$.}
  \label{fig:exp_results1}
  \vspace{-0.5cm}
\end{figure}

Next, we compare the performance of the proposed CTS algorithm to the TS algorithm for BAI without contextual information \citep{Garivier2016}. For a Bernoulli bandit model, we consider a sample scenario with marginalized mean rewards $\{\mu_1, \mu_2, \mu_3, \mu_4\} = \{0.3, 0.21, 0.2, 0.19\}$, which is the same as a scenario used in \citet{Garivier2016}. Suppose that there exist two contexts $X_t \in\{1,2\}$, where the conditional mean rewards are given as $\{\mu_{1,1}, \mu_{2,1}, \mu_{3,1}, \mu_{4,1}\} = \{0.5, 0.01, 0.4, 0.01\}$ and $\{\mu_{1,2}, \mu_{2,2}, \mu_{3,2}, \mu_{4,2}\} = \{0.1, 0.41, 0., 0.37\}$. The context $1$ and $2$ appear with probability $0.5$, respectively. Because $\beta(t,\delta)$ can be determined by us within the range suggested in Theorems~\ref{thm:deltaPAC}--\ref{thm:AsymptoticSC}, and because the role of $\beta(t,\delta)$ does not change considerably between the CTS and TS algorithms, we display the value of the GLRT statistic in Figure~\ref{fig:res_track}. The earlier this value becomes large, the smaller the sample complexity that can be achieved under a properly specified $\beta(t,\delta)$. This figure indicates that the CTS algorithm achieves a smaller sample complexity than TS, as suggested by the theoretical results. Conversely, the reason why the CTS algorithm indicates a smaller GLRT statistic compared with TS in the early rounds is likely because the number of parameters to be estimated is proportional to the number of contexts; thus it requires more time to converge in finite samples. In Appendix~\ref{appdx:exp}, we present more details and additional results under different settings.

\begin{figure}[t]
 \begin{center}
    \includegraphics[height=60mm]{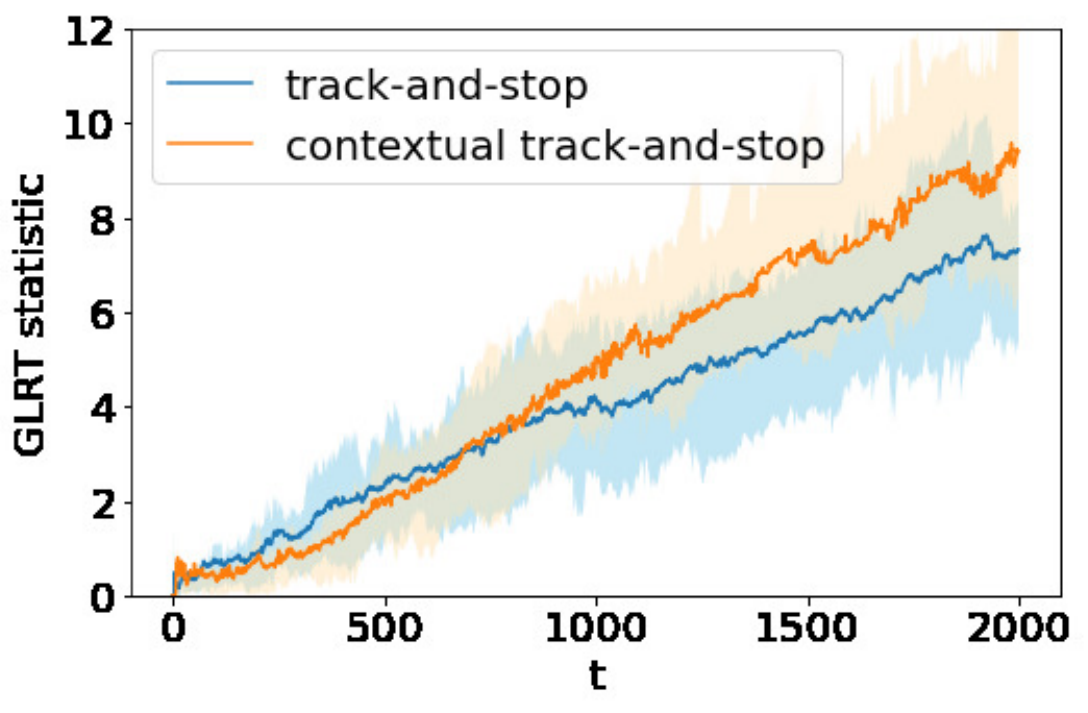}
  \end{center}
  \vspace{-0.5cm}
  \caption{Graph illustrating the maximum GLRT statistic $\max_{a\in[K]}\min_{b\in[K]\backslash\{a\}}Z_{a,b}(t)$. The solid line represents the averaged value over $20$ trials; the light-colored area indicates the values between the first and third quartiles.
  }
  \label{fig:res_track}
\end{figure}

\section{Conclusion}
This paper proposed contextual BAI, where contextual information can be used to identify marginalized mean rewards. We noted that even contextual information that is not immediately related to the parameter that we wish to identify could help us solve the task more efficiently. We proposed the CTS algorithm as an algorithm when the rewards follow Bernoulli distributions, and confirmed that it performs better theoretically and experimentally when contextual information is provided. We also found that when the rewards and context follow a multivariate normal distribution in the two-armed bandit problem, we could improve the efficiency of BAI without changing the conventional algorithm. These properties have not been discussed to date. We consider that these results are related to semiparametric inference and the James--Stein shrinkage estimator; however, it is a future task to clarify their relationship

\section*{Acknowledgement}
The authors thank Alexandre Proutière for detailed discussions.

\vskip 0.2in
\bibliography{arXiv.bbl}
\bibliographystyle{asa}

\clearpage

\appendix

\section{Notations, Terms, and Abbreviations}\label{sec:notation}
In this section, we summarize the notations used in this paper. 
\begin{table}[h]
    \centering
    \caption{Summary of notations}
    \label{tbl:sum_not}
    \small
    \begin{tabular}{l|l}
    $X_t,\ A_t,\, R_t$ & Context, action, and reward observed in round $t$ \\ 
    $[K],\ \set{X}$ & Sets of actions and contexts\\ 
    $R_{t,a}$ & Potential reward of arm $a$ \\
    $\zeta$ & Distribution of $X_t$ \\
    $\vec{p} = (p_{1, x}, p_{2, x}, \ldots, p_{K, x})$ & Reward distributions of the potential outcome given $x \in \mathcal{X}$.\\
    $\vec{\mu} = (\mu_{1, x}, \mu_{2, x}, \ldots, \mu_{K, x})$ & Conditional mean rewards given $x \in \mathcal{X}$.\\
    $\mu_a = \mathbb{E}_{X \sim \zeta}[\mu_{a, X}] = \mathbb{E}_{X \sim \zeta}[\mathbb{E}_{\set{V}}[R_{t, a} | X]] $ & \multirow{2}{*}{Marginalized mean reward of arm $a$}.\\
    $\ \ \ \ \ = \mathbb{E}_{\set{V}}[R_{t, a}]$ & {} \\
    $\mathcal{V} = (\vec{p}, \zeta)$ & Bandit problem.\\
    $\nu = ((\mu_{a,x}), (\zeta_x))$ & Bernoulli  bandit problem with finite context.\\
    $\Omega$ (resp. $\Theta$) & Class of $\mathcal{V}$ (resp. $\nu$).\\
    $a^* = a^*(\set{V}) = \argmax_a \mu_a$ & Best arm with the highest marginalized mean reward.\\
    $\mathcal{F}_t = \sigma(X_1, A_1, R_1, \ldots, X_t, A_t, R_t, X_{t+1})$ & Sigma-algebras with the observations until $t$ and  $X_{t+1}$. \\
    $\mathcal{G}_t = \sigma(X_1, A_1, R_1, \ldots, X_t, A_t, R_t)$ & Sigma-algebras with all observations up to $t$.\\
    $\tau_{\delta}$ & Stopping time under a fixed confidence $\delta > 0$.\\
    $\hat{a}_{\tau_\delta}$ & Recommended arm.\\
    $\Alt(\mathcal{V}) := \{(\vec{q}, \zeta) \in \Omega: a^*((\vec{q}, \zeta)) \neq a^*((\vec{p}, \zeta))\}$ & Set of alternative problems.\\
    $N_x(t) = \sum_{s=1}^t \indicator\{ X_s = x\}$ & The number of times we observe context $x$.\\
    $N_{a, x}(t) = \sum_{s=1}^t \indicator\{X_s = x, A_s = a\}$ & The number of times we choose arm $a$ given context $x$.\\
    $\KL(p_{a, x}, q_{a, x})$ & KL divergence from $p_{a, x}$ to $ q_{a, x}$\\
    \multirow{2}{*}{$\kl(\mu, \nu)$} & KL divergence of the canonical\\
    {} & {one-parameter exponential family.} \\
    $d(\mu, \nu)$ & \multirow{2}{*}{KL divergence of Bernoulli distributions.}\\
    $=\mu\log (\mu/\nu) + (1-\mu)\log((1-\mu)/(1-\nu))$ & {} \\
    $\hat{\mu}_{a,x},\ \hat{\zeta}_x$ & Estimators of $\mu_{a,x}$ and $\zeta_x$ in round $t$.\\
    $w_{a,x}$ & Allocation for arm $a$ given context $x$.\\
    $\set{W}$ & Set of allocation rule.\\
    \multirow{2}{*}{$p_{\mu_a}\big((\underline{R}_{a,x}(t))_{x \in \mathcal{X}}, \underline{X}(t)\big)$} & 
    Likelihood of parameters $\mu_a = (\mu_{a,x})_{a\in [K],x \in \mathcal{X}}$ given 
    \\
    {} & the observations $(\underline{R}_{a,x}(t))_{x \in \mathcal{X}}$ and $\underline{X}(t)$.
    \\
    $Z_{a,b}(t)$ & GLRT statistic.\\
    $\beta(t,\delta)$ & Threshold for stopping rule.
    \end{tabular}
    \label{tab:notation}
\end{table}

\section{Proof of Lemma~\ref{lem:kauf_lemma_extnd_finite}}
\label{appdx:kauf_lemma_extnd_finite}
For each problem $\nu = ((\mu_{a,x}), (\zeta_x))$, for each $a \in [K]$, $ x \in \set{X}$, let us denote by $f_{a, x}^\nu$  be the density (w.r.t. the Lebesgue measure) of the reward with the action-context pair $(a, x)$.
Let us define a log-likelihood ratio between the observation under the model $\nu = ((\mu_{a,x}), (\zeta_x))$ to the model $\nu' = ((\lambda_{a,x}), (\zeta_x))$ as
\begin{align*}
    L_\tau = \sum_{t=1}^\tau \sum_{x \in \set{X}} \sum_{a \in [K]} \indicator\{X_t = x, A_t = a\} \log \left(\frac{f_{a, x}^\nu(R_t)}{f_{a, x}^{\nu'}(R_t)}\right)
\end{align*}
We have
\begin{align*}
    \EXP_{\nu}[L_\tau]
    & = \EXP_{\nu}\left[ \sum_{t=1}^\tau \sum_{a \in [K]} \sum_{x \in \set{X}}  \indicator\{X_t = x, A_t = a\} \log \left(\frac{f_{a, x}^\nu(R_t)}{f_{a, x}^{\nu'}(R_t)}\right)\right]
    \\
    & \stackrel{(a)}{=} \EXP_{\nu}\left[ \sum_{x \in \set{X}} \sum_{a \in [K]} \sum_{k=1}^{N_{a,x}(\tau)} \log \left(\frac{f_{a, x}^\nu(Y_k^{(x,a)})}{f_{a, x}^{\nu'}(Y_k^{(x,a)})}\right) \right]
    \\
    & = \sum_{x \in \set{X}} \sum_{a \in [K]} \EXP_{\nu} [N_{a, x}(\tau)] \kl(\mu_{a, x}, \lambda_{a, x}),
\end{align*}
where for $(a)$, we introduced random variables: $Y_{k}^{(x,a)}$ denotes $k$-th time the reward with the context $x$ and the action $a$ is observed and for the last equality, we used Wald's lemma for each $(x, a)$ pair. From the data-processing inequality applied to the change-of-measure argument \citet{garivier2019explore}, we have, for any $\set{E} \in \set{G}_\tau$, 
\begin{align*}
     \sum_{x \in \set{X}} \sum_{a \in [K]} \EXP_{\nu} [N_{a, x}(\tau)] \kl(\mu_{a, x}, \lambda_{a, x}) \ge d(\Pr_\nu(\set{E}), \Pr_{\nu'}(\set{E})).
\end{align*}
This concludes the proof of Lemma~\ref{lem:kauf_lemma_extnd_finite}.

\section{Proof of Theorem~\ref{thm:lower_bound}}
\label{appdx:proof_lower_bound}
\paragraph{Proof.}
From Lemma~\ref{lem:kauf_lemma_extnd_finite} with $\set{E} =\{\hat{a}_\tau = {a}^*(\nu)\}$, for each $\nu \in \Theta$ and $\nu' \in \Alt(\nu)$, we have
\begin{align*}
\sum_{x \in \set{X}}  \sum_{a \in [K]} \EXP_{\nu} [N_{a, x}(\tau)] \kl(\mu_{a, x}, \lambda_{a, x}) & \ge  d(\Pr_\nu(\set{E}), \Pr_{\nu'}(\set{E}))
 \ge  \kl(\delta, 1- \delta),
\end{align*}
where for the last inequality, we used the definition of the $\delta$-PAC algorithm and monotonicity of the KL divergence. Let $N_x(\tau) = \sum_{t=1}^\tau \indicator\{ X_t = x\}$.
For each $\nu \in \Theta$,
\begin{align*}
d(\delta, 1- \delta) & \le \inf_{((\lambda_{a,x}), \zeta) \in \Alt(\nu)} \sum_{x \in \set{X}} \sum_{a \in [K]} \EXP_{\nu} [N_{a, x}(\tau)] \kl(\mu_{a, x}, \lambda_{a, x})\\
& = \inf_{((\lambda_{a,x}), \zeta) \in \Alt(\nu)} \sum_{x \in \set{X}}  \EXP_{\nu} [N_{x}(\tau)] \sum_{a \in [K]} \frac{ \EXP_{\nu} [N_{a, x}(\tau)]}{ \EXP_{\nu} [N_{x}(\tau)]} \kl(\mu_{a, x}, \lambda_{a, x})\\
& = \inf_{((\lambda_{a,x}), \zeta) \in \Alt(\nu)} \EXP_{\nu}[\tau_\delta] \sum_{x \in \set{X}} \frac{\EXP_{\nu} [N_{x}(\tau)]}{\EXP_{\nu}[\tau_\delta]} \sum_{a \in [K]} \frac{ \EXP_{\nu} [N_{a, x}(\tau)]}{ \EXP_{\nu} [N_{x}(\tau)]} \kl(\mu_{a, x}, \lambda_{a, x})\\
& \stackrel{(a)}{=} \inf_{((\lambda_{a,x}), \zeta) \in \Alt(\nu)} \EXP_{\nu}[\tau_\delta] \sum_{x \in \set{X}}  \frac{\EXP_{\nu}[\tau_\delta]\zeta_x}{\EXP_{\nu}[\tau_\delta]} \sum_{a \in [K]} \frac{ \EXP_{\nu} [N_{a, x}(\tau)]}{ \EXP_{\nu} [N_{x}(\tau)]}  \kl(\mu_{a, x}, \lambda_{a, x})\\
& = \EXP_{\nu}[\tau_\delta]  \inf_{((\lambda_{a,x}), \zeta) \in \Alt(\nu)} \sum_{x \in \set{X}}  \zeta_x \sum_{a \in [K]} \frac{ \EXP_{\nu} [N_{a, x}(\tau)]}{ \EXP_{\nu} [N_{x}(\tau)]}  \kl(\mu_{a, x}, \lambda_{a, x})\\
& \le \EXP_{\nu}[\tau_\delta]  \sup_{w \in \set{W}} \inf_{((\lambda_{a,x}), \zeta) \in \Alt(\nu)} \sum_{x \in \set{X}}  \zeta_x \sum_{a \in [K]} w_{a, x} \kl(\mu_{a, x}, \lambda_{a, x}),
\end{align*}
where for $(a)$, we used Wald's lemma for each $x$. This concludes the proof.

\section{Proof of Theorem~\ref{thm:lower_bound_continuous}}
\label{appdx:proof_lower_bound_continuous}
We show Theorem~\ref{thm:lower_bound_continuous}. Let $\mathcal{B}(\mathbb{R})$ be a Borel $\sigma$-algebra on $\mathbb{R}$. Let us introduce two random counting measures on $\mathbb{R}$: (i) for each $ A \in \mathcal{B}(\mathbb{R})$, $\Xi (A) $ counts the number of times contexts has arrived in $A$, (ii) $\Upsilon_a (A) $ counts the number of times the algorithm selected action $a$ under the context is in $A$. 

The intensity measure is a characteristic analogous to the mean of a real-valued random variable \citep{chiu2013stochastic}. Let us denote the intensity measures of $\Xi $ and $\Upsilon_a $ by $\gamma$ and $\kappa_a$, respectively; that is, $\gamma(A) = \mathbb{E}\left[\Xi(A)\right]$ and $\kappa_a(A) = \mathbb{E}\left[\Upsilon_a(A)\right]$ for each $ A \in \mathcal{B}(\mathbb{R})$. Suppose that $\gamma $ and $\kappa_a$ are absolutely continuous with respect to $\zeta $  \citep{kallenberg2017random}. Furthermore, $\kappa_a$ is absolutely continuous with respect to $\gamma $. Let $\frac{d\gamma}{d x}(x)$ and $\frac{\mathrm{d}\kappa_a}{\mathrm{d}x}(x)$ be densities of $\gamma$ and $\kappa_a $ with respect to the Lebesgue measure. 

Then, we extend our Lemma~\ref{lem:kauf_lemma_extnd_finite} to the case of continuous contexts.

\begin{lemma}
\label{lem:kauf_lemma_extnd_infinite}
 Take $\mathcal{V} = (\vec{p}, \zeta), \mathcal{M} = (\vec{q}, \zeta) \in \Omega$. 
For any almost-surely finite stopping time $\tau$ with respect to $(\mathcal{G}_{t})_{t \ge 1}$,
\begin{align*}
&\sum^K_{a=1} \int_{\mathbb{R}} \frac{\mathrm{d}\kappa_a}{\mathrm{d}x}(x) \KL(p_{a, x}, q_{a, x}) \mathrm{d} x \geq \sup_{\Ecal\in\mathcal{G}_\tau}d(\Pbb_{\mathcal{V}}(\Ecal), \Pbb_{\mathcal{M}}(\Ecal)),
\end{align*}
where $\mathbb{E}_{\set{V}}$ (resp. $(\mathbb{E}_{\set{M}})$) and $\mathbb{P}_{\set{V}}$ (resp. $\mathbb{P}_{\set{M}}$ ) are the expectation under the model $\set{V}$ (resp. $\set{M}$) and the probability under the model $\set{V}$ (resp. $\set{M}$), respectively. 
\end{lemma}

In the proof, we use Campbell's theorem.
\begin{proposition}[Campbell's theorem from Theorem~4.1 in \citet{chiu2013stochastic}]
\label{prp:campbell}
For any nonnegative measurable function $f(x)$ and $ A \in \mathcal{B}(\mathbb{R})$,
\begin{align*}
    \mathbb{E}\left[\sum_{x\in \Upsilon_a(A)}f(x)\right] = \mathbb{E}\left[\int_{A} f(x) \Upsilon_a(\mathrm{d}x)\right] = \int_{A}f(x)\kappa_a(\mathrm{d}x).
\end{align*}
\end{proposition}
We show the proof of Lemma~\ref{lem:kauf_lemma_extnd_infinite} as follows.
\begin{proof}
For each $a \in [K]$, $ x \in \set{X}$, let us denote by $f_{a, x}$ and $f'_{a,x}$ the probability density functions of $ p_{a,x}$ and $q_{a,x}$ with respect to the Lebesgue measure. We have that
\begin{align*}
\EXP_{\set V}\left[\log \left(\frac{f_{a, x}(R_{t, a})}{f_{a, x}'(R_{t, a})}\right) \middle| X_t = x\right] = \KL(p_{a,x}, q_{a,x}).
\end{align*}
Let us define a log-likelihood ratio from the observation under the model $\set{V} = ((p_{a,x}), \zeta)$ to the model $\set{M}=((q_{a,x}), \zeta)$
\begin{align*}
L_\tau = \sum_{t=1}^\tau \sum_{a \in [K]} \indicator\{ A_t = a\} \log \left(\frac{f_{a, X_t}(R_{t,a})}{f_{a, X_t}'(R_{t,a})}\right).
\end{align*}

Let us define $\vec{x}_\infty = (x_1, x_2, \ldots)$, $ \vec{X}_\infty = (X_1, X_2, \ldots)$, and $ \set{X}'(\vec{x}_\infty) = \cup_{t=1}^\infty \{x_t\} $.
We have 
\begin{align*}
    \EXP_\set{V}[L_\tau] & = \EXP_\set{V}[\EXP_\set{V}[L_\tau | \vec{X}_\infty]]
    \\
    & = \int_{\vec{x}_\infty \in \mathbb{R}^\infty}\EXP_\set{V}[L_\tau | \vec{X}_\infty = \vec{x}_\infty ]  \prod_{t=1}^\infty \zeta(x_t) \mathrm{d}\bm{x}_\infty 
    \\
    & = \int_{\vec{x}_\infty \in \mathbb{R}^\infty}\EXP_\set{V}\left[\sum_{t=1}^\tau \sum_{a \in [K]} \indicator\{A_t = a\} \log \frac{f_{a, X_t}(R_{t,a})}{f'_{a, X_t}(R_{t,a})}  \middle| \vec{X}_\infty = \vec{x}_\infty \right]  \prod_{t=1}^\infty \zeta(x_t) \mathrm{d}\bm{x}_\infty
    \\
        & = \int_{\vec{x}_\infty \in \mathbb{R}^\infty}\EXP_\set{V}\left[\sum_{t=1}^\tau \sum_{a \in [K]} \indicator\{A_t = a, X_t = x_t\} \log \frac{f_{a, x_t}(R_{t,a})}{f'_{a, x_t}(R_{t,a})}  \middle| \vec{X}_\infty = \vec{x}_\infty \right]  \prod_{t=1}^\infty \zeta(x_t) \mathrm{d}\bm{x}_\infty
    \\
    & \stackrel{(a)}{=} \sum_{a \in [K]}  \int_{\vec{x}_\infty \in \mathbb{R}^\infty}\EXP_\set{V}\left[\sum_{x \in \set{X}'(\vec{x}_\infty)}  \sum_{k=1}^{N_{a, x}(\tau)}  \log \frac{f_{a, x}(Y_k^{(a,x)})}{f'_{a, x}(Y_k^{(a,x)})}  \middle| \vec{X}_\infty = \vec{x}_\infty \right]  \prod_{t=1}^\infty \zeta(x_t) \mathrm{d}\bm{x}_\infty
    \\
    & = \sum_{a \in [K]}  \int_{\vec{x}_\infty \in \mathbb{R}^\infty} \sum_{x \in \set{X}'(\vec{x}_\infty)}  \EXP_\set{V}\left[ \sum_{k=1}^{N_{a, x}(\tau)}  \log \frac{f_{a, x}(Y_k^{(a,x)})}{f'_{a, x}(Y_k^{(a,x)})}  \middle| \vec{X}_\infty = \vec{x}_\infty \right]  \prod_{t=1}^\infty \zeta(x_t) \mathrm{d}\bm{x}_\infty
    \\
    &  = \sum_{a \in [K]}  \int_{\vec{x}_\infty \in \mathbb{R}^\infty} \sum_{x \in \set{X}'(\vec{x}_\infty)}  \EXP_\set{V}\left[N_{a, x}(\tau) | \vec{X}_\infty = \vec{x}_\infty \right] \EXP_\set{V}\left[\log \frac{f_{a, x}(Y_1^{(a,x)})}{f'_{a, x}(Y_1^{(a,x)})}  \middle| \vec{X}_\infty = \vec{x}_\infty \right]  \prod_{t=1}^\infty \zeta(x_t) \mathrm{d}\bm{x}_\infty
    \\
    & = \sum_{a \in [K]}  \int_{\vec{x}_\infty \in \mathbb{R}^\infty} \sum_{x \in \set{X}'(\vec{x}_\infty)}  \EXP_\set{V}\left[N_{a, x}(\tau) | \vec{X}_\infty = \vec{x}_\infty \right] \KL(p_{a,x}, q_{a,x})  \prod_{t=1}^\infty \zeta(x_t) \mathrm{d}\bm{x}_\infty
    \\
    & = \sum_{a \in [K]}  \int_{\vec{x}_\infty \in \mathbb{R}^\infty}\EXP_\set{V}\left[ \sum_{x \in \set{X}'(\vec{x}_\infty)}  N_{a, x}(\tau)  \KL(p_{a,x}, q_{a,x}) | \vec{X}_\infty = \vec{x}_\infty \right] \prod_{t=1}^\infty \zeta(x_t) \mathrm{d}\bm{x}_\infty
    \\
    &  \stackrel{(b)}{=} \sum_{a \in [K]} \EXP_{\set V} \left[ \int_{\mathbb R} \KL(p_{a,x}, q_{a,x}) \Upsilon_a (\mathrm{d}x)\right]
    \\
    & \stackrel{(c)}{=}  \sum_{a \in [K]} \int_{\mathbb R} \KL(p_{a,x}, q_{a,x}) \kappa_a (\mathrm{d}x) 
    \\
    & =  \sum_{a \in [K]} \int_{x \in \mathbb{R}} \frac{\mathrm{d} \kappa_a}{\mathrm{d}x} (x) \KL(p_{a,x}, q_{a,x}) \mathrm{d}x.
\end{align*}
For $(a)$, we introduced random variable $Y_{k}^{(a,x)}$, denoting $k$-th time the reward with the context $x$ and the action $a$ is observed. For $(b)$, the computation is as follows:
\begin{align*}
    &\sum_{a \in [K]}  \int_{\vec{x}_\infty \in \mathbb{R}^\infty}\EXP_\set{V}\left[ \sum_{x \in \set{X}'(\vec{x}_\infty)}  N_{a, x}(\tau)  \KL(p_{a,x}, q_{a,x}) | \vec{X}_\infty = \vec{x}_\infty \right] \prod_{t=1}^\infty \zeta(x_t) \mathrm{d}\bm{x}_\infty
    \\
    &=\sum_{a \in [K]}  \EXP_\set{V}\left[ \sum_{x \in \set{X}'(\vec{X}_\infty)}  N_{a, x}(\tau)  \KL(p_{a,x}, q_{a,x}) \right]
    \\
    & =\sum_{a \in [K]}  \EXP_\set{V}\left[ \sum_{x \in (X_1, \ldots, X_\tau)}  N_{a, x}(\tau)  \KL(p_{a,x}, q_{a,x}) \right]
    \\
     &=\sum_{a \in [K]}  \EXP_\set{V}\left[  \int_{\mathbb{R}} \KL(p_{a,x}, q_{a,x}) \Upsilon_a (\mathrm{d}x)\right],
\end{align*}
where the last equality follows from the definition of $ \Upsilon_a$.
For $(c)$, we used Campbell’s theorem (Proposition~\ref{prp:campbell}).

From the data-processing inequality applied to the change-of-measure argument \citet{garivier2019explore}, we have, for any $\set{E} \in \set{G}_\tau$, 
\begin{align*}
\sum_{a \in [K]} \int_{\mathbb{R}} \frac{\mathrm{d} \kappa_a}{\mathrm{d} x}(x)\KL(p_{a,x}, q_{a,x})\mathrm{d}x \ge d(\Pr_{\set{V}}(\set{E}), \Pr_{\set{M}}(\set{E})).
\end{align*}
This concludes the proof of Lemma~\ref{lem:kauf_lemma_extnd_infinite}.
\end{proof}

Then, we show the proof of Theorem~\ref{thm:lower_bound_continuous}.
\begin{proof}
From Lemma~\ref{lem:kauf_lemma_extnd_infinite} with $\set{E} =\{\hat{a}_\tau = {a}^*(\set{V})\}$, for each $\mathcal{V} \in \Omega$ and $\set{M} \in \Alt(\set{V})$, we have
\begin{align*}
\sum_{a \in [K]} \int_{\mathbb{R}} \frac{\mathrm{d} \kappa_a}{\mathrm{d} x}(x)\KL(p_{a,x}, q_{a,x})\mathrm{d}x & \ge \kl(\Pr_\set{V}(\set{E}), \Pr_\set{M}(\set{E}))
 \ge d(\delta, 1- \delta),
\end{align*}
where for the last inequality, we used the definition of the $\delta$-PAC algorithm and monotonicity of the KL divergence.

For each $\set{V} \in \Omega$, we have
\begin{align*}
    d(\delta, 1-\delta) & \le \inf_{(\vec{p}, \zeta) \in \Alt(\set{V})}\sum_{a \in [K]} \int_{\mathbb{R}} \frac{\mathrm{d} \kappa_a}{\mathrm{d} x}(x)\KL(p_{a,x}, q_{a,x})\mathrm{d}x
    \\
    & = \inf_{(\vec{p}, \zeta) \in \Alt(\set{V})} \sum_{a \in [K]} \int_{\mathbb{R}} \frac{\mathrm{d}\kappa_a}{\mathrm{d} \gamma} \frac{\mathrm{d} \gamma}{\mathrm{d} \zeta}(x)\zeta(x) \KL(p_{a,x}, q_{a,x})\mathrm{d}x
    \\
    & {=} \inf_{(\vec{p}, \zeta) \in \Alt(\set{V})} \int_{\mathbb{R}} \underbrace{\frac{\mathrm{d} \gamma}{\mathrm{d} \zeta}(x)\zeta(x)}_{\EXP_{\set{V}}[\tau_\delta] \zeta(x)}\sum_{a \in [K]}  {\frac{\mathrm{d} \kappa_a}{\mathrm{d} \gamma} (x)}\KL(p_{a,x}, q_{a,x})\mathrm{d}x
    \\
    & \stackrel{(a)}{=} \EXP_{\set{V}}[\tau_\delta] \inf_{(\vec{p}, \zeta)\in \Alt(\set{V})} \int_{\mathbb{R}} \sum_{a \in [K]}  {\frac{\mathrm{d} \kappa_a}{\mathrm{d} \gamma} (x)}\KL(p_{a,x}, q_{a,x}) \zeta(x)\mathrm{d}x
    \\
      & \le \EXP_{\set{V}}[\tau_\delta] \sup_{w \in \set{W}} \inf_{(\vec{p}, \zeta) \in \Alt(\set{V})} \int_{\mathbb{R}} \sum_{a \in [K]} {w_{a,x}} \KL(p_{a,x}, q_{a,x}) \zeta(x)\mathrm{d}x,
\end{align*}
where for $(a)$ we used the equivalence:
\begin{align*}
      \EXP_{\set{V}}[\tau_\delta] & =  \int_{\mathbb{R}} \frac{\mathrm{d} \gamma}{\mathrm{d}x} \mathrm{d}x
      \\
      & = \int_{\mathbb{R}} \frac{\mathrm{d} \gamma}{\mathrm{d} \zeta} \zeta(x) \mathrm{d}x
      \\
      & \stackrel{(b)}{=} \frac{\mathrm{d} \gamma}{\mathrm{d} \zeta} \int_{\mathbb{R}} \zeta(x) \mathrm{d}x
      \\
      & = \frac{\mathrm{d} \gamma}{\mathrm{d} \zeta},
\end{align*}
where for $(b)$, we used the fact that $\frac{\mathrm{d} \gamma}{\mathrm{d} \zeta}$ is a constant does not depend on $x$.

\end{proof}

\section{Proof of Results in Section~\ref{sec:cont_2arm}  and  \texorpdfstring{$\alpha$}{TEXT}-Elimination Algorithm with Contextual Information}

\subsection{\texorpdfstring{$\alpha$}{TEXT}-Elimination Algorithm with Contextual Information}
\label{appdx:alpha_elim}
We use an algorithm that is almost identical to the $\alpha$-elimination of \citet{Kaufman2016complexity}. The only difference between the proposed $\alpha$-elimination and that of \citet{Kaufman2016complexity} is that we construct an estimator of the marginalized mean reward in the following form:
\begin{align*}
&\hat{\mu}_1(t) = \frac{1}{\sum^t_{s=1} \mathbbm{1}[A_s = 1]}\sum^t_{s=1}\left( R_{s, 1} - \frac{\rho_{\set{X}1} \sigma_1}{\sigma_{\set{X}}}(X_s - \mu_{\set{X}})\right)\mathbbm{1}[A_s = 1],\\
&\hat{\mu}_2(t) = \frac{1}{\sum^t_{s=1} \mathbbm{1}[A_s = 2]}\sum^t_{s=1}\left( R_{s, 2} - \frac{\rho_{\set{X}2} \sigma_2}{\sigma_{\set{X}}}(X_s - \mu_{\set{X}})\right)\mathbbm{1}[A_s = 2].
\end{align*}
Here, we used that $\mu_{a} = \mu_{a, x} - \frac{\rho_{\set{X}a} \sigma_a}{\sigma_{\set{X}}}(x - \mu_{\set{X}})$. 
This estimator is based on the form of the conditional distribution of $R_{t,a}$. We replace $\hat{\mu}_a(t)$ in the original $\alpha$-elimination with these estimators.

\subsection{Proof of Theorem~\ref{thm:cont_2arm_lowerbound}}
\label{appdx:thm:cont_2arm_lowerbound}

Recall that the KL divergence from $\set{N}(\mu_1, \sigma^2)$ to $\set{N}(\mu_2, \sigma^2)$ is given as
\begin{align*}
    \KL\left(\set{N}(\mu_1, \sigma^2),  \set{N}(\mu_2, \sigma^2)\right) = \frac{(\mu_1 - \mu_2)^2}{2 \sigma^2}.
\end{align*}If we ignore sets of measure zero, we have
\begin{align*}
    T^\star(\set{V})^{-1} & = \sup_{\bm{w} \in\set{W}} \inf_{ (\vec{q}, \zeta)\in\Alt(\mathcal{V})}\sum^2_{a=1} \int_{\mathbb{R}}w_{a, x}\KL(p_{a, x}, q_{a, x}) \zeta(x) \mathrm{d}x
    \\
    & = \sup_{\bm{w} \in\set{W}} \inf_{ (\vec{q}, \zeta)\in\Alt(\mathcal{V})}\sum^2_{a=1} \int_{\mathbb{R}}w_{a, x} \frac{\left(\mu_a + \frac{\sigma_{\set{X}a}}{\sigma_{\set{X}}^2}(x - \mu_{\set{X}}) - \lambda_{a, x}\right)^2}{2\sigma_a'^2}\zeta(x) \mathrm{d}x
    \\
    & =  \sup_{\bm{w} \in\set{W}} \inf_{\int_{\mathbb{R}}\lambda_{2, x} \zeta(x) \mathrm{d}x > \int_{\mathbb{R}}\lambda_{1, x} \zeta(x) \mathrm{d}x}\sum^2_{a=1} \int_{\mathbb{R}}w_{a, x} \frac{\left(\mu_a + \frac{\sigma_{\set{X}a}}{\sigma_{\set{X}}^2}(x - \mu_{\set{X}}) - \lambda_{a, x}\right)^2}{2\sigma_a'^2}\zeta(x) \mathrm{d}x
    \\
    & \stackrel{(a)}{=}  \max_{\bm{w} \in\set{W}} \min_{\int_{\mathbb{R}}\lambda_{2, x} \zeta(x) \mathrm{d}x = \int_{\mathbb{R}}\lambda_{1, x} \zeta(x) \mathrm{d}x}\sum^2_{a=1} \int_{\mathbb{R}}w_{a, x} \frac{\left(\mu_a + \frac{\sigma_{\set{X}a}}{\sigma_{\set{X}}^2}(x - \mu_{\set{X}}) - \lambda_{a, x}\right)^2}{2\sigma_a'^2}\zeta(x) \mathrm{d}x
\end{align*}
where for $(a)$, we used the same argument as in Lemma~\ref{lem:allEqual}.  
From the property of the multivariate Gaussian distribution, 
\begin{align*}
\lambda_{1,x}  = \lambda_1 + \frac{\sigma_{\set{X}1}}{\sigma_{\set{X}}^2}(x - \mu_{\set{X}}) \quad \textnormal{and} \quad \lambda_{2,x}  = \lambda_2 + \frac{\sigma_{\set{X}2}}{\sigma_{\set{X}}^2}(x - \mu_{\set{X}}).
\end{align*}
From $\int_{\mathbb{R}}\lambda_{2, x} \zeta(x) \mathrm{d}x = \int_{\mathbb{R}}\lambda_{1, x} \zeta(x) \mathrm{d}x$, $\lambda_1 = \lambda_2 = \lambda$. Therefore, we get
\begin{align*}
    \frac{1}{2 \sigma'^2_{a}}\left(\mu_a + \frac{\sigma_{\set{X}a}}{\sigma_{\set{X}}^2}(x - \mu_{\set{X}}) - \lambda_{a, x}\right)^2 = \frac{1}{2 \sigma'^2_{a}}\left(\mu_a - \lambda\right)^2.
\end{align*}
Therefore, the optimization problem can be further simplified
\begin{align*}
    T^\star(\set{V})^{-1} &= \max_{\bm{w} \in\set{W}} \min_{\int_{\mathbb{R}}\lambda_{2, x} \zeta(x) \mathrm{d}x = \int_{\mathbb{R}}\lambda_{1, x} \zeta(x) \mathrm{d}x}\sum^2_{a=1} \int_{\mathbb{R}}w_{a, x} \frac{\left(\mu_a + \frac{\sigma_{\set{X}a}}{\sigma_{\set{X}}^2}(x - \mu_{\set{X}}) - \lambda_{a, x}\right)^2}{2\sigma_a'^2}\zeta(x) \mathrm{d}x
    \\
    & = \max_{\bm{w} \in\set{W}} \min_{\lambda \in \mathbb{R}} \int_{\mathbb{R}}\sum^2_{a=1} w_{a, x} \frac{\left( \mu_a - \lambda\right)^2}{2\sigma_a'^2} \zeta(x) \mathrm{d}x.
\end{align*}
At each point $x \in \mathbb{R}$, the optimization problem
\begin{align*}
    \max_{w_{1, x} + w_{2,x} = 1} \min_{\lambda \in \mathbb{R}} \sum^2_{a=1} w_{a, x} \frac{\left( \mu_a - \lambda\right)^2}{2\sigma_a'^2}
\end{align*}
is an identical problem  as is given in Theorem~6 in \citet{Kaufman2016complexity} (two arm Gaussian bandits with known variances) and we know from Theorem~9 in \citet{Kaufman2016complexity}, the maximum is attained when $w_{1, x}= \sigma'_1/( \sigma'_1 +  \sigma'_2)$. Thus, we compute
\begin{align*}
    T^\star(\set{V})^{-1} & = \min_{\lambda \in \mathbb{R}} \int_{\mathbb{R}}\sum^2_{a=1} \frac{\sigma'_a}{\sigma'_1 +  \sigma'_2} \frac{\left( \mu_a - \lambda\right)^2}{2\sigma_a'^2} \zeta(x) \mathrm{d}x = \frac{1}{\sigma'_1 +  \sigma'_2} \min_{\lambda \in \mathbb{R}} \int_{\mathbb{R}}\sum^2_{a=1} \frac{\left( \mu_a - \lambda\right)^2}{2\sigma_a'} \zeta(x) \mathrm{d}x,
\end{align*}

When the minimum is attained,
\begin{align*}
     - \frac{1}{\sigma'_{1}}\left(\mu_1 - \lambda\right) - \frac{1}{\sigma'_{2}}\left(\mu_2 - \lambda\right) =  0,
\end{align*}
Therefore, 
\begin{align*}
    \lambda = \frac{\frac{1}{\sigma'_1} \mu_1 + \frac{1}{\sigma'_2}\mu_2 }{\frac{1}{\sigma'_1}  + \frac{1}{\sigma'_2} }.
\end{align*}
Then, 
\begin{align*}
    \sum^2_{a=1} \frac{1}{2 \sigma'_{a}}\left(\mu_a + \frac{\sigma_{\set{X}a}}{\sigma_{\set{X}}^2}(x - \mu_{\set{X}}) - \lambda_{a, x}\right)^2 & = \sum^2_{a=1} \frac{1}{2 \sigma'_{a}}\left(\mu_a - \lambda \right)^2 
    \\
    & = \left(\frac{\mu_1 - \mu_2 }{\sigma'_1 + \sigma'_2} \right)^2 \frac{\sigma'_1}{2} + \left(\frac{\mu_1 - \mu_2}{\sigma'_1 + \sigma'_2}\right)^2 \frac{\sigma'_2}{2}
    \\
    & = \frac{(\mu_1 - \mu_2 )^2}{2(\sigma'_1 + \sigma'_2)} 
\end{align*}

Therefore, we have
\begin{align*}
    T^\star(\set{V})^{-1} = \frac{1}{2}\left(\frac{\mu_1 - \mu_2}{\sigma'_1 + \sigma'_2}\right)^2.
\end{align*}
\ep

\subsection{Proof of Theorem~\ref{thm:cont_2arm_upperbound}}
\label{appdx:thm:cont_2arm_upperbound}
We note that except that the variances of the sample from the arm $a$ is $\sigma_a'^2$, the proof is almost identical to that of Theorem~9 of \citet{Kaufman2016complexity}.
Let $\alpha = {\sigma_1'}/{(\sigma_1'+\sigma_2')}$ and $d_t =  \hat{\mu}_1(t) - \hat{\mu}_2(t)$. We first prove that the strategy
is $\delta$-PAC for every $\set{V} \in \tilde{\Omega}$.  Assume that $\mu_1>\mu_2$ and recall $\tau=\inf\{t\in \mathbb{N} :
|d_t| > \sqrt{2\sigma^2_t(\alpha)\beta(t,\delta)}\}$, where
$d_t:=\hat{\mu}_1(t) - \hat{\mu}_2(t)$.  The probability of error of the
$\alpha$-elimination strategy is upper bounded by
\begin{eqnarray*}
  \Pr_\set{V}\left(d_\tau \leq -\sqrt{{2\sigma_\tau^2(\alpha) \beta(\tau,\delta)}} \right) &\leq& \Pr_\set{V}\left(d_\tau-(\mu_1-\mu_2) \leq - \sqrt{{2\sigma_\tau^2(\alpha) \beta(\tau,\delta)}} \right) \\
  &\leq& \Pr_\set{V}\left(\exists t\in \mathbb{N}^* : d_t - (\mu_1 - \mu_2) <  - \sqrt{{2\sigma_t^2(\alpha) \beta(t,\delta)}}\right)\\
  & \leq & \sum_{t=1}^{\infty} \exp\left(-\beta(t,\delta)\right),
\end{eqnarray*}
where we used union bound and Chernoff bound applied to $d_t - (\mu_1 - \mu_2)\sim
\set{N}(0, \sigma^2_t(\alpha))$ in the last inequality.  We have
\begin{eqnarray*}\sum_{t=1}^{\infty} \exp\left(-\beta(t,\delta)\right)&\leq& \delta \sum_{t=1}^{\infty}
\frac{1}{t(\log(6t))^2}\leq {\delta}\left(\frac{1}{(\log 6)^2} +
  \int_{1}^{\infty}\frac{dt}{t(\log(6t))^2}\right) \\
&=&{\delta}\left(\frac{1}{(\log 6)^2} + \frac{1}{\log(6)}\right)\leq \delta.
\end{eqnarray*}

For the guarantee of the expected sample complexity, we first prove the
probability that $\tau$ exceeds some fixed $T$:
\begin{eqnarray*}
  \Pr_{\set{V}}(\tau \geq T) &\leq& \Pr_{\set{V}}\left(\forall t \in [T], \ d_t \leq \sqrt{{2\sigma_t^2(\alpha)\beta(t,\delta)}}\right) 
  \\
  &\leq& \Pr_{\set{V}}\left(d_T \leq \sqrt{{2\sigma_T^2(\alpha)\beta(T,\delta)}}\right) \\
  &=& \Pr_{\set{V}}\left( d_T - (\mu_1 - \mu_2) \leq -\left[(\mu_1 - \mu_2)-\sqrt{{2\sigma_T^2(\alpha)\beta(T,\delta)}}\right]\right) \\
  &\leq & \exp\left(-\frac{1}{2\sigma_T^2(\alpha)}\left[(\mu_1 - \mu_2)-\sqrt{{2\sigma_T^2(\alpha)\beta(T,\delta)}}\right]^2\right),
\end{eqnarray*}
where for the last inequality we used Chernoff bound with $T$ such that
$(\mu_1 - \mu_2)>\sqrt{2\sigma^2_T(\alpha)\beta(T,\delta)}.$ For $\gamma
\in (0,1)$, define
\begin{eqnarray*}
  T^*_{\gamma} & := & \inf\left\{t_0 \in \mathbb{N} : \forall t \geq t_0, (\mu_1-\mu_2) - \sqrt{2\sigma^2_t(\alpha)\beta(t,\delta)} > \gamma (\mu_1 - \mu_2)\right\}. 
\end{eqnarray*}
We have,
\begin{eqnarray*}
  \EXP_{\set{V}} [\tau]& \leq & T_\gamma^* + \! \sum_{T=T_{\gamma}^*+1}^\infty\mathbb{P}\left(\tau \geq T\right) \\
  &\leq & T_\gamma^* + \! \sum_{T=T_{\gamma}^*+1}^\infty \exp\left(-\frac{1}{2\sigma_T^2(\alpha)}\left[(\mu_1 - \mu_2)-\sqrt{{2\sigma_T^2(\alpha)\beta(T,\delta)}}\right]^2\right) \\
  & \leq & T_\gamma^* + \! \sum_{T=T_{\gamma}^*+1}^\infty\exp\left(-\frac{1}{2\sigma_T^2(\alpha)}\gamma^2(\mu_1 - \mu_2)^2\right). 
\end{eqnarray*}
For all $t$, it is easy to show that the following upper bound on
$\sigma_t^2(\alpha)$ holds:
\begin{equation}
    \sigma_t^2(\alpha) \leq \frac{(\sigma_1' + \sigma_2')^2}{t}\times \frac{t - \frac{\sigma_1'}{\sigma_2'}}{t-\frac{\sigma_1'}{\sigma_2'} -1}.
  \label{ineq:BoundSigma}
\end{equation}
Using the inequality (\ref{ineq:BoundSigma}), we have
\begin{eqnarray*}
  \EXP_{\set{V}}[\tau]& \leq & T_\gamma^* + \int_{0}^{\infty} \exp\left(-\frac{t}{2(\sigma_1' + \sigma_2')^2} \frac{t - \frac{\sigma_1'}{\sigma_2'}-1}{t - \frac{\sigma_1'}{\sigma_2'}}\gamma^2(\mu_1-\mu_2)^2\right)dt\\
  & \leq &  T_\gamma^* +  \frac{2(\sigma_1'+\sigma_2')^2}{\gamma^2(\mu_1-\mu_2)^2}\exp\left(\frac{\gamma^2(\mu_1-\mu_2)^2}{2(\sigma_1'+\sigma_2')^2}\right).
\end{eqnarray*}
Next, we upper bound $T_\gamma^*$. Let $r\in[0,e/2-1]$. There exists
$N_0(r)$ such that for $t\geq N_0(r)$, $\beta(t,\delta)\leq \log
({t^{1+r}}/{\delta})$.  Again, using the inequality (\ref{ineq:BoundSigma}), we have
$T_\gamma^*=\max(N_0(t),\tilde{T}_\gamma)$, where $\tilde{T}_\gamma$ is defined as
$$\tilde{T}_\gamma = \inf \left\{ t_0 \in \mathbb{N} : \forall t \geq t_0, \frac{(\mu_1 - \mu_2)^2}{2(\sigma_1'+\sigma_2')^2}(1-\gamma)^2 t > \frac{t - \frac{\sigma_1'}{\sigma_2'}-1}{t - \frac{\sigma_1'}{\sigma_2'}} \log \frac{t^{1+r}}{\delta}\right\}.$$
When $t > (1 + \gamma \frac{\sigma_1'}{\sigma_2'})/{\gamma}$, $(t -
\frac{\sigma_1'}{\sigma_2'}-1)/(t - \frac{\sigma_1'}{\sigma_2'}) \leq
{(1-\gamma)^{-1}}$. We get $\tilde{T}_\gamma =\max((1 + \gamma
\frac{\sigma_1'}{\sigma_2'})/{\gamma}, T_\gamma')$, with
$$T_\gamma' = \inf \left\{ t_0 \in \mathbb{N} : \forall t\geq t_0, \exp\left(\frac{(\mu_1 - \mu_2)^2}{2(\sigma_1'+\sigma_2')^2}(1-\gamma)^3 t\right) \geq \frac{t^{1+r}}{\delta}\right\}.$$ 
We use the following algebraic Lemma by \citet{Kaufman2016complexity}.
\begin{lemma}[Lemma~22 of \citet{Kaufman2016complexity}]\label{lem:Garivier2} For every $\beta,\eta>0$ and $s\in[1,e/2]$,
  the following implication is true:
$$x_0 = \frac{s}{\beta}\log\left(\frac{e\log\left({1}/{(\beta^s\eta)}\right)}{\beta^s\eta}\right)  \ \ \ \Rightarrow \ \ \ \forall x\geq x_0, \ \ e^{\beta x} \geq \frac{x^s}{\eta}.$$
\end{lemma}
Applying Lemma \ref{lem:Garivier2} with $\eta=\delta$, $s=1+r$ and $\beta
={(1-{\gamma})^3(\mu_1-\mu_2)^2}/{(2(\sigma_1'+\sigma_2')^2)}$ leads to
\[T_\gamma' \leq\frac{(1+r)}{(1-\gamma)^3}\times \frac{2(\sigma_1' +
  \sigma_2')^2}{(\mu_1 - \mu_2)^2}\left[\log \frac{1}{\delta} + \log\log
  \frac{1}{\delta} \right] + R(\mu_1,\mu_2,\sigma_1',\sigma_2',\gamma,r),\] with
\[R(\mu_1,\mu_2,\sigma_1',\sigma_2',\gamma,r)=\frac{1+r}{(1-\gamma)^3}\frac{2(\sigma_1'+\sigma_2')^2}{(\mu_1-\mu_2)^2}\left[1
  +
  (1+r)\log\left(\frac{2(\sigma_1'+\sigma_2')^2}{(1-\gamma)^3(\mu_1-\mu_2)^2}\right)\right].\]
For fixed $\epsilon>0$, choosing small enough $r$ and $\gamma$,  we have
$$\EXP_\set{V}[\tau] \leq (1+\epsilon)\frac{2(\sigma_1' + \sigma_2')^2}{(\mu_1 - \mu_2)^2}\left[\log \frac{1}{\delta} + \log\log\frac{1}{\delta}\right] + \set{C}(\mu_1,\mu_2,\sigma_1',\sigma_2',\epsilon),$$
where $\set{C}$ is a constant independent of $\delta$ summarizing the terms: $ R(\mu_1,\mu_2,\sigma_1',\sigma_2',\gamma,r)$, $(1 + \gamma
\frac{\sigma_1'}{\sigma_2'})/{\gamma} $, $N_0(t)$, and $ \frac{2(\sigma_1'+\sigma_2')^2}{\gamma^2(\mu_1-\mu_2)^2}\exp\left(\frac{\gamma^2(\mu_1-\mu_2)^2}{2(\sigma_1'+\sigma_2')^2}\right)$.  $\set{C}(\mu_1,\mu_2,\sigma_1',\sigma_2',\epsilon)$ goes to infinity when $\epsilon$ goes to zero, but for a fixed $\epsilon>0$, 
\[(1+\epsilon)\frac{2(\sigma_1' + \sigma_2')^2}{(\mu_1 - \mu_2)^2}\log\log\frac{1}{\delta} + \set{C}(\mu_1,\mu_2,\sigma_1',\sigma_2',\epsilon) = \underset{\delta\rightarrow 0}{o_{\epsilon}}\left(\log \frac{1}{\delta}\right).\]
This concludes the proof.

\section{Proof of Results in Section~\ref{sec:optimal}}

\subsection{Proof of Lemma~\ref{lem:CDinterpretation}}
\label{appdx:CDinterpretation}

\begin{proof}
We have
\begin{align*}
    \Alt(\nu) & = \{((\lambda_{a,x}), \zeta)\in \Theta: a^*(((\lambda_{a,x}), \zeta) ) \neq a^*(\nu) = 1\}
    \\
    & = \bigcup_{a\neq 1}  \Big\{ ((\lambda_{a,x}), \zeta)\in \Theta: \sum_{x \in \set{X}} \zeta_x \lambda_{a, x} > \sum_{x \in \set{X}} \zeta_x \lambda_{1, x} \Big\}.
\end{align*}

Then, we get
\begin{align*}
    &  \inf_{((\lambda_{a,x}), \zeta) \in\Alt(\nu)}\sum_{x \in \set{X}}\zeta_x\sum^K_{a=1}w_{a, x}\kl(\mu_{a,x}, \lambda_{a,x})
    \\
    & = \inf_{((\lambda_{a,x}), \zeta) : \exists a \in [K], \sum_{x \in \set{X}} \zeta_x \lambda_{a, x} > \sum_{x \in \set{X}} \zeta_x \lambda_{1, x} }\sum_{x \in \set{X}}\zeta_x\sum^K_{a=1}w_{a, x}\kl(\mu_{a,x}, \lambda_{a,x})
    \\
    & =   \min_{a\neq 1}  \inf_{((\lambda_{a,x}), \zeta) : \sum_{x \in \set{X}} \zeta_x \lambda_{a, x} > \sum_{x \in \set{X}} \zeta_x \lambda_{1, x}} \sum_{x \in \set{X}}\zeta_x\sum^K_{a=1}w_{a, x}\kl(\mu_{a,x}, \lambda_{a,x})
    \\
    & = \min_{a\neq 1}  \inf_{((\lambda_{a,x}), \zeta) : \sum_{x \in \set{X}} \zeta_x \lambda_{a, x} > \sum_{x \in \set{X}} \zeta_x \lambda_{1, x}} \sum_{x \in \set{X}}\zeta_x\Big(w_{1, x}\kl(\mu_{1,x}, \lambda_{1,x}) + w_{a, x}\kl(\mu_{a,x}, \lambda_{a,x})\Big).
\end{align*}
\end{proof}

\subsection{Proof of Lemma~\ref{lem:allEqual}}
\label{appdx:allEqual}

\begin{proof}
Let $a \in [K]$ be one of the arguments that minimizes 
\begin{align*}
    \inf_{((\lambda_{a,x}), \zeta) : \sum_{x \in \set{X}} \zeta_x \lambda_{a, x} > \sum_{x \in \set{X}} \zeta_x \lambda_{1, x}} f_a((\lambda_{x,a})) 
\end{align*}
and suppose $\sum_{x \in \set{X}}\zeta_x \lambda^*_{a,x} > \sum_{x \in \set{X}}\zeta_x \lambda^*_{1,x}$.
For such $a$, from the assumption on $\Theta$, there exists $x \in \set{X}$ such that $\mu_{1,x} > \mu_{a, x}$. For such $x$, from the monotonicity of the KL divergence, 
\begin{align*}
    \mu_{1, x} \ge \max(\lambda_{1, x}, \lambda_{a, x}) \ge  \min(\lambda_{1, x}, \lambda_{a, x}) \ge \mu_{a, x}.
\end{align*}
Then, by the assumption $\sum_{x \in \set{X}}\zeta_x \lambda^*_{a,x} > \sum_{x \in \set{X}}\zeta_x \lambda^*_{1,x} $, one can modify the value of $ \lambda^*_{1,x}$ as $\lambda^*_{1,x} + \varepsilon$ or $ \lambda^*_{a,x}$ as $\lambda^*_{a,x} - \varepsilon$ ($\varepsilon$ is some small constant) to make the value of $f_{a}((\lambda_{x,a}))$ strictly smaller.  This is a contradiction and concludes the proof. 

\end{proof}

\subsection{Proof of Lemma~\ref{lem:conti_V}}
\label{appdx:conti_V}

\begin{proof}
Let us define a function 
\begin{align*}
    f(\nu, (\lambda_{a,x})) = \sum_{x \in \set{X}}\zeta_x\sum^K_{a=1}w_{a, x}\kl(\mu_{a,x}, \lambda_{a,x}).
\end{align*}
We call the point-to-set mapping
\begin{align*}
    X(\nu) = \bigcup_{a\neq 1} \Big\{ ((\lambda_{a,x}), \zeta_x) \in \Theta: \sum_{x \in \set{X}} \zeta_x \lambda_{a, x}> \sum_{x \in \set{X}} \zeta_x \lambda_{1, x} \big\}
\end{align*}
as a constraint mapping. It is easy to check that $ X(\nu)$ is 
outer semicontinuous at every $\nu$. 
Similarly, $ X(\nu)$ is 
inner semicontinuous at every $\nu$. Therefore, from the stability theory in optimization \citet{hogan1973point} and the continuity of the KL divergence, $m(\bm w, \nu)$ is continuous at every $\nu$ when $\bm w$ is fixed. 

\end{proof}

\subsection{Proof of Lemma~\ref{lem:continui_w}}
\label{appdx:continui_w}

\begin{proof}
The proof is similar to that of Lemma~\ref{lem:conti_V}. The constraint $\sum_{x\in \set{X}} \zeta_x \lambda_{a, x}> \sum_{x \in \set{X}} \zeta_x \lambda_{1, x}$ is invariant under the changes of $\vec{w} \in \set{W}$ and the KL divergence is continuous. From the stability theory \citet{hogan1973point}, $m(\vec{w}, \set{W})$ is continuous when $\nu$ is fixed. 
\end{proof}

\subsection{Proof of Lemma~\ref{lem:converge_w}}
\label{appdx:converge_w}

\begin{proof}
Suppose $\vec{w}_k$ does not converge to $\Phi(\nu)$. Then, there exists $\varepsilon >0$ such that for any $n_1 \in \mathbb{N}$, there exists $k \ge n_1$ such that
\begin{align*}
        \inf_{(w_{a,x}) \in \bar{\set{W}}} \max_{a, x}| w_{a,x}^{(k)} - w_{a,x}| \ge \varepsilon.
\end{align*}
Also, there exists $C(\varepsilon)>0$ such that,
\begin{align}\label{eq:exists_max_difference}
    \max_{\vec{w} \in \set{W}} m(\vec{w}, \nu) - \max_{w \in \set{W} : \inf_{(w_{a,x}) \in \bar{\set{W}}} \max_{a, x}| w_{a,x}^{(k)} - w_{a,x}|\ge \varepsilon}  m(\vec{w}, \nu) \ge C(\varepsilon).
\end{align}
Let $w^* \in \argmax m(\vec{w}, \nu)$. We can find a constant $\varepsilon_2(C(\varepsilon))>0$ such that for any $n_2 \in \mathbb{N}, n_2 \ge n_1$, there exists $k \ge n_2$ such that 
\begin{align*}
    & |\max_{w \in \set{W}} m(\vec{w}, \nu) - \max_{w \in \set{W}} m(\vec{w}, \nu_k)| 
    \\
    & = | m(\vec{w}^*, \nu) - m(\vec{w}^*, \nu_k) + m(\vec{w}^*, \nu_k) -\max_{w \in \set{W}} m(\vec{w}, \nu_k) |
    \\
    & \ge \left||m(\vec{w}^*, \nu) - m(\vec{w}^*, \nu_k)| - |m(\vec{w}^*, \nu_k) -\max_{w \in \set{W}} m(\vec{w}, \nu_k) | \right|
    \\
    & \stackrel{(a)}{\ge} \varepsilon_2 (C(\varepsilon)),
\end{align*}
where for $(a)$, we used (i) $ m(\vec{w}^*, \nu) \to m(\vec{w}^*, \nu_k)$: from the continuity of $m(\vec{w}, \nu)$ with respect to $\nu$ for a fixed $w$ (Lemma~\ref{lem:conti_V}) with the convergence assumption of $ (\nu_k)_{k \ge 1}$ and (ii) $| m(\vec{w}^*, \nu_k) -\max_{w \in \set{W}} m(\vec{w}, \nu_k)  | \ge C(\varepsilon)$: from the optimality gap (\ref{eq:exists_max_difference}). Therefore, $ \max_{w \in \set{W}} m(\vec{w}, \nu_k)$ does not converge to $\max_{w \in \set{W}} m(\vec{w}, \nu)$, hence contradiction. 

\end{proof}

\subsection{Proof of Lemma~\ref{lem:aloca_convex}}
\label{appdx:aloca_convex}
\begin{proof}
Take any $(x_{a, x}^*), (y_{a, x}^*) \in \Phi(\nu)$ and any $\alpha \in [0, 1]$. We have 
\begin{align*}
    & m(\alpha(x_{a, x}^*) + (1 - \alpha)(y_{a, x}^*), \nu) 
    \\
    & = \inf_{((\lambda_{a,x}), \zeta) \in \Alt(\mathcal{\nu})} \sum_{x\in\mathcal{X}}\zeta_x\sum_{a=1}^{K} ( \alpha x_{a, x}^* + (1 - \alpha)y_{a, x}^*)\kl(\mu_{a, x}, \lambda_{a, x})
    \\
    & \ge \alpha\inf_{((\lambda_{a,x}), \zeta) \in \Alt(\mathcal{\nu})}  \sum_{x\in\mathcal{X}}\zeta_x\sum_{a=1}^{K}   x_{a, x}^* \kl(\mu_{a, x}, \lambda_{a, x}) + (1- \alpha)\inf_{((\lambda_{a,x}), \zeta) \in \Alt(\nu)}  \sum_{x\in\mathcal{X}}\zeta_x\sum_{a=1}^{K}   y_{a, x}^* \kl(\mu_{a, x}, \lambda_{a, x}) 
    \\
    & = \max_{\vec{w}' \in \set{W}} m (\vec{w}', \nu).
\end{align*}
Hence, $\alpha(x_{a, x}^*) + (1 - \alpha)(y_{a, x}^*) \in \Phi(\nu)$. This concludes the proof.
\end{proof}

\section{Proofs of Results in Section~\ref{sec:trak_stop} and CTS Algorithm}
\subsection{Proof of Lemma~\ref{lmm:tracking}}
\label{appdx:tracking}
Our proof for the tracking lemma is inspired by that of D-tracking for linear bandits by \citet{jedra2020optimal}. 
Let us denote by $C$ what we want to track. For a sequence that converges to $C$, in the following lemma, we show how to design a sampling rule so that $\frac{N_{a,x}(t)}{t}$ also converges to  $C$.

\begin{lemma}\label{lmm:tracking lemma}
(Tracking a set $C$)  Let $(w(t))_{t\ge1}$ be a sequence taking values in $\set{W}$, such that there exists a compact, convex and non empty subset $C$ in $\set{W}$, there exists $\varepsilon >0$ and $t_0(\varepsilon) \ge 1$ such that $\forall t \geq t_0(\varepsilon)$, 
$$\min_{\bm{w}'\in C}\max_{a\in[K], x\in\mathcal{X}}\left|w_{a, x}(t) - w'_{a,x}\right|\le \varepsilon$$

Let $g:\mathbb{N} \to \mathbb{R}$ be a non-decreasing function that $g(0) = 0$, $g(t)/t \to 0$ as $t\to\infty$ and $\forall n, m \geq 1$,
\begin{align*}
\inf \big\{ n \in \mathbb{N}: g(n) \geq m \big\} > \inf \big\{ n \in \mathbb{N}: g(n) \geq m-1 \big\} + K.
\end{align*}

Define for every $t'\in\{0,\dots,t-1\}$, $\varphi^g_{x,t'} = \{ a : N_{a, x}(t') < g(N_{x}(t'))\}$ and a sampling rule as (\ref{eq:forced})
Then for all $a \in [K]$ and $x\in\mathcal{X}$, 
$$N_{a, x}(t) > g(N_{x}(t)) -1,$$
and there exists $t_1(\varepsilon) \ge t_0(\varepsilon)$ such that $\forall t \ge t_1(\varepsilon)$, 
$$\min_{\bm w\in C}\max_{a\in[K], x\in\mathcal{X}}\left|\frac{N_{a,x}(t)}{t} -  \frac{1}{t}\sum_{s=1}^t\mathbbm{1}[X_s = x]w_{a,x}\right| \le 3(KD - 1) \varepsilon.$$
\end{lemma}

The proof of Lemma~\ref{lmm:tracking lemma} is inspired by the proof of Lemma 3 in \citet{Antos2008}, Lemma 17 in \citet{Garivier2016}, and Lemma~6 and Proposition~2 of \citet{jedra2020optimal}. We show the proof of Lemma~\ref{lmm:tracking lemma} as follows.
\begin{proof}
We separately show that 
$$N_{a, x}(t) > g(N_{x}(t))-1$$ and 
$$\min_{\bm w\in C}\max_{a\in[K], x\in\mathcal{X}}\left|\frac{N_{a,x}(t)}{t} -  \frac{1}{t}\sum_{s=1}^t\mathbbm{1}[X_s = x]w_{a,x}\right| \le 3(KD - 1) \varepsilon.$$

\textbf{Proof of $N_{a, x}(t) > g(N_{x}(t))-1$.} First, we justify that $N_{a, x}(t) > g(N_{x}(t))-1$.

For all $ m \in\mathbb{N}$, let us define
\begin{align*}
&k_m=\inf\{n \in \mathbb{N} : g(n)\geq m\},\\
&\mathcal{I}_m = \{ k_m, \dots, k_{m+1} -1 \}.
\end{align*}
From our assumptions on $g$, we have 
\begin{align*}
&|\mathcal{I}_m| > K,\\
&m \leq g(n) < m +1\ \ \ \forall n \in\mathcal{I}_m.
\end{align*}

We consider the following statement for all $ m\in\mathbb{N}$ and for all $ x\in \set{X}$: 
\begin{align}
\label{HR1} 
&\text{for}\ \text{all}\  t\in\mathbb{N}\ \text{such}\ \text{that}\ N_{x}(t) \in \mathcal{I}_m, \text{we}\ \text{have}\ \text{for}\ \text{all}\  a\in[K],  N_{a, x}(t) \geq m;\\
&\text{for}\ \text{all}\ t\in\mathbb{N}\ \text{such}\ \text{that}\ N_{x}(t) \geq k_{m} + K,  \ \text{we}\ \text{have}\ \varphi^g_{x, t} = \emptyset \ \ \text{and} \ N_{a, x}(t) \geq m+1.\nonumber
\end{align}
If \eqref{HR1} holds for all $m$, then using that for all $t$ and for all $a\in[K]$,  
$$N_{a, x}(n) > g(N_x(t)) -1$$
because from the definitions of $\mathcal{I}_m$ and $k_m$, for $t$ such that 
$$N_{x}(t) \in \mathcal{I}_m,$$ 
we have 
$$N_{a,x}(t) \geq m > g(N_{x}(t)) -1.$$
Here, we used $g(N_x(t)) \geq m$ and $g(N_x(t)) < m+1$ from the definition of $k_m$.

We prove \eqref{HR1} by induction with respect to  $m\in \mathbb{N}$. First, we show the statement holds for $m=0$. For all $t$ such that $N_{x}(t) \in \mathcal{I}_0$, it holds that for all $a\in[K]$ and for all $x\in\set{X}$,
$$\varphi^g_{x, t} = \{ a : 0\leq N_{a, x}(t) < g(N_{x}(t)) < 1\} =  \{ a : N_{a, x}(t) = 0\}.$$
Here, we used $\mathcal{I}_0 = \{ k_0, \dots, k_{1} -1 \}$ with $k_0=\inf\{n \in \mathbb{N} : g(n)\geq 0\}$ and  $k_1=\inf\{n \in \mathbb{N} : g(n)\geq 1\}$.
Therefore, for $t$ such that $N_{x}(t) \geq K = k_0 + K$, we have $N_{a,x}(t) \geq 1$ and $\varphi^g_{x, t} = \emptyset$. Thus, the statement holds for $m=0$. 

Suppose that for $m = m'\geq 0$, the statement is true; that is,
\begin{align*}
&\text{for}\ \text{all}\  t\in\mathbb{N}\ \text{such}\ \text{that}\ N_{x}(t) \in \mathcal{I}_{m'}, \text{we}\ \text{have}\ \text{for}\ \text{all}\  a\in[K],  N_{a, x}(t) \geq m';\\
&\text{for}\ \text{all}\ t\in\mathbb{N}\ \text{such}\ \text{that}\ N_{x}(t) \geq k_{m'} + K,  \ \text{we}\ \text{have}\ \varphi^g_{x, t} = \emptyset \ \ \text{and} \ N_{a, x}(t) \geq m'+1.\nonumber
\end{align*}
Then, we show the statement holds for $m = m'+1$. From the inductive hypothesis and assumption $k_{m' + 1} > k_{m'} + K$, since $k_{m'+1} - 1 \geq k_{m'} + K$, it holds that for all $a\in[K]$ and for all $x\in\set{X}$, 
$$N_{a,x}(k_{m'+1} - 1) \geq m' + 1.$$
From the definition of $\mathcal{I}_{m'  + 1}$, for $t$ such that $N_{x}(t) \in \mathcal{I}_{m'  + 1} = \{ k_{m'+1}, \dots, k_{m'+2} -1 \}$, $N_{a,x}(t) \geq N_{a,x}(k_{m'+1} - 1)$.
Therefore,
\begin{align*}
N_{a,x}(t) \geq N_{a,x}(k_{m'+1} - 1) \geq m' + 1
\end{align*}
Besides, for $t$ such that $N_{x}(t) \in \mathcal{I}_{m'  + 1}$ and for all $x\in\set{X}$,
$$m' + 1 \leq g(N_{x}(t)) < m'+2.$$
This leads to
$$\varphi^g_{x, t} = \{ a : m' + 1\leq N_{a, x}(t) < g(N_{x}(t)) < m' + 2\} = \{ a : N_{a,x}(t) = m' + 1\}.$$
Then, $A_t$ is chosen among this set while it is non empty. Therefore, for t such that $N_x(t) \geq k_{m'+1} + K$, it holds that for all $a\in[K]$ and $x\in\set{X}$, $\varphi^g_{x, t} = \emptyset$ and $N_{a,x}(t) \geq m'+2$. Thus, the statement \eqref{HR1} holds when $m = m'+1$.

\textbf{Proof of $\min_{\bm w\in C}\max_{a\in[K], x\in\mathcal{X}}\left|\frac{N_{a,x}(t)}{t} -  \frac{1}{t}\sum_{s=1}^t\mathbbm{1}[X_s = x]w_{a,x}\right| \le 3(KD - 1) \varepsilon$.} First, the condition
$$\min_{\bm{w}'\in C}\max_{a\in[K], x\in\mathcal{X}}\left|w_{a, x}(t) - w'_{a,x}\right| \leq \varepsilon$$
for $\forall t \geq t_0(\varepsilon)$ ensures that for large $t$,
$$
\min_{\bm {w}' \in C}\max_{a\in[K], x\in\mathcal{X}} \left| \frac{1}{t}\sum_{s=1}^t\mathbbm{1}[X_s = x] w_{a,x}(s) - \frac{1}{t}\sum_{s=1}^t\mathbbm{1}[X_s = x] w'_{a,x}\right| \le \varepsilon.
$$

For all $t \ge 1$, we define
$$
\overline{\eta}_{a,x}(t) = \frac{1}{t}\sum_{s=1}^t\mathbbm{1}[X_s = x] w_{a,x}(s).
$$
Next, since $C$ is non-empty and compact, we can define
$$
\tilde{\bm{w}}(t) = \argmin_{\bm{w}^{\dagger}\in C} \max_{a\in[K], x\in\mathcal{X}}\left| \overline{\eta}_{a,x}(t)- \frac{1}{t}\sum_{s=1}^t\mathbbm{1}[X_s = x] w^{\dagger}_{a,x}\right|.
$$
Here, by convexity of $C$, there exists $t_0'(\varepsilon) \ge t_0(\varepsilon)$ such that $\forall t \ge t_0'(\varepsilon)$, we can obtain the following inequalities:
\begin{align}
\label{eq:first_condition_for_lemma}
&\min_{\bm{w}^{\dagger}\in C}\max_{a\in[K], x\in\set{X}}\left|\frac{N_{a,x}(t)}{t} -  \frac{1}{t}\sum_{s=1}^t\mathbbm{1}[X_s = x]w^{\dagger}_{a,x}\right|\le \max_{a\in[K], x\in\set{X}}\left| \frac{N_{a,x}(t)}{t}-  \frac{1}{t}\sum_{s=1}^t\mathbbm{1}[X_s = x] \tilde{w}_{a,x}(t)\right|
\end{align}
 and 
 \begin{align}
 \label{eq:second_condition_for_lemma}
 \max_{a\in[K], x\in\mathcal{X}}\left| \overline{\eta}_{a,x}(t) - \frac{1}{t}\sum_{s=1}^t\mathbbm{1}[X_s = x] \tilde{w}_{a,x}(t)\right| \le 2 \varepsilon.
 \end{align}

The first result can be directly obtained from the definition. We show the second result. To see that \eqref{eq:second_condition_for_lemma} holds, let us define for all $t \ge 1$, 
$$\bm v(t) = \argmin_{\bm{w}^{\dagger} \in C} \max_{a\in[K], x\in\mathcal{X}}|w_{a,x}(t) - w^{\dagger}_{a,x}|,$$
and observe that for all $ a \in [K]$ and $x\in\mathcal{X}$, we have
\begin{align*}
&\left\vert \overline{\eta}_{a,x}(t) - \frac{1}{t}\sum_{s=1}^t \mathbbm{1}[X_s = x]v_{a,x}(s)\right\vert\\
& = \left\vert \frac{1}{t}\sum_{s=1}^t \mathbbm{1}[X_s = x]w_{a, x}(s) - \frac{1}{t}\sum_{s=1}^t \mathbbm{1}[X_s = x]v_{a,x}(s)\right\vert\\
& \le \frac{1}{t} \sum_{s=1}^{t_0} \mathbbm{1}[X_s = x]\left\vert w_{a, x}(s) - v_{a,x}(s)\right\vert + \frac{1}{t} \sum_{s=t_0 + 1}^t \mathbbm{1}[X_s = x]\left\vert w_{a, x}(s) - v_{a,x}(s)\right\vert \\
& \le \frac{t_0(\varepsilon)}{t} + \frac{t-t_0(\varepsilon)}{t} \varepsilon.
\end{align*}
Note that $t_0(\varepsilon)$ is defined in the statement. Thus if $t\ge t_0' = \frac{t_0(\varepsilon)}{\varepsilon}$, then 
$$\max_{a\in[K], x\in\mathcal{X}}\left\vert \overline{\eta}_{a,x}(t) - \frac{1}{t}\sum_{s=1}^t \mathbbm{1}[X_s = x]v_{a,x}(s)\right\vert \le 2 \varepsilon.$$
Finally since the convexity of $C$ leads to
$$\left(\frac{1}{t} \sum_{s=1}^t \mathbbm{1}[X_s = 1]\bm{v}_1(s), \dots, \frac{1}{t} \sum_{s=1}^t \mathbbm{1}[X_s = D]\bm{v}_D(s)\right)^\top \in C,$$
 it follows that $\forall t \ge t_0'$
\begin{align*}
&\max_{a\in[K], x\in\mathcal{X}}\left\vert \overline{\eta}_{a,x}(t) - \frac{1}{t}\sum_{s=1}^t \mathbbm{1}[X_s = x]\tilde{w}_{a,x}(s)\right\vert \le \max_{a\in[K], x\in\mathcal{X}}\left\vert \overline{\eta}_{a,x}(t) - \frac{1}{t}\sum_{s=1}^t \mathbbm{1}[X_s = x]v_{a,x}(s)\right\vert \le 2\varepsilon.
\end{align*}
Thus, we showed that \eqref{eq:second_condition_for_lemma} holds.
By using \eqref{eq:first_condition_for_lemma} and \eqref{eq:second_condition_for_lemma}, we consider bounding the term
$$
\min_{\bm w\in C}\max_{a\in[K], x\in\mathcal{X}}\left|\frac{N_{a,x}(t)}{t} -  \frac{1}{t}\sum_{s=1}^t\mathbbm{1}[X_s = x]w_{a,x}\right|.
$$
Let us define for $a \in [K]$ and for all $t\ge  1$,
\begin{align*}
E_{a, x, t} = N_{a, x}(t) - \sum_{s=1}^t\mathbbm{1}[X_s = x]\tilde{w}_{a,x}(t).
\end{align*}
From \eqref{eq:first_condition_for_lemma}, there exists $t_1 \geq t'_0(\varepsilon)$ such that, for all $t\geq t_1$, 
\[\min_{\bm w\in C}\max_{a\in[K], x\in\mathcal{X}}\left|\frac{N_{a,x}(t)}{t} -  \frac{1}{t}\sum_{s=1}^t\mathbbm{1}[X_s = x]w_{a,x}\right| \leq \max_{a\in[K], x\in\mathcal{X}} \left|\frac{E_{a,x,t}}{t}\right|\;,\]

Therefore, we consider bounding $\max_{a\in[K], x\in\mathcal{X}} \left|\frac{E_{a,x,t}}{t}\right|$. Since
$$
\sum^K_{a=1}\sum_{x\in\mathcal{X}}E_{a,x,t} = \sum^K_{a=1}\sum_{x\in\mathcal{X}}N_{a, x}(t) -  \sum^K_{a=1}\sum_{x\in\mathcal{X}} \sum^t_{s=1}\mathbbm{1}[X_s = x]\tilde{w}_{a,x}(t) = t - t = 0
$$
we have
\[\sup_{a,x} |E_{a,x,t}| \leq (KD-1) \sup_{a,x} \ E_{a,x,t}.\]
Then, for every $a\in[K]$ and $x \in \set{X}$, we have $E_{a,x,t} \leq \sup_{a'\in[K]} \sup_{x' \in\mathcal{X}} E_{a',x',t}$ and  
\[E_{a, x, t} = - \sum_{(a', x') \neq (a,x)} E_{a',x',t} \geq - \sum_{(a', x') \neq (a,x)} \sup_{a', x'} E_{a',x',t} = - (KD-1)\sup_{a',x'}E_{a',x',t}\;.\] 

Next, we give an upper bound on $\sup_{a, x} E_{a, x, t}$, for $t$ large enough. Let $t_0' \geq t_0$ such that 
\[\forall t \geq t_0', \ \ \ g(t) \leq 2t \epsilon \ \ \text{and} \ \ 1/t \leq \epsilon\;.\]
We first show that for $t\geq t_0'$,  
\begin{equation}
(A_{t+1} = a) \subseteq \left(E_{a,x,t} \leq 2t \epsilon\right)\label{IntermediateEvent}
\end{equation}
To prove this, we write 
\[(A_{t+1} = a) \subseteq \mathcal{E}_1 \cup \mathcal{E}_2,
\]
where
\begin{align*}
\mathcal{E}_1 & = \left(a \in \argmin_{a\in[K]} \left(N_{a, x}(t) - t\sum^t_{s=1}\mathbbm{1}[X_s = x] w_{a, x}(s)\right)\right)\\
\mathcal{E}_2 & = \left(N_{a, X_t}(t) \leq g(N_x(t))\right)
\end{align*}
This inclusion is immediate by construction. Therefore, we show that
\[\mathcal{E}_1 \cup \mathcal{E}_2 \subseteq \left(E_{a,x,t} \leq 2t \epsilon\right).
\]

For the second case ($\mathcal{E}_2$), if $N_{a, x}(t) \leq g(N_x(t))$, we have 
$$E_{a,x,t} \leq g(N_x(t)) - \sum_{s=1}^t\mathbbm{1}[X_s = x] w_{a, x}(s) \leq g(N_x(t))\leq g(t) \leq 2t\epsilon,$$  by definition of $t_0'$.

In the first case ($\mathcal{E}_1$), for $t \ge t_0$, we have
\begin{align*}
&E_{a,x,t}  = N_{a, x}(t) -  \sum_{s=1}^t\mathbbm{1}[X_s = x]\tilde{w}_{a,x}(t)\nonumber \\
& = N_{a, x}(t) -  \sum^t_{s=1}\mathbbm{1}[X_s = x] w_{a, x}(s)  + \sum^t_{s=1}\mathbbm{1}[X_s = x] w_{a, x}(s) - \sum_{s=1}^t\mathbbm{1}[X_s = x]\tilde{w}_{a,x}(t)\\
&\leq N_{a, x}(t) -  \sum^t_{s=1}\mathbbm{1}[X_s = x] w_{a, x}(s)+ 2t \varepsilon\ \ \ \ \ \ \ \ \ \ \ \ \ \ \left(\text{since}\ \max_{a\in[K], x\in\set{X}}\left|\overline{\eta}(t), \frac{1}{t}\sum^t_{s=1}\mathbbm{1}[X_s = x]\tilde{w}(t)\right| \le 2 \varepsilon\right)\\
&\leq \min_{a\in [K]} \left(N_{a, x}(t) - t\sum^t_{s=1}\mathbbm{1}[X_s = x] w_{a, x}(s)\right) + 2t\varepsilon\ \ \ \ \ \ \ \ \ \ \ \ \ \ \ \ \ \ \ \ \ \ \ \ \ \ \ \ \ \  (\text{since}\ \mathcal{E}_1\ \text{holds})\\
&\leq 2t\varepsilon.
\end{align*}
where the last inequality holds because $\min_{a, x} E_{a,x,t} \leq 0$ holds from $\sum^K_{a=1}\sum_{x\in\mathcal{X}}E_{a,x,t} = 0$. This proves \eqref{IntermediateEvent}.

Here, $E_{a,x,t}$ satisfies $E_{a,x,t+1} = E_{a,x,t} + \mathbbm{1}[A_{t+1}=a, X_{t+1} = x]  - \mathbbm{1}[X_s = x]\tilde{w}_{a,x}(t+1) $, therefore, if $t\geq t'_0$, 
\begin{align*}
 E_{a,x,t+1} & \leq E_{a,x,t} + \mathbbm{1}[A_{t+1}=a, X_{t+1} = x]  - \mathbbm{1}[X_s = x] \tilde{w}_{a,x}(t+1)\\
 & \leq E_{a,x,t} + \mathbbm{1}[E_{a,x,t}\leq 2t\epsilon] - \mathbbm{1}[X_s = x]\tilde{w}_{a,x}(t+1). 
\end{align*}
We now prove by induction that for every $t \geq t'_0$, we have 
\[E_{a, x, t} \leq \max(E_{a,x, t'_0}, 2t\epsilon +1 ).\]
For $t=t'_0$, this statement clearly holds. Let $t\geq t'_0$ such that the statement holds.
If $E_{a, x, t} \leq 2t\epsilon$, we have
\begin{eqnarray*}E_{a,x,t+1} &\leq& 2t\epsilon + 1 - \tilde{w}_{a,x}(t+1) \leq 2t\epsilon +1 \leq \max(E_{a,x,t'_0},2t\epsilon+1) \\
&\leq& \max(E_{a,x,t'_0},2(t+1)\epsilon +1).\end{eqnarray*}
If $E_{a,x,t} > 2t\epsilon$, the indicator is zero and  
\[E_{a,x,t+1} \leq \max(E_{a,x,t'_0},2t\epsilon +1) - \tilde{w}_{a,x}(t+1) \leq \max(E_{a,x,t'_0},2(t+1)\epsilon+1),\]
which concludes the induction. 

For all $t\geq t'_0$, using that $E_{a,x,t'_0} \leq t'_0$ and $1/t \leq \epsilon$, it follows that 
\[\max_{a\in[K], x\in\mathcal{X}} \left|\frac{E_{a,x,t}}{t}\right| \leq  (KD-1) \max\left(2\epsilon + \frac{1}{t}, \frac{t'_0}{t}\right) \leq (KD - 1) \max\left(3\epsilon, \frac{t_0}{t}\right).\]
Hence, as mentioned above, from \eqref{eq:first_condition_for_lemma}, there exists $t_1 \geq t'_0(\varepsilon)$ such that, for all $t\geq t_1$, 
\[\min_{\bm w\in C}\max_{a\in[K], x\in\mathcal{X}}\left|\frac{N_{a,x}(t)}{t} -  \frac{1}{t}\sum_{s=1}^t\mathbbm{1}[X_s = x]w_{a,x}\right| \leq \max_{a\in[K], x\in\mathcal{X}} \left|\frac{E_{a,x,t}}{t}\right| \leq  3(KD-1) \epsilon\;,\]
which concludes the proof.
\end{proof}

Then, we can prove Lemma~\ref{lmm:tracking} as follows.
\begin{proof} Let $g(n) = (\sqrt{n} - K/2)_+$. Let $\varepsilon' =  \frac{\varepsilon}{3KD - 1} > 0$ and $C=\Phi(\nu)$. First, by Lemma~\ref{lem:aloca_convex}, and Lemma~\ref{lem:converge_w}, there exists $\xi(\varepsilon') > 0$ such that for all $\nu' = ((\mu'_{a,x}), (\zeta'_x))$ such that 
$$|\mu_{a,x} - \mu'_{a,x}| < \xi(\varepsilon')$$
and 
$$|\zeta_x - \zeta'_x | \le \xi(\varepsilon'),$$
we have
$$\max_{\bm w \in \Phi(\nu')}\min_{\bm {w}' \in \Phi(\nu)}\max_{a\in[K], x\in\mathcal{X}} \left| \frac{1}{t}\sum_{s=1}^t\mathbbm{1}[X_s = x] w_{a,x}(s) - \frac{1}{t}\sum_{s=1}^t\mathbbm{1}[X_s = x] w'_{a,x}\right| \le \varepsilon'/2.$$

From the law of large numbers, there exists $t_0(\varepsilon') \ge 0$ such that for all $t\ge t_0(\varepsilon')$, we have $| \mu_{a,x} - \hat{\mu}_{a,x}(t) \Vert \le \xi(\varepsilon')$ and $| \zeta_x - \hat{\zeta}_{x}(t)| \le \xi(\varepsilon')$. Here, the $\hat{\nu}(t)$ in the plug-in estimate $\Phi(\hat{\nu}(t))$ is $\hat{\nu}(t) = ((\hat{\mu}_{a,x}(t)), (\hat{\zeta}_x(t)))$. The condition \eqref{eq:lazycond} states that 

$$\lim_{t\to\infty}\min_{\bm{w}'\in \Phi(\nu)}\max_{a\in[K], x\in\mathcal{X}}\left|w_{a, x}(t) - w'_{a,x}\right| = 0$$ 
almost surely. This guarantees that there exist $t_1 \ge 1$ such that for all $t\ge t_1$, we have 
$$\min_{\bm {w}' \in \Phi(\nu)}\max_{a\in[K], x\in\mathcal{X}} \left| \frac{1}{t}\sum_{s=1}^t\mathbbm{1}[X_s = x] w_{a,x}(s) - \frac{1}{t}\sum_{s=1}^t\mathbbm{1}[X_s = x] w'_{a,x}\right| < \varepsilon'/2.$$
Now for all $t \ge \max (t_0(\varepsilon'),  t_1 )$,  we have
\begin{align*}
  & \min_{\bm {w}' \in \Phi(\nu)}\max_{a\in[K], x\in\mathcal{X}} \left| \frac{1}{t}\sum_{s=1}^t\mathbbm{1}[X_s = x] w_{a,x}(s) - \frac{1}{t}\sum_{s=1}^t\mathbbm{1}[X_s = x] w'_{a,x}\right|\\
  & \le \min_{\bm {w}' \in \Phi(\nu)}\max_{a\in[K], x\in\mathcal{X}} \left| \frac{1}{t}\sum_{s=1}^t\mathbbm{1}[X_s = x] w_{a,x}(s) - \frac{1}{t}\sum_{s=1}^t\mathbbm{1}[X_s = x] w'_{a,x}\right|\\
  &\ \ \  + \max_{\bm w \in \Phi(\hat{\nu}(t))}\min_{\bm {w}' \in \Phi(\nu)}\max_{a\in[K], x\in\mathcal{X}} \left| \frac{1}{t}\sum_{s=1}^t\mathbbm{1}[X_s = x] w_{a,x}(s) - \frac{1}{t}\sum_{s=1}^t\mathbbm{1}[X_s = x] w'_{a,x}\right| < \varepsilon'.
\end{align*}
Thus, we have shown that 
$$ \min_{\bm {w}' \in \Phi(\nu)}\max_{a\in[K], x\in\mathcal{X}} \left| \frac{1}{t}\sum_{s=1}^t\mathbbm{1}[X_s = x] w_{a,x}(s) - \frac{1}{t}\sum_{s=1}^t\mathbbm{1}[X_s = x] w'_{a,x}\right| \to 0$$ 
almost surely. 

Next, we recall that by Lemmas~\ref{lem:allEqual} and \ref{lem:aloca_convex}, $\Phi(\nu)$ is non empty, compact and convex. Thus, applying the (strong) law of large numbers and Lemma~\ref{lmm:tracking lemma} yields immediately that with 

\begin{align*}
\mathbb{P}\left(\min_{\vec{w}^* \in \Phi(\nu)}\left\{\lim_{t\to\infty}\frac{N_{a,x}(t)}{t} = \zeta_xw^*_{a, x} \right\} \right) = 1
\end{align*}

Here, we used
\begin{align*}
&\min_{\vec{w}^* \in \Phi(\nu)} \left| \frac{N_{a, x}(t)}{t}  - p(x)w^*_{a,x}(t)  \right|\\
&=\min_{\vec{w}^* \in \Phi(\nu)} \left| \frac{N_{a, x}(t)}{t} - \frac{1}{t}\sum_{s=1}^t\mathbbm{1}[X_s = x]w^*_{a,x}(t) + \frac{1}{t}\sum^t_{s=0}\mathbbm{1}[X_s = x]w^*_{a,x}(t) - p(x)\tilde{w}_{a,x}(t)  \right|\\
&\leq \min_{\vec{w}^* \in \Phi(\nu)}\left| \frac{N_{a, x}(t)}{t} - \frac{1}{t}\sum_{s=1}^t\mathbbm{1}[X_s = x]w^*_{a,x}(t)  \right| + \left| \left\{\frac{1}{t}\sum^t_{s=0}\mathbbm{1}[X_s = x] - p(x)\right\} w^*_{a, x}(t) \right|\\
&\leq \min_{\vec{w}^* \in \Phi(\nu)} \left| \frac{N_{a, x}(t)}{t} - \frac{1}{t}\sum_{s=1}^t\mathbbm{1}[X_s = x]w^*_{a,x}(t)  \right| + \left| \frac{1}{t}\sum^t_{s=0}\mathbbm{1}[X_s = x] - p(x)\right|,
\end{align*}
and for $t \geq t_0(\varepsilon')$
\begin{align*}
&\min_{\bm w^*\in C}\max_{a\in[K], x\in\mathcal{X}}\left|\frac{N_{a,x}(t)}{t} -  \frac{1}{t}\sum_{s=1}^t\mathbbm{1}[X_s = x]w^*_{a,x}\right| \le \varepsilon'\\
&\Vert\bm{\zeta} - \bm{\hat{\zeta}}_{t} \Vert \le \xi(\varepsilon').
\end{align*}
\end{proof}

\subsection{Proof of Theorem~\ref{thm:deltaPAC}}
\label{appdx:deltaPAC}

We proceed similarly to~\citet{Garivier2016}. Introducing, for $a,b\in[K]$, $T_{a,b} := \inf\{ t\in \mathbb{N} : Z_{a,b}(t) > \beta(t,\delta)\}$, we have 
\begin{align*}
\mathbb{P}_{\nu}(\tau_\delta<\infty,\hat{a}_{\tau_\delta}\neq a^*) & \leq \mathbb{P}_{\nu}\left(\exists a \in [K]\setminus a^*, \exists t \in \mathbb{N} : Z_{a,a^*}(t) > \beta(t,\delta)\right) \\ &\leq \sum_{a\in [K]\setminus \{a^*\}} \mathbb{P}_{\nu}(T_{a,a^*} < \infty)\;.
\end{align*}
We show that if $\beta(t,\delta)=\log(2t(K-1)/\delta)$ and $\mu_a<\mu_b$, then $\mathbb{P}_{\nu}(T_{a,b} < \infty) \leq \frac{\delta}{K-1}$.
For such a pair of arms, observe that on the event $\big\{T_{a,b}=t\big\}$ time $t$ is the first moment when $Z_{a,b}(t)$ exceeds the threshold $\beta(t,\delta)$, which implies by definition that  
\[1 \leq e^{-\beta(t,\delta)} \frac{\max_{\xi_a\geq \xi_b} p_{\xi_a}\big(\underline{R}_{a,x}(t), \underline{X}(t)\big) p_{\xi_b}\big(\underline{R}_{b,x}(t), \underline{X}(t)\big)}{\max_{\xi_a\leq \xi_b} p_{\xi_a}\big(\underline{R}_{a,x}(t), \underline{X}(t)\big) p_{\xi_b}\big(\underline{R}_{b,x}(t), \underline{X}(t)\big)}\;.\]

It thus holds that
\begin{align*}
&\mathbb{P}_{\nu}(T_{a,b} < \infty) = \sum^\infty_{t=1}\mathbb{P}_{\nu}(T_{a,b} = t) = \sum^\infty_{t=1}\mathbb{E}_{\nu}\Big[\mathbbm{1}[T_{a,b} = t]\Big]\\
&\leq \sum^\infty_{t=1} \exp\big(-\beta(t,\delta)\big)\mathbb{E}_\nu\left[\mathbbm{1}\left[T_{a,b} = t\right]\frac{\max_{\xi_a\geq \xi_b} p_{\xi_a}\big(\underline{R}_{a,x}(t), \underline{X}(t)\big) p_{\xi_b}\big(\underline{R}_{b,x}(t), \underline{X}(t)\big)}{\max_{\xi_a\leq \xi_b} p_{\xi_a}\big(\underline{R}_{a,x}(t), \underline{X}(t)\big) p_{\xi_b}\big(\underline{R}_{b,x}(t), \underline{X}(t)\big)}\right]\\
&\leq \sum^\infty_{t=1} \exp\big(-\beta(t,\delta)\big)\mathbb{E}_\nu\left[\mathbbm{1}\left[T_{a,b} = t\right]\frac{\max_{\xi_a\geq \xi_b} p_{\xi_a}\big(\underline{R}_{a,x}(t), \underline{X}(t)\big) p_{\xi_b}\big(\underline{R}_{b,x}(t), \underline{X}(t)\big)}{ p_{ \mu_a}\big(\underline{R}_{a,x}(t), \underline{X}(t)\big) p_{ \mu_b}\big(\underline{R}_{b,x}(t), \underline{X}(t)\big) }\right].
\end{align*}

We expand the expectation $\mathbb{E}_\nu\left[\mathbbm{1}\left[T_{a,b} = t\right]\frac{\max_{\xi_a\geq \xi_b} p_{\xi_a}\big(\underline{R}_{a,x}(t), \underline{X}(t)\big) p_{\xi_b}\big(\underline{R}_{b,x}(t), \underline{X}(t)\big)}{ p_{ \mu_a}\big(\underline{R}_{a,x}(t), \underline{X}(t)\big) p_{ \mu_b}\big(\underline{R}_{b,x}(t), \underline{X}(t)\big) }\right]$ as follows:
\small
\begin{align}
&\mathbb{E}_\nu\left[\mathbbm{1}\left[T_{a,b} = t\right]\frac{\max_{\xi_a\geq \xi_b} p_{\xi_a}\big(\underline{R}_{a,x}(t), \underline{X}(t)\big) p_{\xi_b}\big(\underline{R}_{b,x}(t), \underline{X}(t)\big)}{ p_{ \mu_a}\big(\underline{R}_{a,x}(t), \underline{X}(t)\big) p_{ \mu_b}\big(\underline{R}_{b,x}(t), \underline{X}(t)\big) }\right]\nonumber
\\
&= \sum_{\underline{r}_t\in\{0, 1\}^t}\sum_{\underline{a}_t\in[K]^t}\sum_{\underline{X}(t)\in\mathcal{X}^t}\mathbbm{1}\left[T_{a,b} = t\right](\underline{r}_t, \underline{a}_t, \underline{x}_t)\max_{\xi_a\geq \xi_b} p_{\xi_a}\big(\underline{R}_{a,x}(t) = \underline{r}_t, \underline{X}(t) = \underline{x}_t\big) p_{\xi_b}\big(\underline{R}_{b,x}(t) = \underline{r}_t, \underline{X}(t) = \underline{x}_t\big)\nonumber\\
&\ \ \  \cdot \prod_{c\in [K]\backslash\{a,b\}}\left[\prod^t_{s=1}p_{\mu_c}\big(R_{s, c} = r_s\mid X_s = x_s\big)p(A_1=c\mid X_1 = x_1)\prod^t_{s=2}p(A_{s}=c\mid X_{s} = x_s, \Omega_{s-1})\right]\frac{1}{p(\underline{X}(t) = \underline{x}_t)}\nonumber\\
\label{eq:pacbound}
&= \sum_{\underline{r}_t\in\{0, 1\}^t}\sum_{\underline{a}_t\in[K]^t}\sum_{\underline{x}_t\in\mathcal{X}^t}\mathbbm{1}\left[T_{a,b} = t\right](\underline{r}_t, \underline{a}_t, \underline{x}_t)\max_{\xi_a\geq \xi_b} p_{\xi_a}\big(\underline{R}_{a,x}(t) = \underline{r}_t\mid \underline{X}(t) = \underline{x}_t\big) p_{\xi_b}\big(\underline{R}_{b,x}(t) = \underline{r}_t\mid \underline{X}(t) = \underline{x}_t\big)\nonumber\\
&\ \ \  \cdot \prod_{c\in [K]\backslash\{a,b\}}\left[\prod^t_{s=1}p_{\mu_c}\big(R_{s,c} = r_s\mid X_s = x_s\big)p(A_1=c\mid X_1 = x_1)\prod^t_{s=2}p(A_{s}=c\mid X_{s} = x_s, \Omega_{s-1})\right]p(\underline{X}(t) = \underline{x}_t),
\end{align}
\normalsize
where $\underline{r}_t$ denotes the sequence $\{r_s\}^t_{s=1}$, $\underline{a}_t$ denotes the sequence $\{a_s\}^t_{s=1}$, $\underline{x}_t$ denotes the sequence $\{x_s\}^t_{s=1}$, $p_{\mu_b}\big(R_{t, c} =r_t\mid X_t = x_t\big)$ denotes the conditional density of $r_t$ given $x_t$ Note that $\mathbbm{1}\left[T_{a,b} = t\right]$ is a random variable depending on $(\underline{R}_t, \underline{A}_t, \underline{X}(t))$, therefore, we denote it as $\mathbbm{1}\left[T_{a,b} = t\right](\underline{r}_t, \underline{a}_t, \underline{x}_t)$.  For a vector $x$, let us introduce the Krichevsky-Trofimov distribution 
\begin{align*}
\mathrm{kt}(x) = \int^1_0 \frac{1}{\pi\sqrt{u(1-u)}}p_u(x),
\end{align*}
as defined in Lemma~11 of \citet{Garivier2016}. Then, following the same procedure as \citet{Garivier2016}, we bound \eqref{eq:pacbound} by
\begin{align*}
&\sum^\infty_{t=1}2t\exp\big(-\beta(t,\delta)\big)\sum_{\underline{r}_t\in\{0, 1\}^t}\sum_{\underline{a}_t\in[K]^t}\sum_{\underline{x}_t\in\mathcal{X}^t}\mathbbm{1}\left[T_{a,b} = t\right](\underline{r}_t, \underline{a}_t, \underline{x}_t)\mathrm{kt}\big(\underline{R}_{a,x}(t)\big)\mathrm{kt}\big(\underline{R}_{b,x}(t)\big)\\
&\times \prod_{c\in [K]\backslash\{a,b\}}\left[\prod^t_{s=1}p_{\mu'_c}\big(R_{s, c} = r_s\mid X_s = x_s\big)p(A_1=c\mid X_1 = x_1)\prod^t_{s=2}p(A_{s}=c\mid X_{s} = x_s, \Omega_{s-1})\right]p(\underline{X}(t) = \underline{x}_t)\\
&=\sum^\infty_{t=1}2t\exp\big(-\beta(t,\delta)\big)\sum_{\underline{r}_t\in\{0, 1\}^t}\sum_{\underline{a}_t\in[K]^t}\sum_{\underline{x}_t\in\mathcal{X}^t}\mathbbm{1}\left[T_{a,b} = t\right](\underline{r}_t, \underline{a}_t, \underline{x}_t)I(\underline{r}_t, \underline{a}_t, \underline{x}_t)p(\underline{X}(t) = \underline{x}_t),
\end{align*}
where the partially integrated likelihood 
\begin{align*}
&I(\underline{r}_t, \underline{a}_t, \underline{x}_t) =  \mathrm{kt}\big(\underline{R}_{a,x}(t)\big)\mathrm{kt}\big(\underline{R}_{b,x}(t)\big)\\
&\times \prod_{c\in [K]\backslash\{a,b\}}\left[\prod^t_{s=1}p_{\mu'_c}\big(R_{s, c} = r_s\mid X_s = x_s\big)p(A_1=c\mid X_1 = x_1)\prod^t_{s=2}p(A_{s}=c\mid X_{s} = x_s, \Omega_{s-1})\right]
\end{align*}
is the density of an alternative probability measure $\tilde{\mathbb{P}}$, under which $\mu_a$ and $\mu_b$ are drawn from a $\mathrm{Beta}(1/2, 1/2)$ distribution at the beginning of the sampling process. This is bounded as
\begin{align*}
&\leq\sum^\infty_{t=1}2t\exp\big(-\beta(t,\delta)\big)\sum_{\underline{x}_t\in\mathcal{X}^t}\tilde{\mathbb{P}}\left(T_{a,b} = t\right)p(\underline{X}(t) = \underline{x}_t)\\
&\leq \frac{\delta}{K-1}\sum^\infty_{t=1}\tilde{\mathbb{P}}\left(T_{a,b} = t\right) = \frac{\delta}{K-1}\tilde{\mathbb{P}}\left(T_{a,b} < \infty\right) \leq \frac{\delta}{K-1},
\end{align*}

Thus, for any $\mu_a<\mu_b$, then $\mathbb{P}_{\nu}(T_{a,b} < \infty) \leq \frac{\delta}{K-1}$. Therefore,
\begin{align*}
\mathbb{P}_{\nu}(\tau_\delta<\infty,\hat{a}_{\tau_\delta}\neq a^*) \leq \sum_{a\in [K]\setminus \{a^*\}} \mathbb{P}_{\nu}(T_{a,a^*} < \infty) \leq K - 1 \frac{\delta}{K-1} = \delta.
\end{align*}

\section{Proofs of Results in Section~\ref{sec:sample_comp}}
\subsection{Proof of Lemma~\ref{lem:AnalysisAS}}
\label{appdx:AnalysisAS}
\begin{proof}
In a Bernoulli bandit model, let $\mathcal{E}$ be an event such that
\begin{align*}
\mathcal{E} = \left\{ \forall a\in[K],\forall x\in\mathcal{X},\ \min_{\vec{w}^* \in \Phi(\nu)}\left|\lim_{t\to\infty}\frac{N_{a,x}(t)}{t} - \zeta_xw^*_{a, x} \right| = 0,\ \hat{\mu}_{a,x}(t)\overset{t\to\infty}{\to} \mu_{a,x},\ \frac{N_x(t)}{t} \overset{t\to\infty}{\to} \zeta_x \right\}.
\end{align*}
When considering a bandit model that belongs to a canonical one-parameter exponential family, suppose that the true parameter $\zeta_x$ is given; that is, $\hat{\zeta}_x(t) = \zeta_x$.
From the assumption on the sampling strategy (see Lemma~\ref{lmm:tracking}) and the law of large numbers, $\mathcal{E}$ is of probability $1$. On $\mathcal{E}$, there exists $t_0$ such that for all $t \geq t_0$, $\hat{\mu}_1(t) > \max_{a \neq 1} \hat{\mu}_a(t)$ and 
\begin{align*}
Z(t) &= \min_{a\neq 1}Z_{1,a}(t)\\
&= t\min_{a\neq 1}\sum_{x\in\mathcal{X}}\Bigg\{\frac{N_{1,x}(t)}{t}\left\{\hat{\mu}_{1,x}(t) \log \frac{\hat{\mu}_{1,x}(t)}{1 - \hat{\mu}_{1,x}(t)} + \log (1 - \hat{\mu}_{1,x}(t))\right\}\ \ \ \ \ \ \ \ \ \ \ \ \ \ \ \ \ \ \ \ \ \ \\
&\ \ \ \ \ \ \ \ \ \ \ \ \ \ \ \ \ \ \ \ \ \ \  + \frac{N_{a,x}(t)}{t}\left\{\hat{\mu}_{a,x}(t)\log \frac{\hat{\mu}_{a,x}(t)}{1 - \hat{\mu}_{a,x}(t)} + \log (1 - \hat{\mu}_{a,x}(t))\right\}\\
&\ \ \ \ \ \ \ \ \ \ \ \ \ \ \ \ \ \ \ \ \ \ \  - \frac{N_{1,x}(t)}{t}\left\{\hat{\mu}_{1,x}(t)\log \frac{\tilde{\xi}_{1, x}(t)}{1 - \tilde{\xi}_{1, x}(t)} + \log (1 - \tilde{\xi}_{1, x}(t))\right\}\ \ \ \ \ \ \ \ \ \ \ \ \ \ \ \ \ \ \ \ \ \ \\
&\ \ \ \ \ \ \ \ \ \ \ \ \ \ \ \ \ \ \ \ \ \ \  - \frac{N_{a,x}(t)}{t}\left\{\hat{\mu}_{a,x}(t)\log \frac{\tilde{\xi}_{a, x}(t)}{1 - \tilde{\xi}_{a, x}(t)} + \log (1 - \tilde{\xi}_{a, x}(t))\right\}\Bigg\}.
\end{align*}

By continuity of $m$, there exists an open neighborhood $\mathcal{N}(\nu, \varepsilon)$ of $\Phi(\nu) \times \lbrace \bm\mu\rbrace\times \lbrace \bm\zeta\rbrace$ such that for all $(\bm w', \bm \mu',\bm \zeta') \in \mathcal{N}(\nu, \varepsilon)$, it holds that
$$
m(\bm w', \nu') \ge (1-\varepsilon) m(\bm w', \nu),
$$
where where $\nu' = (\bm \mu', \bm \zeta')$, and $\bm w^\star$ is some element in $\Phi(\nu)$. Recall that the function $m$ is defined in Section~\ref{sec:characteristic} Now, observe that under the event $\mathcal{E}$, there exists $t_1 \ge t_0$ such that for all $t\ge t_1$ it holds that $((\hat{\mu}_{a,x}(t)), (\hat{\zeta}_x(t))) \in \mathcal{N}(\nu, \varepsilon)$, thus for all $t \ge t_0$, it follows that
$$
m(\left(N_a(t)/t \right)_{a \in[K]}, \hat{\nu}_t) \ge \frac{1}{1 + \epsilon} m(\bm{w}^*, \nu),
$$
where $\hat{\nu}_t = (\hat{\bm{\mu}}_t, \hat{\bm{\zeta}}_t)$.
Therefore, on $\mathcal{E}$, for all $t\geq t_1$,
\begin{align*}
Z(t) &= tm(\left(N_a(t)/t \right)_{a \in [K]}, \hat{\bm{\mu}}_t, \hat{\bm{\zeta}}_t) \ge \frac{t}{1 + \epsilon} m(\bm{w}^*, \nu) = \frac{t}{(1+\epsilon) T^\star(\nu)}.
\end{align*}

Consequently,
\begin{eqnarray*}
 \tau_\delta &= &  \inf\{ t \in \mathbb{N} : Z(t) \geq \beta(t,\delta) \} \\
 & \leq & t_1 \vee \inf\{t \in \mathbb{N} : t (1+\epsilon)^{-1}T^\star(\nu)^{-1} \geq \log(r(t) /\delta) \} \\ 
 & \leq & t_1 \vee \inf\{t \in \mathbb{N} : t (1+\epsilon)^{-1}T^\star(\nu)^{-1} \geq \log(Ct^{\alpha} /\delta) \}, 
\end{eqnarray*}
for some positive constant $C$. Using the technical Lemma~18 in \citet{Garivier2016}, it follows that on $\mathcal{E}$, as $\alpha \in [1,e/2]$,
\[\tau_\delta \leq t_1 \vee \alpha (1+\epsilon) T^\star(\nu) \left[\log \left(\frac{Ce((1+\epsilon)T^\star(\bm \mu))^\alpha}{\delta}\right) +  \log\log \left(\frac{C((1+\epsilon)T^\star(\nu))^\alpha}{\delta}\right)\right]\;.\]
Thus $\tau_\delta$ is finite on $\mathcal{E}$ for every $\delta \in (0,1)$, and 
\[\limsup_{\delta \rightarrow 0} \frac{\tau_\delta}{\log(1/\delta)} \leq (1+\epsilon)\,\alpha\,  T^\star(\nu)\;.\]
Letting $\epsilon$ go to zero concludes the proof.

\end{proof}

\subsection{Proof of Theorem~\ref{thm:AsymptoticSC}}
\label{appdx:ProofSC}

This proof also mainly follows \citet{Garivier2016}. We use the following proposition from \citet{Garivier2016}.

\begin{proposition}[Lemma~18 of \citet{Garivier2016}]\label{lem:technical} For every $\alpha \in [1,e/2]$, for any two constants $c_1,c_2>0$, 
\[x = \frac{\alpha}{c_1}\left[\log\left(\frac{c_2 e}{c_1^\alpha}\right) + \log\log\left(\frac{c_2}{c_1^\alpha}\right)\right]\]
is such that $c_1  x \geq \log(c_2 x^\alpha)$. 
\end{proposition}

To ease the notation, we assume that the bandit model $\nu$ is such that $\mu_1 > \mu_2 \geq \dots \geq \mu_K$.
Let $\epsilon>0$. From Lemma~\ref{lem:converge_w}, there exists $\Upsilon=\Upsilon(\epsilon) \leq (\mu_1-\mu_2)/4$ such that 
\begin{align*}
&\mathcal{I}_{\mu, \epsilon} := \prod_{x\in\set{X}}\Big( [\mu_{1, x} - \Upsilon,\mu_{1,x} + \Upsilon] \times [\mu_{2, x} - \Upsilon,\mu_{2,x} + \Upsilon] \times \dots \times [\mu_{K, x} - \Upsilon, \mu_{K, x} + \Upsilon]\Big),\\
&\mathcal{I}_{\zeta, \epsilon} := \prod_{x\in\set{X}} [\zeta_{x} - \Upsilon,\zeta_{x} + \Upsilon]
\end{align*}
satisfy that for all $\nu' \in \mathcal{I}_{\mu, \epsilon}\times \mathcal{I}_{\zeta, \epsilon}$, for $\vec{w}' \in \Phi(\nu')$, 
$$\min_{\bm {w} \in \Phi(\nu)}\max_{a\in[K], x\in\mathcal{X}} \left| \frac{1}{t}\sum_{s=1}^t\mathbbm{1}[X_s = x] w_{a,x}(s) - \frac{1}{t}\sum_{s=1}^t\mathbbm{1}[X_s = x] w'_{a,x}\right| \le \varepsilon.$$
In particular, whenever $(\hat{\mu}_{a,x}(t), \hat{\zeta}_{x})_{x\in\set{X}} \in \mathcal{I}_{\mu, \epsilon}\times \mathcal{I}_{\zeta, \epsilon}$, the empirical best arm is $\hat{a}_t =1$.

Let $T \in \mathbb{N}$ and define $h(T):=T^{1/4}$ and the event 
\[\mathcal{E}_T(\epsilon)= \bigcap_{t = h(T)}^{T}\left((\hat{\mu}_{a,x}(t), \hat{\zeta}_{x})_{x\in\set{X}} \in \mathcal{I}_{\mu, \epsilon}\times \mathcal{I}_{\zeta, \epsilon}\right).\]
The following proposition is a consequence of the proposed CTS algorithm, which ensures that each arm is drawn at least of order $\sqrt{t}$ times at round $t$. 

\begin{lemma}\label{lem:concSimple} There exist two constants $B,C$ (that depend on $\nu$ and $\epsilon$) such that \[\mathbb{P}_{\nu}(\mathcal{E}_T^c) \leq B T \exp(-C T^{1/8}).\] 
\end{lemma}

By using these gradients, we prove Theorem~\ref{thm:AsymptoticSC}.
\begin{proof}
On the event $\mathcal{E}_T$, it holds for $t \geq h(T)$ that $\hat{a}_t=1$ and the Chernoff stopping statistic rewrites 
\begin{align*}
&\max_{a\in[K]}\min_{b\neq 1}Z_{1,b}(t)=\min_{a\neq 1}Z_{1,a}(t)\\
&= t\min_{a\neq 1}\sum_{x\in\mathcal{X}}\Bigg\{\frac{N_{1,x}(t)}{t}\left\{\hat{\mu}_{1,x}(t) \log \frac{\hat{\mu}_{1,x}(t)}{1 - \hat{\mu}_{1,x}(t)} + \log (1 - \hat{\mu}_{1,x}(t))\right\}\ \ \ \ \ \ \ \ \ \ \ \ \ \ \ \ \ \ \ \ \ \ \\
&\ \ \ \ \ \ \ \ \ \ \ \ \ \ \ \ \ \ \ \ \ \ \  + \frac{N_{a,x}(t)}{t}\left\{\hat{\mu}_{a,x}(t)\log \frac{\hat{\mu}_{a,x}(t)}{1 - \hat{\mu}_{a,x}(t)} + \log (1 - \hat{\mu}_{a,x}(t))\right\}\\
&\ \ \ \ \ \ \ \ \ \ \ \ \ \ \ \ \ \ \ \ \ \ \  - \frac{N_{1,x}(t)}{t}\left\{\hat{\mu}_{1,x}(t)\log \frac{\tilde{\xi}_{1, x}(t)}{1 - \tilde{\xi}_{1, x}(t)} + \log (1 - \tilde{\xi}_{1, x}(t))\right\}\ \ \ \ \ \ \ \ \ \ \ \ \ \ \ \ \ \ \ \ \ \ \\
&\ \ \ \ \ \ \ \ \ \ \ \ \ \ \ \ \ \ \ \ \ \ \  - \frac{N_{a,x}(t)}{t}\left\{\hat{\mu}_{a,x}(t)\log \frac{\tilde{\xi}_{a, x}(t)}{1 - \tilde{\xi}_{a, x}(t)} + \log (1 - \tilde{\xi}_{a, x}(t))\right\}\Bigg\}\\
 & = t g \left((\hat{\mu}_{a,x}(t))_{x\in\set{X}},(\hat{\zeta}_{x}(t))_{x\in\set{X}},\left(\frac{N_{a,x}(t)}{t}\right)_{a\in[K], x\in\set{X}}\right)\;, 
\end{align*}
where we introduce the function 
\begin{align*}
&g \left((\hat{\mu}_{a,x}(t))_{x\in\set{X}},(\hat{\zeta}_{x}(t))_{x\in\set{X}},\left(\frac{N_{a,x}(t)}{t}\right)_{a\in[K], x\in\set{X}}\right)\\
&= \min_{a\neq 1}\sum_{x\in\mathcal{X}}\Bigg\{\frac{N_{1,x}(t)}{t}\left\{\hat{\mu}_{1,x}(t) \log \frac{\hat{\mu}_{1,x}(t)}{1 - \hat{\mu}_{1,x}(t)} + \log (1 - \hat{\mu}_{1,x}(t))\right\}\ \ \ \ \ \ \ \ \ \ \ \ \ \ \ \ \ \ \ \ \ \ \\
&\ \ \ \ \ \ \ \ \ \ \ \ \ \ \ \ \ \ \ \ \ \ \  + \frac{N_{a,x}(t)}{t}\left\{\hat{\mu}_{a,x}(t)\log \frac{\hat{\mu}_{a,x}(t)}{1 - \hat{\mu}_{a,x}(t)} + \log (1 - \hat{\mu}_{a,x}(t))\right\}\\
&\ \ \ \ \ \ \ \ \ \ \ \ \ \ \ \ \ \ \ \ \ \ \  - \frac{N_{1,x}(t)}{t}\left\{\hat{\mu}_{1,x}(t)\log \frac{\tilde{\xi}_{1, x}(t)}{1 - \tilde{\xi}_{1, x}(t)} + \log (1 - \tilde{\xi}_{1, x}(t))\right\}\ \ \ \ \ \ \ \ \ \ \ \ \ \ \ \ \ \ \ \ \ \ \\
&\ \ \ \ \ \ \ \ \ \ \ \ \ \ \ \ \ \ \ \ \ \ \  - \frac{N_{a,x}(t)}{t}\left\{\hat{\mu}_{a,x}(t)\log \frac{\tilde{\xi}_{a, x}(t)}{1 - \tilde{\xi}_{a, x}(t)} + \log (1 - \tilde{\xi}_{a, x}(t))\right\}\Bigg\}.
\end{align*}
From Lemma~\ref{lmm:tracking lemma}, there exists a constant for $T_\epsilon$ such that the following inequality holds on $\mathcal{E}_T$:
\begin{align*}
\forall t \geq \sqrt{T},\ \min_{\bm w\in \Phi(\nu)}\max_{a\in[K], x\in\mathcal{X}}\left|\frac{N_{a,x}(t)}{t} -  \frac{1}{t}\sum_{s=1}^t\mathbbm{1}[X_s = x]w_{a,x}\right| \le 3(KD - 1) \varepsilon.
\end{align*}
Then, we introduce
\[H^*_\epsilon(\nu) = \inf_{\substack{\mu'_{a,x} : |\mu'_{a,x} - \mu_{a,x}| \leq \Upsilon(\epsilon) \\ \zeta'_x : |\zeta'_x - \zeta_x| \leq \Upsilon(\epsilon) \\ w'_{a,x} : |w'_{a,x} -w^*_{a,x}|\leq 3(KD-1)\epsilon}} g(\bm \mu',\bm w')\;,\]
where
\begin{align*}
\bm w^* = \argmin_{\bm w\in \Phi(\nu)}\max_{a\in[K], x\in\mathcal{X}}\left|\frac{N_{a,x}(t)}{t} -  \frac{1}{t}\sum_{s=1}^t\mathbbm{1}[X_s = x]w_{a,x}\right| \le 3(KD - 1) \varepsilon.
\end{align*}
Here, on the event $\mathcal{E}_T$ it holds that for every $t \geq \sqrt{T}$, 
\[\left(\max_{a\in[K]} \min_{b \neq a} \ Z_{a,b}(t) \geq t H^*_\epsilon(\nu)\right)\;.\]

Let us define$T \geq T_\epsilon$. Then, on the event $\mathcal{E}_T$, 
\begin{eqnarray*}
 \min(\tau_\delta,T) & \leq & \sqrt{T} + \sum_{t=\sqrt{T}}^T \mathbbm{1}\left[\tau_\delta > t\right] \leq \sqrt{T} + \sum_{t=\sqrt{T}}^T \mathbbm{1}\left[\max_{a\in[K]} \min_{b \neq a} \ Z_{a,b}(t) \leq \beta(t,\delta)\right] \\
 & \leq & \sqrt{T} + \sum_{t=\sqrt{T}}^T \mathbbm{1}\left[t H_\epsilon^*(\nu) \leq \beta(T,\delta)\right] \leq \sqrt{T} + \frac{\beta(T,\delta)}{H_\epsilon^*(\nu)}\;.
\end{eqnarray*}
Introducing 
\[T_0(\delta) = \inf \left\{ T \in \mathbb{N} : \sqrt{T} + \frac{\beta(T,\delta)}{H_\epsilon^*(\nu)} \leq T \right\},\]
for every $T \geq \max (T_0(\delta), T_\epsilon)$, we have $\mathcal{E}_T \subseteq (\tau_\delta \leq T)$, therefore 
\[\mathbb{P}_{\nu}\left(\tau_\delta > T\right) \leq \mathbb{P}(\mathcal{E}_T^c) \leq BT \exp(-C T^{1/8})\]
and
\[\mathbb{E}_{\nu}[\tau_\delta] \leq T_0(\delta) + T_\epsilon + \sum_{T=1}^\infty BT \exp(-C T^{1/8})\;.\]
We now provide an upper bound on $T_0(\delta)$. Let us define $\eta >0$ and the constant 
\[C(\eta) = \inf \{ T \in \mathbb{N} : T - \sqrt{T} \geq T/(1+\eta)\}.\]
Then, we have
\begin{eqnarray*}
 T_0(\delta) & \leq & C(\eta) + \inf \left\{T \in \mathbb{N} : \frac{1}{H_\epsilon^*(\nu)} \log\left(\frac{r(T)}{\delta}\right) \leq \frac{T}{1+\eta}\right\} \\
 & \leq & C(\eta) + \inf \left\{T \in \mathbb{N} : \frac{H_\epsilon^*(\nu)}{1+\eta} T \geq \log\left(\frac{Dt^{1+\alpha}}{\delta}\right) \right\},
\end{eqnarray*}
where the constant $D$ is such that $r(T) \leq D T^\alpha$. By using Proposition~\ref{lem:technical}, we obtain, for $\alpha \in [1,e/2]$, 
\[T_0(\delta) \leq C(\eta) + \frac{\alpha(1+\eta)}{H^*_\epsilon(\nu)}\left[\log \left(\frac{De(1+\eta)^\alpha}{\delta (H_\epsilon^*(\nu))^\alpha}\right) +  \log\log \left(\frac{D(1+\eta)^\alpha}{\delta (H_\epsilon^*(\nu))^\alpha}\right) \right].\]
The last upper bound yields, for every $\eta>0$ and $\epsilon>0$, 
\[\liminf_{\delta \rightarrow 0} \frac{\mathbb{E}_{\nu}[\tau_\delta]}{\log(1/\delta)} \leq \frac{\alpha (1+\eta)}{H_\epsilon^*(\nu)}.\]
As $\eta$ and $\epsilon$ go to zero, by continuity of $g$ and by definition of $w^*$,
\[\lim_{\epsilon \rightarrow 0} H_\epsilon^*(\nu) = T^\star(\nu)^{-1}.\]
This yields 
\[\liminf_{\delta \rightarrow 0} \frac{\mathbb{E}_{\nu}[\tau_\delta]}{\log(1/\delta)} \leq \alpha T^\star(\nu)\;.\]
\end{proof}

\section{Details of Experiments}
\label{appdx:exp}

\subsection{Calculation of an Optimal Weight}
To update the allocation $w(t)$, we need to solve minimax optimization problem defined as \eqref{eq:inner_opt}. Unlike \citet{Garivier2016}, we do not have an analytical solution for this problem. Therefore, we solve this problem numerically, using sequential quadratic programming. In our experiments, we use the sequential least squares programming (SLSQP) algorithm implemented in the {\texttt{optimize.minimize}} method of scipy, which is a Python library. Note that \citet{Garivier2016} only used the bisection method for the numerical optimization from the help of the analytical solution of the inner optimization in $L_{a,b}(\cdot)$. Unlike \citet{Garivier2016}, in our case, errors of optimization affect the results more. 

\subsection{Environment of Experiments}
All experiments were conducted on a MacBook Pro with a 2.8GHz quad-core Intel Core i7. We use Python language. The version of Python is 3.7.5, and that of SciPy is 1.4.1. To reduce the computational load, $\bm w$ is updated once every $10$ trial. This is an asymptotically negligible heuristic.

\subsection{Experimental Settings and Additional Results with Bernoulli bandit models}
In all experiments with Bernoulli bandit models, we assume that there exist two contexts $X_t \in\{1,2\}$ and each context is drawn with probability $0.5$. 

We conduct three additional experiments with different settings from the one in Section~\ref{sec:experiments}. For the Bernoulli bandit model, we consider a situation where the marginalized mean rewards are $\{\mu_1, \mu_2, \mu_3, \mu_4\} = \{0.5, 0.45, 0.43, 0.4\}$, which is the same as one of the scenarios used in \citet{Garivier2016}. Suppose that for each context, the conditional mean rewards are given as $\{\mu_{1,1}, \mu_{2,1}, \mu_{3,1}, \mu_{4,1}\} = \{0.5, 0.01, 0.4, 0.01\}$ and $\{\mu_{1,2}, \mu_{2,2}, \mu_{3,2}, \mu_{4,2}\} = \{0.5, 0.89, 0.46, 0.79\}$. We show the evolutions of the GLRT statistic in Figure~\ref{fig:res_track2}. As well as the result shown in Section~\ref{sec:experiments}, the CTS algorithm achieves a smaller sample complexity than TS. However, the variance is larger than the case discussed in Section~\ref{sec:experiments}. We believe that this is due to the gaps between the mean rewards are smaller than in the previous case and to the errors of the estimation/optimization affect the results more.

Next, we consider another scenario: $\{\mu_1, \mu_2, \mu_3, \mu_4\} = \{0.3, 0.21, 0.2, 0.19, 0.18\}$ and $\{\mu_1, \mu_2, \mu_3, \mu_4\} = \{0.5, 0.45, 0.43, 0.4\}$, which are the same as \citet{Garivier2016} and our previous experiments. 
For each setting, we use the same conditional mean rewards as $\{\mu_{1,1}, \mu_{2,1}, \mu_{3,1}, \mu_{4,1}\} = \{0.5, 0.2, 0.2, 0.1\}$. The counterparts and $\{\mu_{1,2}, \mu_{2,2}, \mu_{3,2}, \mu_{4,2}\}$ for $\{\mu_1, \mu_2, \mu_3, \mu_4\} = \{0.3, 0.21, 0.2, 0.19, 0.18\}$ and $\{\mu_1, \mu_2, \mu_3, \mu_4\} = \{0.5, 0.45, 0.43, 0.4\}$ are $\{\mu_{1,2}, \mu_{2,2}, \mu_{3,2}, \mu_{4,2}\} = \{0.1, 0.22, 0.2, 0.28\}$ and $\{\mu_{1,2}, \mu_{2,2}, \mu_{3,2}, \mu_{4,2}\} = \{0.5, 0.7, 0.66, 0.7\}$, respectively. Compared to these cases, the previous experiments take more extreme values of the conditional mean rewards. Therefore, in the current setting, we expect the difference between the results of track-and-stop and contextual track-and-stop to be less than in the previous ones. We show the value of the GLRT statistic in Figure~\ref{fig:res_track3}. As we expect, improvement is limited in this case.

\begin{figure}[t]
\vspace{-0.5cm}
  \begin{center}
    \includegraphics[width=100mm]{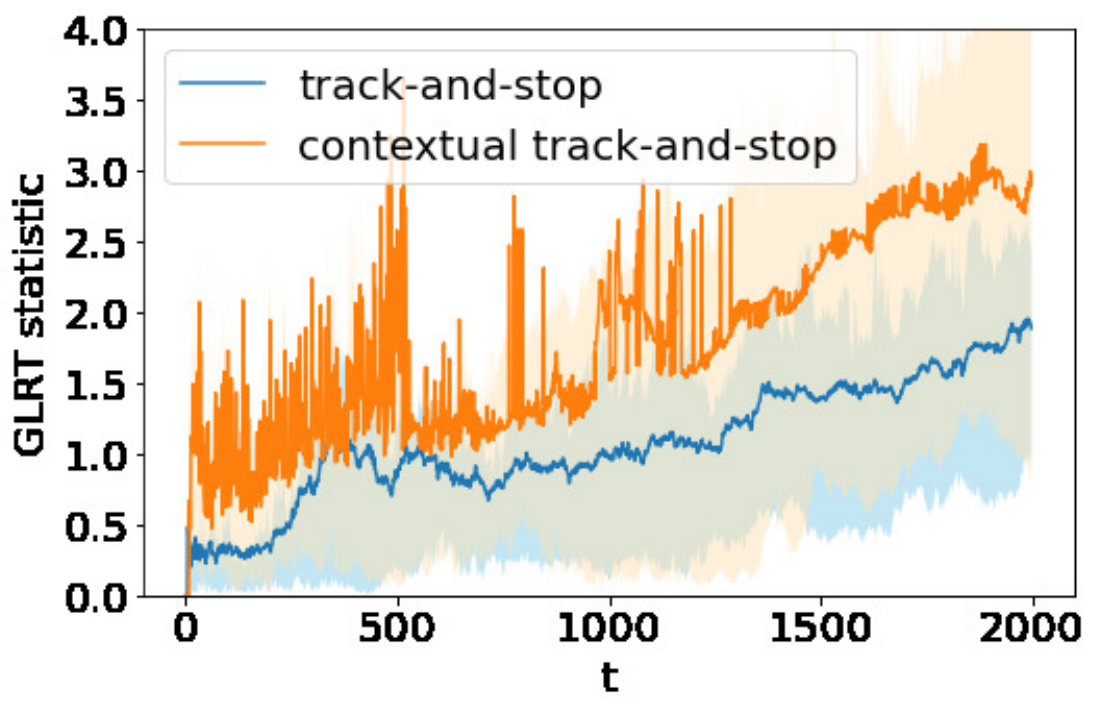}
  \end{center}
  \vspace{-0.5cm}
  \caption{This graph illustrates the maximum GLRT statistic $\max_{a\in[K]}\min_{b\in[K]\backslash\{a\}}Z_{a,b}(t)$. The solid line represents the averaged value over $20$ trials, and the light-colored area shows the values between the first and third quartiles.
  }
  \vspace{-0.2cm}
  \label{fig:res_track2}
  
  \begin{center}
    \includegraphics[width=140mm]{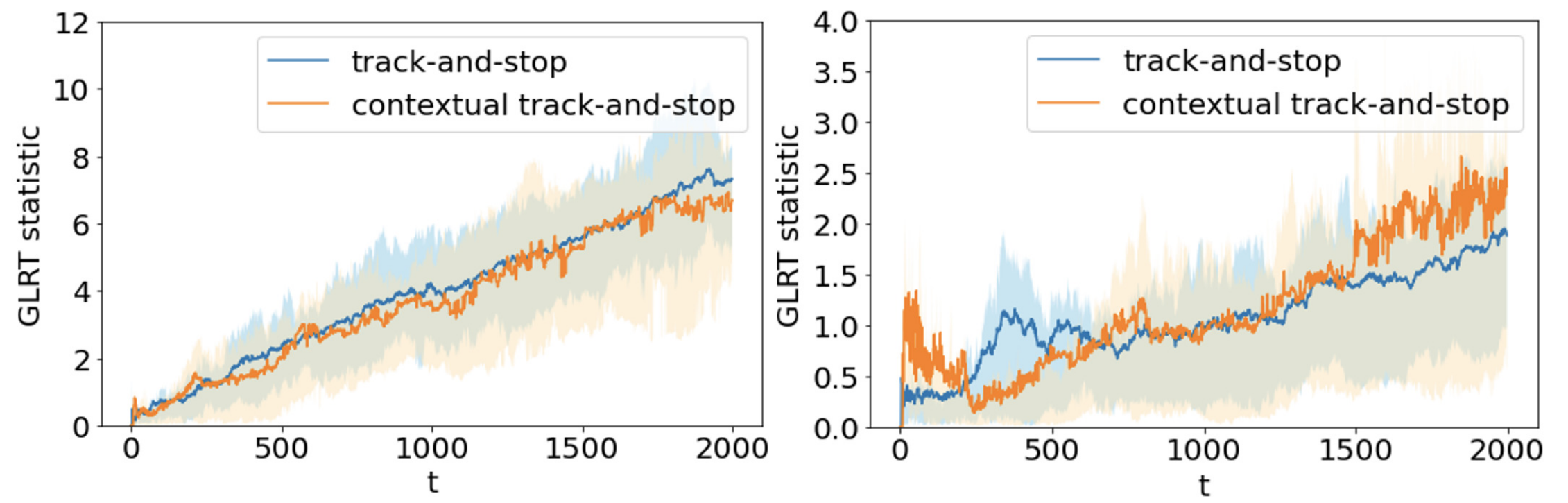}
  \end{center}
  \vspace{-0.5cm}
  \caption{This graph illustrates the maximum GLRT statistic $\max_{a\in{k}}\min_{b\in[K]\backslash\{a\}}Z_{a,b}(t)$. The left figure shows when $\{\mu_1, \mu_2, \mu_3, \mu_4\} = \{0.3, 0.21, 0.2, 0.19, 0.18\}$ is given. The right figure shows the results when $\{\mu_1, \mu_2, \mu_3, \mu_4\} = \{0.5, 0.45, 0.43, 0.4\}$ is given. The solid line represents the averaged value over $20$ trials, and the light-colored area shows the values between the first and third quartiles.
  }
  \vspace{-0.2cm}
  \label{fig:res_track3}
\end{figure}

\end{document}

%% file: def.tex
\usepackage[utf8]{inputenc}
\usepackage[T1]{fontenc}
\usepackage{graphicx}
\usepackage{booktabs}
\usepackage{bm}
\usepackage{tcolorbox}
\usepackage{multicol,multirow}
\usepackage{natbib}
\usepackage{comment}
\usepackage{booktabs}

\usepackage{enumerate}   
\usepackage{amsmath}
\usepackage{bbm}
\usepackage{tikz}

\usetikzlibrary{positioning,arrows}
\usepackage{wrapfig}
\DeclareMathOperator*{\argmax}{arg\,max}
\DeclareMathOperator*{\argmin}{arg\,min}

\usepackage{appendix}
\usepackage{amssymb}

\usepackage[capitalise]{cleveref}
\Crefname{assumption}{Assumption}{Assumptions}

\newcommand{\pa}{\mathrm{\pa}}

\renewcommand{\eqref}[1]{(\ref{#1})}

\newcommand{\RN}[1]{%
  \textup{\uppercase\expandafter{\romannumeral#1}}%
}

\usepackage{color,graphicx,lscape,longtable}

\def\boxit#1{\vbox{\hrule\hbox{\vrule\kern6pt\vbox{\kern6pt#1\kern6pt}\kern6pt\vrule}\hrule}}

\usepackage{xcolor}
\newcount\Comments  
\newcommand{\Pbb}{\mathbb{P}}
\newcommand{\Ebb}{\mathbb{E}}

\newcommand{\Ecal}{\mathcal{E}}

\newcommand{\KL}{\mathrm{KL}}
\newcommand{\kl}{\mathrm{kl}}
\newcommand{\Alt}{\mathrm{Alt}}

\usepackage{algpseudocode}
\algnewcommand{\IfThenElse}[3]{
  \State \algorithmicif\ #1\ \algorithmicthen\ #2\ \algorithmicelse\ #3}
\algnewcommand{\Foro}[2]{
  \State \algorithmicfor\ #1\ \algorithmicdo\ #2}

\newcommand{\ep}{\hfill $\Box$}

\usepackage{bbm}
\usepackage{paralist}
\usepackage{enumitem}
\usepackage{xcolor}
\usepackage{mathtools}
\usepackage{url}

\usepackage{dsfont}

\renewcommand{\vec}{\boldsymbol}

\DeclareMathOperator{\EXP}{\mathbb{E}}

\renewcommand{\tilde}{\widetilde}

\renewcommand{\Pr}{\mathbb{P}}
\newcommand{\set}[1]{\mathcal{#1}}

\newcommand{\indicator}{\mathds{1}}

\usepackage{scrextend}
\makeatletter
\def\maketitle{\par
 \begingroup
   \def\thefootnote{\fnsymbol{footnote}}
   \def\@makefnmark{\hbox to 0pt{$^{\@thefnmark}$\hss}}
   \deffootnote[1.7em]{1.6em}{2em}{$^\thefootnotemark$}
   \@maketitle \@thanks
 \endgroup
\setcounter{footnote}{0}
 \let\maketitle\relax \let\@maketitle\relax
 \gdef\@thanks{}\gdef\@author{}\gdef\@title{}\let\thanks\relax}
\makeatother